\documentclass[11pt,a4paper]{article}

\usepackage{fancyhdr}
\usepackage{amsmath,mathtools}
\usepackage{bm}
\usepackage{esint}
\usepackage{amsfonts}
\usepackage{amsthm}
\usepackage{amssymb}
\usepackage{amsbsy}
\usepackage{verbatim}
\usepackage{graphicx}
\usepackage{url}
\usepackage{mathrsfs}
\usepackage[hmargin = 1in,vmargin=.75in]{geometry}
\usepackage{color}
\usepackage{amstext}
\usepackage{stmaryrd}
\usepackage{enumerate}
\usepackage{algpseudocode}
\usepackage{algorithm}
\usepackage{pifont}
\usepackage{prettyref}
\usepackage{subcaption}

\font\msbm=msbm10

\numberwithin{equation}{section}

\theoremstyle{plain}
\newtheorem{theorem}{Theorem}[section]
\newtheorem{lemma}[theorem]{Lemma}

\newtheorem{proposition}[theorem]{Proposition}

\newtheorem{open}[theorem]{Open Problem}
\newtheorem{remark}[theorem]{Remark}
\def\mathbb#1{\hbox{\msbm{#1}}}

\newcommand{\bn}{\boldsymbol{n}}

\newcommand{\A}{\mathcal{A}}

\newcommand{\pa}{\partial}

\newcommand{\CC}{\mathbb{C}}

\newcommand{\RR}{\mathbb{R}}
\newcommand{\lag}{\langle}
\newcommand{\rag}{\rangle}

\newcommand{\Tr}{\text{Tr}}
\newcommand{\eps}{\epsilon}
\newcommand{\TB}{T^{\bot}}

\renewcommand{\Pr}{\mathbb{P}}
\newcommand*\diff{\mathop{}\!\mathrm{d}}

\DeclareMathOperator{\E}{\mathbb{E}}

\DeclareMathOperator{\diag}{diag}

\DeclareMathOperator{\RatioCut}{RatioCut}
\DeclareMathOperator{\NCut}{NCut}

\DeclareMathOperator{\cut}{cut}
\DeclareMathOperator{\Vol}{vol}

\DeclareMathOperator{\sym}{sym}
\DeclareMathOperator{\rw}{rw}
\DeclareMathOperator{\iso}{iso}
\DeclareMathOperator{\rcut}{rcut}
\DeclareMathOperator{\ncut}{ncut}

\definecolor{xl}{RGB}{200,50,120}

\begin{document}
\title{\bf Certifying Global Optimality of Graph Cuts via Semidefinite Relaxation: \\
A Performance Guarantee for Spectral Clustering}

\author{Shuyang Ling\thanks{Courant Institute of Mathematical Sciences, New York University (Email: sling@cims.nyu.edu).}~~and Thomas Strohmer\thanks{Department of Mathematics, University of California at Davis (Email: strohmer@math.ucdavis.edu).}\thanks{T.~Strohmer acknowledges partial support from the NSF via grants DMS 1620455 and DMS 1737943.} }

\maketitle
\begin{abstract} 
Spectral clustering has become one of the most widely used clustering techniques when the structure of the individual clusters is non-convex or highly anisotropic. Yet, despite its immense popularity, there exists fairly little theory about performance guarantees for spectral clustering.
This issue is partly due to the fact that  spectral clustering typically involves two steps which complicated its theoretical analysis: 
first, the eigenvectors of the associated graph Laplacian are used to embed the dataset, and second,  k-means clustering algorithm is applied
to the embedded dataset to get the labels. This paper is devoted to the theoretical foundations of spectral clustering and graph cuts. We consider a convex relaxation of graph cuts, namely ratio cuts and normalized cuts, that makes the usual two-step approach of  spectral clustering obsolete and at the same time
gives rise to a rigorous theoretical analysis of graph cuts and spectral clustering. We derive deterministic bounds for successful spectral clustering via a~\emph{spectral proximity condition} that naturally depends on the algebraic connectivity of each cluster
and the inter-cluster connectivity. Moreover, we demonstrate by means of some popular examples that our bounds can achieve near-optimality. Our findings are also fundamental to the theoretical understanding of kernel k-means. Numerical simulations confirm and complement our analysis.
%
\end{abstract}

\section{Introduction}

Organizing data into meaningful groups is one of the most fundamental tasks in data analysis and machine learning~\cite{hastie2009unsupervised,jain2010data}. 
{ k-means} is probably the most well known and most widely used clustering method~\cite{Lloyd82,ArthurV07,hastie2009unsupervised} in unsupervised learning. Yet, its performance is severely limited by two obstacles: (i) The k-means objective function is non-convex and finding its actual minimum is computationally hard {
even if there are only two clusters~\cite{ADHP09} or if the points lie in a 2D plane~\cite{MNV09}.}
(ii) k-means operates under the tacit assumption that individual clusters lie within convex boundaries, and in addition are reasonably isotropic or widely separated. To address the first obstacle, heuristics such as Lloyd's algorithm~\cite{Lloyd82}, are usually employed in an attempt to compute the solution in a numerically efficient manner. The second obstacle is more severe and independent of the actual algorithm used to find the objective function's minimum. 

Spectral clustering has arguably become the most popular clustering technique when the structure of the individual clusters is non-convex and/or highly anisotropic~\cite{Von07,BelkinN03,NgJW02}. The spectral clustering algorithm typically involves two steps: 
(i) Laplacian eigenmaps: construct a similarity graph from the data and the eigenvectors of the associated graph Laplacian are used to embed the dataset into the feature space; (ii) rounding procedure: k-means is applied
to the embedded dataset to obtain the clustering. As pointed out in~\cite{Von07}, the immense success of spectral clustering lies in its flexibility to deal with data of various shapes and complicated geometry, mainly due to the Laplacian eigenmap based embedding prior to the k-means procedure. Thus spectral clustering is also regarded as a variant of kernel k-means~\cite{DhillonGK04}. 
However, despite its enormous popularity and success, our theoretical understanding of the performance of spectral clustering is still rather vague. While there is vast empirical evidence of clustering examples in which e.g.\ spectral clustering by far outperforms k-means, there exists little rigorous theoretical analysis---even for very simple cases---that would prove the superiority of spectral clustering, partly due to the two-step procedure which complicates its theoretical analysis. 

This paper is devoted to the theoretical foundations of spectral clustering and graph cuts. 
To begin with, we look at spectral clustering from a graph cut point of view and briefly review the state-of-the-art results which suggest to some extent why spectral clustering works.
The basic intuition behind data clustering is to partition points into different groups based on their similarity. A partition of the data points always corresponds to a graph cut of the associated adjacency/similarity matrix. From the perspective of graph cuts, instead of computing the minimal graph cuts to obtain the data clustering, 
it is preferable to find a graph cut such that the sizes of all clusters are balanced and the inter-cluster connectivity is minimized. This is made possible by considering minimal ratio cuts and normalized cuts,  which represent a traditional problem in graph theory~\cite{bondy1976graph,arora2009expander}. Those problems arise in a diverse range of applications besides spectral clustering, including community detection~\cite{Abbe17,Agarwal17,AminiL18}, computer vision and image segmentation~\cite{ShiM00}. While finding the optimal balanced graph cuts is a computationally hard problem in general, significant progresses have been made to relax this problem by linking it to the spectra of the associated Laplacian matrix. This link immediately leads to spectral graph theory~\cite{Chung97} which has made a great impact on many branches of mathematics and computer sciences. 
The two-step spectral clustering algorithm can be derived via graph ratio cuts~\cite{BelkinN02,BelkinN03,Von07,HagenK92} and normalized cuts~\cite{NgJW02,DhillonGK04,ShiM00}, which in turn are connected to graph Laplacian and normalized graph Laplacian, respectively. 

The rather limited existing theory on spectral clustering is based on plain matrix perturbation analysis~\cite{Von07,Stewart90,DK70SIAM} especially via the famous Davis-Kahan theorem. The main reasoning behind perturbation analysis relies on the (unrealistic) assumption that if all underlying clusters on the graph are disconnected from one another, the eigenvectors of graph Laplacian with respect to the first few smallest eigenvalues are exactly indicator vectors which identify the data labels automatically. 
In the case when the eigenvectors are not exactly the indicator vectors (i.e., when the graph is connected), the perturbation argument fails to give the exact clustering and thus k-means is needed to perform the ``rounding" procedure. 
Therefore, the perturbation argument, despite its simplicity, does not yield any optimality bounds that would establish under which conditions spectral clustering will succeed or fail to provide correct clustering.

Another direction of the state-of-the-art mathematical theories concentrates on spectral clustering for random data generative model especially for stochastic block model in~\cite{RoheCY11,LeiR15}. With the help of randomness, the performance bounds (such as misclassification rate) of the two-step spectral clustering algorithm are derived in~\cite{RoheCY11} and the consistency of spectral clustering is given in~\cite{LeiR15}. Yet another related line of research  focuses on understanding the convergence of the graph Laplacian associated with random samples to the 
Laplace-Beltrami operator on Riemannian manifolds~\cite{VonBB08,BelkinN03,BelkinN05,Singer06,SingerW16,Trillos18,TrillosS16}.  Those excellent works establish a  rigorous bridge between the discrete graph Laplacian and its continuous counterpart Laplace-Beltrami operator on the manifold~\cite{DoCarmo92}. 

\vskip0.25cm
While all the previous works are inspirational and illuminating, the key question is not fully addressed: under what type of conditions is spectral clustering  able to identify the planted underlying partition exactly? More generally, how to certify a graph cut as the global optimum of ratio cuts or normalized cuts by using only the spectral properties of the (either normalized or unnormalized) graph Laplacian?

In this paper we answer these fundamental theoretical questions by taking a different approach, namely  via considering convex relaxations of ratio cuts and normalized cuts, which solve spectral clustering as a special case (which may at first sound like a tautology, since spectral clustering in itself can be obtained as a relaxation of graph cuts). Our framework makes the standard
 two-step spectral clustering approach obsolete and at the same time 
gives rise to a rigorous theoretical analysis. We derive~\emph{deterministic} bounds under which our semidefinite programming (SDP) relaxation of spectral clustering will produce the correct planted clusters. One highlighted feature of our result is that no assumption  is imposed about the underlying probability distribution governing the data. Another important feature is that the derived bounds are independent of the number of clusters. This desirable property is a clear advancement over known theoretical results for SDP relaxation of k-means clustering, which do depend on the number of clusters. 

Moreover, our theory serves as a simple criterion to certify whether a graph cut is globally optimal under either ratio cuts or normalized cuts. The guarantees depend on a {\em spectral proximity condition}, a deterministic condition that  encodes the algebraic connectivity of each cluster (the Fiedler eigenvalue of the graph Laplacian) and the inter-cluster connectivity in the case of ratio cuts and graph Laplacian. For the normalized cuts and normalized graph Laplacian, the guarantees have an intuitive probabilistic interpretation from a random walk point of view. 
Our bounds can be seen as a kernel-space analog of the well-known (Euclidean-space) proximity condition appearing in the theoretical analysis of k-means~\cite{KumarK10,AwasthiS12,LLLSW17}. Furthermore, we demonstrate by means of simple, often used examples, that our theoretical framework can provide nearly optimal performance bounds for SDP spectral clustering.

Our approach is inspired by recent progress regarding the convex relaxation of k-means~\cite{PengW07,AwasBCKVW15,IguchiMPV17,LLLSW17} and follows the ``Relax, no need to round''-paradigm. 
Note however, while the convex relaxation of k-means in~\cite{AwasBCKVW15,IguchiMPV17,LLLSW17} can provide a theoretical analysis concerning a successful computation of the optimal solution to the k-means objective function, it cannot overcome the fundamental limitations of the k-means objective function itself vis-a-vis nonconvex and anisotropic clusters. One attempt to address the latter shortcomings of k-means consists in replacing the Euclidean distance with a kernel function, leading to the aptly named kernel k-means algorithm~\cite{DhillonGK04,XingJ03}. One can interpret the SDP spectral clustering framework derived in the current paper as an extension of the theoretical analysis of the convex relaxation k-means approach in~\cite{AwasBCKVW15,LLLSW17} to kernel k-means.

Moreover, due to the natural connection between graph Laplacians and diffusion maps~\cite{CLL05}, our paper also sheds light on our theoretical understanding of diffusion map based data organization. Finally, we  would like to acknowledge being influenced by the recent progress on the convex relaxation of community detection under stochastic block model~\cite{EmmanuelBH16,Bandeira18,Abbe17,Agarwal17,AminiL18,YanSC17} in which the adjacency matrix is a binary random matrix. In fact, community detection problem can be viewed as an example of {the graph partition problem} on random graphs and hence our approach shares certain similarities with these previous works to some extent.  However, as pointed out previously,  our theoretic framework is significantly different from the existing literature since our theory does not assume any randomness as prior information and thus applies to more general settings besides community detection problem.

\subsection{Organization of our paper}
This paper is organized as follows. In Section~\ref{s:spectral} we review the basics of spectral clustering, motivated by both ratio cuts and normalized cuts. The proposed semidefinite relaxation of spectral clustering and our main theorems are presented in Section~\ref{s:sdp}. We then demonstrate the near-optimality of the theoretical bounds by means of simple, well-known examples, see Section~\ref{s:bounds}, and propose two open problems. Section~\ref{s:numerics} is devoted to numerical experiments that illustrate and complement the theoretical analysis.  Finally, the proofs of our results can be found in Section~\ref{s:proofs}.

\subsection{Notation}

For a vector $z$, we denote by $\|z\|$ its Euclidean norm and by $\|z\|_{\infty}$ the maximum of the absolute values of its components. 
For a matrix $Z$, we denote by $Z^{(a,b)}$  the $(a,b)$-block of $Z$ (the size of the block $Z^{(a,b)}$ will be clear from the context) and by $Z^{\top}$  the transpose of $Z$. Furthermore,  $\|Z\|$ is the operator norm, $\|Z\|_F$ is its Frobenius norm, and $\|Z\|_{\infty} : = \max_{i}\sum_{j}|Z_{ij}|$ is the matrix infinity norm. We define $\lambda_l(Z)$ to be the $l$-th smallest eigenvalue of $Z$. Given two matrices $Z,Y\in\RR^{m\times n }$, we let $\lag Z,Y\rag$ be the canonical inner product of $Z$ and $Y$, i.e., $\lag Z,Y\rag = \Tr(Z^{\top} Y)$. For a vector $z$, we define $\diag(z)$ to be a diagonal matrix whose diagonal entries consist of $z$. For a scalar $z$, we let $\lfloor z\rfloor$ be the largest integer not exceeding $z$. 

The vector $1_m$ represents the $m\times 1$ vector with all entries equal to 1,  $J_{m\times n} = 1_m1_n^{\top}$ is the  $m\times n$  ``all-1" matrix, and $I_n$ is the $n\times n$ identity matrix. We say $Z\succeq Y$ if $Z- Y\succeq 0$, i.e., $Z-Y$ is positive semidefinite, and $Z\geq Y$ if every entry of $Z-Y$ is nonnegative, i.e., $Z_{ij} - Y_{ij}\geq 0$. {We denote $f(n) \gtrsim g(n)$ for two positive sequences $\{f(n)\}_{n\in\mathbb{Z}^+}$ and $\{g(n)\}_{n\in\mathbb{Z}^+}$ if there exists an absolute constant $c_0>0$ such that $f(n)\geq c_0 g(n)$ for all $n \in\mathbb{Z}^+$}. Finally,  ${\cal S}_n$ is the set of $n\times n$ symmetric matrices, ${\cal S}_n^+$ is the set of $n\times n$ symmetric positive semidefinite matrices, and $\RR^{n\times n}_+$ denotes the set of all $n\times n$ nonnegative matrices.

\section{Spectral clustering and graph cuts}\label{s:spectral}

Spectral clustering can be understood from the perspective of graph cuts. Here, we give a short introduction to spectral clustering and spectral graph theory.  The interested readers may refer to the excellent review~\cite{Von07} for more details about spectral clustering and its variations.
Spectral clustering is based on a similarity graph constructed from a given set of data points $\{x_i\}_{i=1}^N$ whose vertices correspond to data and edges are assigned a weight which encodes the similarity between any pair of data points,  i.e., if $x_i$ and $x_j$ are close with respect to some similarity measure, then a larger weight is assigned to the edge $(i,j)$. Once the graph is obtained, one can compute the graph Laplacian, either normalized or unnormalized, and get the eigenvectors of the graph Laplacian to embed the data, followed by k-means or other rounding procedures  to obtain the final clustering outcome. 
We refer the reader to~\cite{BelkinN02,BelkinN03,Von07,HagenK92} 
for spectral clustering based on the unnormalized graph Laplacian and~\cite{NgJW02,DhillonGK04,ShiM00} for the normalized version.

\subsection{A short tour of spectral clustering}
We introduce the basics in spectral graph theory such as several versions of graph Laplacian which will be used later, and also the standard algorithms of spectral clustering.
The first step of spectral clustering is to design the similarity matrix based on the data points.
A well-known way to construct such a graph is to employ a non-negative, even kernel function $\Phi_{\sigma}(x, y)$ where $\sigma$ determines the size of the neighborhood. Sometimes we call $\sigma$ the bandwidth. 
A common choice for kernel function is of the following form,
\[
\Phi_{\sigma}(x,y) = \Phi\left(\frac{\|x-y\|}{\sigma}\right),
\]
where $\Phi(t)$ is a decreasing function of $t$. Typical examples for $\Phi$ include:
\begin{itemize}
\item $\Phi(t) = 1_{\{|t| \leq 1\}}$ which connects points if their pairwise distance is smaller than $\sigma.$ This graph is known as $\sigma$-neighborhood graph and is more likely to be disconnected if some points are isolated.
\item $\Phi(t) = e^{-\frac{t^2}{2}}$, the heat kernel. The resulting similarity matrix is a weighted complete graph. This kernel is also related to the diffusion process on the graph. The heat kernel is probably the most widely used kernel in connection with spectral clustering and graph cuts.
\end{itemize}

Suppose we have $k$ planted clusters and the $a$-th cluster $\Gamma_a$ has $n_a$ data points, i.e., $|\Gamma_a| = n_a$. The data may not necessarily be linearly separable. 
Given a certain kernel $\Phi(\cdot)$, we denote the similarity matrix between cluster $\Gamma_a$ and cluster $\Gamma_b$ via
\begin{equation}\label{def:W}
W^{(a,b)}_{ij} := \Phi\left(\frac{\|x_{a,i}-x_{b,j}\|}{\sigma}\right), \quad W^{(a,b)} \in \RR^{n_a\times n_b},
\end{equation}
where $x_{a,i}$ is the $i$-th point in $\Gamma_a$.
A particularly popular choice is  the heat kernel, in which case  $W$ takes the form
\begin{equation}
\label{heatkernel}
W^{(a,b)}_{ij} := e^{-\frac{\|x_{a,i} - x_{b,j}\|^2}{2\sigma^2}}.
\end{equation}
 The total number of data points is $N = \sum_{a=1}^k n_a.$
Without loss of generality we assume that the vertices are ordered according to the clusters they are associated with, i.e., lexicographical order for $\{x_{a,i}\}_{1\leq i\leq n_a, 1\leq a\leq k}$. 
Hence, by combining all pairs of clusters, the full weight matrix becomes
\[
W := \begin{bmatrix}
W^{(1,1)} & W^{(1,2)} & \cdots & W^{(1,k)} \\
W^{(2,1)} & W^{(2,2)} & \cdots & W^{(2,k)} \\    
\vdots & \vdots & \ddots & \vdots \\
W^{(k,1)} & W^{(k,2)} & \cdots & W^{(k,k)}
\end{bmatrix}\in\RR^{N\times N}.
\]
From now on, we let $w_{ij}$ be the $(i,j)$ entry of the weight matrix $W$ and use $W_{ij}^{(a,b)}$ specifically for the $(i,j)$ entry in the $(a,b)$ block of $W$.
Given the full weight matrix $W$, the degree of vertex $i$ is $d_i = \sum_{j=1}^N w_{ij}$ and
the associated degree matrix is
\[
D := \diag(W1_N)
\]
where $D$ is an $N\times N$ diagonal matrix with $\{d_i\}_{i=1}^N$ on the diagonal.  
We  define the unnormalized graph Laplacian for the weight matrix $W$ as
\begin{equation}\label{def:L}
L := D - W
\end{equation}
and the symmetric normalized graph Laplacian as
\begin{equation}\label{def:Lsym}
L_{\sym} : = I_N - D^{-\frac{1}{2}}WD^{-\frac{1}{2}} = D^{-\frac{1}{2}}LD^{-\frac{1}{2}}.
\end{equation}
It is a simple exercise to verify that the quadratic form of $L$ satisfies
\begin{equation}\label{eq:qual}
v^{\top} L v = \sum_{1\leq i<j\leq N} w_{ij}(v_i - v_j)^2
\end{equation}
where $v_i$ is the $i$-th entry of $v$.

We also define
\begin{equation}\label{def:P}
P : = D^{-1}W, \qquad L_{\rw} := I_N - P
\end{equation}
where the row sums of $P$ are all equal to 1 and thus $P$ defines a Markov transition matrix on the graph; $L_{\rw}$ is called random walk normalized Laplacian. Here $P_{ij} = \frac{w_{ij}}{d_{i}}$, the $(i,j)$ entry of $P$, denotes the probability of a random walk starting from vertex $i$ and moving to the vertex $j$ in the next step.

\vskip0.5cm

For later use, we define a set of matrices with subscript ``$\iso$" which capture the within-cluster information. 
We denote the ``isolated'' weight matrix by $W_{\iso}$ that excludes the edges between different clusters, i.e., 
\[
 W_{\iso} := 
 \begin{bmatrix}
W^{(1,1)} & 0 & \cdots & 0 \\
0 & W^{(2,2)} & \cdots & 0 \\
\vdots & \vdots & \ddots & \vdots \\
0 & 0 & \cdots & W^{(k,k)}
 \end{bmatrix},
\] 
and the corresponding degree matrix 
\[
D_{\iso} := \diag(W_{\iso} 1_N).
\]
The unnormalized graph Laplacian associated with  $W_{\iso}$ is
\begin{equation}\label{def:Liso}
L_{\iso} := D_{\iso} - W_{\iso}.
\end{equation}
We also define the random walk normalized Laplacian and Markov transition matrix for $W_{\iso}$  as
\begin{equation}\label{def:Piso}
P_{\iso} := D_{\iso}^{-1}W_{\iso}, \quad L_{\rw, \iso} := I_N - D_{\iso}^{-1}W_{\iso}
\end{equation}
where $P_{\iso}$ and $L_{\rw, \iso}$ are block-diagonal matrices.

\vskip0.5cm

The following four matrices with subscript $``\delta"$ are used to measure the inter-cluster connectivity, namely,
\begin{align}\label{def:delta}
\begin{split}
W_{\delta} & := W - W_{\iso},  \\
D_{\delta} & : = D - D_{\iso} = \diag((W - W_{\iso})1_N), \\
 L_{\delta} & := L - L_{\iso} = D_{\delta} - W_{\delta}, \\
P_{\delta} & : = D^{-1} W_{\delta} = P - D^{-1}W_{\iso}.
\end{split}
\end{align}
From the definition above, we can see that $W_{\delta}$ and $P_{\delta}$ are the off-diagonal blocks of $W$ and $P$ respectively, and $D_{\delta}$ is a diagonal matrix whose diagonal entries equal the row sum of $W_{\delta}$. These three matrices contain information about the inter-cluster connectivity. 
\vskip0.25cm

We would like to point out that matrices with subscripts ``$\iso$" or ``$\delta$" are~\emph{depending} on the underlying partition $\{\Gamma_l\}_{l=1}^k$. 
So far, we also have seen graph Laplacian of three different weight matrices, i.e., $L$, $L_{\iso}$, and $L_{\delta}$, which are all positive semidefinite matrices because they are diagonally dominant and also can be seen from~\eqref{eq:qual}, and moreover the constant vector $1_N$ is in the null space. 
As long as a graph is connected, its corresponding graph Laplacian has a positive second smallest eigenvalue, cf.~\cite{Chung97}. Moreover, the dimension of the nullspace of the graph Laplacian equals the number of connected components. Therefore, if all edge weights satisfy $w_{ij} > 0$ (which is possible if e.g.\ the Gaussian kernel is used), we have
\begin{align*}
& \lambda_2(L)  > 0, \quad \lambda_k(L_{\iso}) = 0, \quad \lambda_{k+1}(L_{\iso}) = \min_{1\leq a\leq k}\lambda_2(L^{(a,a)}_{\iso}) > 0, 
\end{align*}
because $L_{\iso}$ has $k$ diagonal blocks and each one corresponds to a connected subgraph.
Moreover, the nullspace of $L_{\iso}$ is spanned by $k$ indicator vectors in $\RR^N$, i.e., the columns of $U_{\iso}$,
\[
U_{\iso} := 
\begin{bmatrix}
\frac{1}{\sqrt{n_1}}1_{n_1} & 0 & \cdots &  0 \\
0 &  \frac{1}{\sqrt{n_2}}1_{n_2} & \cdots & 0 \\
\vdots & \vdots & \ddots & \vdots\\
0 & 0 & \cdots & \frac{1}{\sqrt{n_k}}1_{n_k}
\end{bmatrix}\in\RR^{N\times k}, \quad U_{\iso}^{\top}U_{\iso} = I_k.
\]

\vskip0.25cm

Assume for the moment that the original data set has $k$ clusters {\em and} that the graph constructed  from the data has $k$ connected components. In this case $L$ will be a true block-diagonal matrix (after necessary permutations), it will have an eigenvalue $0$ with multiplicity $k$ and the corresponding eigenvectors will be indicator vectors that represent cluster membership of the data~\cite{Von07}. 

However, since initially we are {\em not} given the graph, but
the data, we would have to assume that we know the cluster membership already a priori to be able to chose the ideal kernel that would then yield a graph with exactly
$k$ connected components. Since this is a futile assumption, $L$ will~\emph{never} be an exact block-diagonal matrix, which in turn implies that the relevant eigenvectors will not be indicator vectors that represent cluster membership. Hence, standard spectral clustering essentially {\em always} necessitates a second step. This step may consist in rounding the eigenvectors to indicator vectors or, more commonly, in applying a method like k-means to the embedded data set. 

We summarize the two most frequently used versions of spectral clustering algorithms in Algorithm~\ref{alg:spectral} and~\ref{alg:spectral2} which use unnormalized and normalized graph Laplacian respectively.


\begin{algorithm}[h!]
\caption{Unnormalized spectral clustering}\label{alg:spectral}
\begin{algorithmic}[1]
\State {\bf Input:} Given the number of clusters $k$ and a dataset $\{x_i\}_{i=1}^N$, construct the similarity matrix $W$ from $\{x_i\}_{i=1}^N.$
\State Compute the unnormalized graph Laplacian $L = D- W$.
\State Compute the eigenvectors $\{u_l\}_{l=1}^k$ of $L$ w.r.t. the smallest $k$ eigenvalues.
\State Let $U = [u_1, u_2, \cdots, u_k]\in\RR^{N\times k}$. Perform k-means clustering on the rows of $U$ by using Lloyd's algorithm. 
\State Obtain the partition based on the outcome of k-means.
\end{algorithmic}
\end{algorithm}

\begin{algorithm}[h!]
\caption{Normalized spectral clustering}\label{alg:spectral2}
\begin{algorithmic}[1]
\State {\bf Input:} Given the number of clusters $k$ and a dataset $\{x_i\}_{i=1}^N$, construct the similarity matrix $W$ from $\{x_i\}_{i=1}^N.$
\State Compute the normalized graph Laplacian $L_{\sym} = I_N- D^{-\frac{1}{2}}WD^{-\frac{1}{2}}$.
\State Compute the eigenvectors $\{u_l\}_{l=1}^k$ of $L_{\sym}$ w.r.t. the smallest $k$ eigenvalues.
\State Let $U = [u_1, u_2, \cdots, u_k]\in\RR^{N\times k}$. Perform k-means clustering on the rows of $D^{-\frac{1}{2}}U$ by using Lloyd's algorithm. 
\State Obtain the partition based on the outcome of k-means.
\end{algorithmic}
\end{algorithm}
In the Step 4 of Algorithm~\ref{alg:spectral2}, one uses $D^{-\frac{1}{2}}U$ instead of $U$, which differs from Algorithm~\ref{alg:spectral}. This is to ensure that $D^{-\frac{1}{2}}U$ consists of $k$ indicator vectors when  the graph has $k$ connected components and $L_{\sym}$ is a block-diagonal matrix.

Despite the tremendous success of spectral clustering in applications, its theoretical understanding is still far from satisfactory. 
Some theoretical justification for spectral clustering has been built on basic perturbation theory, by considering $L$ as the sum of the block-diagonal matrix $L_{\iso}$ and the perturbation term $L_{\delta}$, cf~\cite{NgJW02}. 
One can then invoke the Davis-Kahan theorem~\cite{DK70SIAM}, which bounds the difference between eigenspaces of symmetric matrices under perturbations. Then an error bound is obtained between $U$ and $U_{\iso}$ in terms of $\|L_{\delta}\|$ (or $\|L_{\delta}\|_F$) where $U$ and $U_{\iso}$ are the eigenvectors w.r.t. the smallest $k$ eigenvalues of $L$ and $L_{\iso}$ respectively.
However, the statements obtained with this line of reasoning have been more of a qualitative nature since the error bound between $U$ and $U_{\iso}$ does not immediately reflect the quality of clustering, partly due to the difficulty of analyzing the performance of k-means applied to $U.$
Thus, the perturbation arguments have not yet provided explicit conditions under which spectral clustering
would succeed or fail, not to speak of bounds that are anywhere near optimality, or even theorems that would just prove that spectral clustering does actually outperform 
k-means in simple, often-used examples when promoting spectral clustering.

\subsection{Understanding spectral clustering via graph cuts}

Graph partitioning provides a powerful tool of understanding and deriving
spectral clustering; it also becomes the foundation of this work.
Given a graph, one wants to divide it into several pieces such that the inter-cluster connectivity is small and each cluster is well connected within itself. However, only based on this criterion, this does usually not give satisfactory results since one single vertex may likely be treated as one cluster. As a consequence, it is usually preferable to have clusters whose sizes are relatively large enough, i.e., clusters of very small size should be avoided. To realize that, one uses ratio cuts~\cite{Von07,HagenK92} and normalized cuts~\cite{DhillonGK04,ShiM00} to ensure the balancedness of cluster sizes. Hence, we now discuss ratio cuts and normalized cuts, and their corresponding spectral relaxation. We also want to point out that the discussion about graph cuts applies to more~\emph{general} settings and spectral clustering is viewed  to some extent as a special case of graph cuts.

\subsubsection*{Ratio cuts and their spectral relaxation}
Given a disjoint partition $\{\Gamma_a\}_{a=1}^k$ such that $\sqcup_{a=1}^k\Gamma_a = [N] : =\{1,\cdots, N\}$, we define ratio cuts (RatioCut) as
\begin{equation}\label{def:ratiocut}
\text{RatioCut}(\{\Gamma_a\}_{a=1}^k) : = \sum_{a=1}^k \frac{\cut(\Gamma_a, \Gamma_a^c)}{|\Gamma_a|}.
\end{equation}
Here, the cut is defined as the weight sum of edges whose two ends are in different subsets, 
\begin{equation}\label{eq:cut}
\cut(\Gamma, \Gamma^c) : = \sum_{i\in\Gamma, j\in\Gamma^c} w_{ij}
\end{equation}
where $\Gamma$ is a subset of vertices and $\Gamma^c$ is its complement. In fact,~\eqref{eq:cut} can be neatly written in terms of the graph Laplacian $L.$
By definition of $L$ in~\eqref{def:L}, 
\begin{equation}\label{eq:rcut-L}
\cut(\Gamma_a, \Gamma_a^c): = \sum_{i\in\Gamma_a, j\in\Gamma_a^c} w_{ij} =  \lag L^{(a,a)}, 1_{|\Gamma_a|}1_{|\Gamma_a|}^{\top}\rag,
\end{equation}
which follows from 
\begin{align*}
\lag L^{(a,a)}, 1_{|\Gamma_a|}1_{|\Gamma_a|}^{\top}\rag 
& = \lag D^{(a,a)} - W^{(a,a)}, 1_{|\Gamma_a|}1_{|\Gamma_a|}^{\top}\rag =
  \sum_{l=1}^k 1_{|\Gamma_a|}^{\top}W^{(a,l)}1_{|\Gamma_l|} - 1_{|\Gamma_a|}^{\top}W^{(a,a)}1_{|\Gamma_a|}\\ 
  & = \sum_{l\neq a}1_{|\Gamma_a|}^{\top}W^{(a,l)}1_{|\Gamma_l|} = \sum_{i\in\Gamma_a, j\in\Gamma_a^c} w_{ij}.
\end{align*}
Therefore, RatioCut is in fact the inner product between the graph Laplacian $L$ and a block-diagonal matrix $X_{\rcut},$
\begin{align*}
\RatioCut(\{\Gamma_a\}_{a=1}^k)& = \sum_{a=1}^k \frac{1}{|\Gamma_a|}\lag L^{(a,a)}, 1_{|\Gamma_a|}1_{|\Gamma_a|}^{\top}\rag = \lag L, X_{\rcut}\rag,
\end{align*}
where 
\begin{equation}\label{def:Xr}
X_{\rcut} := \sum_{a=1}^k \frac{1}{|\Gamma_a|} 1_{\Gamma_a}1_{\Gamma_a}^{\top} = \text{blockdiag}\left(\frac{1}{|\Gamma_1|}1_{|\Gamma_1|}1_{|\Gamma_1|}^{\top}, \cdots, \frac{1}{|\Gamma_k|}1_{|\Gamma_k|}1_{|\Gamma_k|}^{\top}\right)\in\RR^{N\times N},
\end{equation}
and $1_{\Gamma_a}(\cdot)$ is an indicator vector which maps a vertex to a vector in $\RR^N$ via
\[
1_{\Gamma_a}(l) = 
\begin{cases}
1, & l \in\Gamma_a, \\
0, & l \notin \Gamma_a.
\end{cases}
\]
Obviously, by putting the cardinality of $\Gamma_a$ in the denominator of~\eqref{def:ratiocut}, one can avoid small clusters and thus RatioCut is a more favorable criterion to conduct graph partition. However, minimizing RatioCut is an NP-hard problem,~{see~\cite{WW93} for a detailed discussion}. Here we discuss one very popular and useful relaxation of RatioCut which relates the RatioCut problem to an eigenvalue problem. 

From our previous discussion, we realize that to minimize RatioCut over all possible partitions $\{\Gamma_a\}_{a=1}^k$ of $[N]$, it suffices to minimize $\lag L, Z\rag$ for all matrices $Z$ as~\eqref{def:Xr} which is essentially a positive semidefinite projection matrix. 
Spectral clustering is a relaxation by these two properties,
\[
X_{\rcut} = U U^{\top}, \quad U^{\top}U = I_k, \quad U\in\RR^{N\times k}.
\]
Therefore, one instead considers a simple matrix eigenvalue/eigenvector problem,
\begin{equation}\label{prog:spectral}
\min_{U\in\RR^{N\times k}} \lag L, UU^{\top}\rag \quad \text{s.t.} \quad U^{\top}U = I_k,
\end{equation}
whose global minimizer is easily found via computing the eigenvectors w.r.t.\ the $k$ smallest eigenvalues of the graph Laplacian $L$. Therefore, the Laplacian eigenmaps step of Algorithm~\ref{alg:spectral} has a natural explanation via  the spectral relaxation of RatioCut.

\subsubsection*{Normalized cuts and their spectral relaxation}
The normalized cut (NCut) differs from RatioCut by using the volume to quantify the size of cluster $\Gamma_l$ instead of the cardinality $|\Gamma_l|$. RatioCut and NCut behave similarly if each node of the graph has very similar degree, i.e., the graph is close to a regular graph. The NCut of a given partition $\{\Gamma_l\}_{l=1}^k$ is defined as
\begin{equation}\label{def:ncut}
\NCut(\{ \Gamma_a\}_{a=1}^k) : = \sum_{a=1}^k \frac{\cut(\Gamma_a, \Gamma_a^c)}{\Vol(\Gamma_a)}
\end{equation}
where the volume of $\Gamma_a$ is defined as the sum of degrees of vertices in the subset $\Gamma_a$, 
\begin{equation}\label{def:vol}
\Vol(\Gamma_a) : = \sum_{i\in\Gamma_a} d_i = \sum_{i\in\Gamma_a} \sum_{j=1}^N w_{ij}.
\end{equation}
Just like the link between RatioCut and the graph Laplacian, we can relate~\eqref{def:ncut} to the normalized Laplacian~\eqref{def:Lsym}. By using~\eqref{eq:cut},~\eqref{eq:rcut-L}, and $\Vol(\Gamma_a) = \lag D, 1_{\Gamma_a}1_{\Gamma_a}^{\top}\rag = 1_{\Gamma_a}^{\top}D1_{\Gamma_a}$, and then
\begin{align*}
\NCut(\{\Gamma_a\}_{a=1}^k) & = \sum_{a=1}^k \frac{1_{\Gamma_a}^{\top}L 1_{\Gamma_a}}{ 1_{\Gamma_a}^{\top}D1_{\Gamma_a} } = \sum_{a=1}^k \left\lag L, \frac{ 1_{\Gamma_a}1_{\Gamma_a}^{\top}}{ 1_{\Gamma_a}^{\top}D1_{\Gamma_a} }\right\rag  \\
& =  \sum_{a=1}^k \left\lag D^{-\frac{1}{2}}LD^{-\frac{1}{2}}, \frac{ D^{\frac{1}{2}} 1_{\Gamma_a}1_{\Gamma_a}^{\top}D^{\frac{1}{2}}}{ 1_{\Gamma_a}^{\top}D1_{\Gamma_a} }\right\rag \\
& = \lag L_{\sym}, X_{\ncut}\rag.
\end{align*}
Here $L_{\sym}$ is the normalized Laplacian in~\eqref{def:Lsym} and 
\begin{equation}\label{def:Xncut}
X_{\ncut} := \sum_{a=1}^k \frac{1}{1_{\Gamma_a}^{\top}D1_{\Gamma_a}}	D^{\frac{1}{2}}1_{\Gamma_a}1_{\Gamma_a}^{\top} D^{\frac{1}{2}}.
\end{equation}
If we replace $D$ with an identify matrix multiplied by a scalar, then $X_{\ncut}$ is equal to $X_{\rcut}$.

Similar to RatioCut, minimizing RatioCut is  an NP-hard problem and one can instead use the following convenient spectral relaxation, 
\begin{equation}\label{prog:spectral-ncut}
\min_{U\in\RR^{N\times k}} \lag L_{\sym}, UU^{\top}\rag, \quad \text{s.t.} \quad U^{\top}U = I_k,
\end{equation}
because $X_{\ncut}$ in~\eqref{def:Xncut} is also a positive semidefinite orthogonal projection matrix and thus can be factorized into $X_{\ncut} =  UU^{\top}$ with $U^{\top}U = I_k.$

\vskip0.25cm
Although it is very convenient to compute the global minimizer $U_{\rcut}$ and $U_{\ncut}$ in~\eqref{prog:spectral} and~\eqref{prog:spectral-ncut} respectively, as mentioned earlier they unfortunately do not usually return the exact cluster membership,  unless the graph has exactly $k$ connected components. Suppose there are $k$ connected components, then it is straightforward to verify
\begin{equation}\label{eq:Ucut}
U_{\rcut} = U_{\iso}, \quad U_{\ncut} = D^{\frac{1}{2}}U_{\iso} (U_{\iso}^{\top}DU_{\iso})^{-\frac{1}{2}}
\end{equation}
are the global minimizer of~\eqref{prog:spectral} and~\eqref{prog:spectral-ncut} respectively up to an orthogonal transformation. Then all columns of $U_{\rcut}$ and $D^{-\frac{1}{2}}U_{\ncut}$ are indicator vectors and they imply the connected components automatically. However, in general, the minimizer $U_{\rcut}$ and $D^{-\frac{1}{2}}U_{\ncut}$ are not in the form of~\eqref{eq:Ucut} if the graph is connected. Thus, k-means, as a rounding procedure, is applied to $U_{\rcut}$ and $D^{-\frac{1}{2}}U_{\ncut}$ to estimate the underlying clusters. Those observations  lead to Algorithm~\ref{alg:spectral} and~\ref{alg:spectral2}, respectively.

\section{SDP relaxation of graph cuts and main results} \label{s:sdp}
In this section, we propose semidefinite relaxation of spectral clustering for both ratio cuts and normalized cuts, and present the~\emph{spectral proximity condition}, which certifies the global optimality of a graph cut under either ratio cuts or normalized cuts. The spectral proximity condition is purely deterministic; it depends only on the within-cluster connectivity (algebraic connectivity of a graph) and inter-cluster connectivity. We then apply our results to spectral clustering, as a special case of graph cuts, and thus obtain the desired theoretical guarantees for spectral clustering. 

\subsection{Graph cuts via semidefinite programming} 
We add one more constraint to both programs,~\eqref{prog:spectral} and~\eqref{prog:spectral-ncut}, and the so obtained modification results in the SDP relaxation of graph cuts. 

\emph{SDP relaxation of RatioCut:}
Note that minimizing RatioCut is equivalent to minimizing $\lag L, Z\rag$ over all matrices $Z$ in the form of~\eqref{def:Xr}, which is a semidefinite  block-diagonal orthogonal projection matrix up to a row/column permutation. Since this combinatorial optimization problem is NP-hard in nature, the idea of SDP relaxation in this context is to replace the feasible matrices in the form of~\eqref{def:Xr} by a convex set which contains all such matrices as a proper subset. We first try to find out what properties matrices $Z$ in the form of~\eqref{def:Xr} have for any given partition: 
\begin{enumerate}
\item $Z$ is positive semidefinite, $Z\succeq 0$; 
\item $Z$ is nonnegative, $Z\geq 0$ entrywisely; 
\item the constant vector is an eigenvector of $Z$ which means $Z1_N = 1_N$; 
\item the trace of $Z$ equals $k$, i.e., $\Tr(Z) = k$. 
\end{enumerate}
It is obvious that the first two conditions are convex, and both conditions 3) and 4) are linear. Therefore, instead of minimizing $\lag L, Z\rag$ over all $Z$ as in~\eqref{def:Xr}, we relax the originally combinatorial optimization by using the following convex relaxation:
\begin{equation}\label{prog:primal}
\min_{Z\in {\cal S}_N} \lag L, Z\rag \quad \text{s.t.} \quad Z\succeq 0, \quad Z\geq 0, \quad \Tr(Z) = k, \quad Z1_N =1_N.
\end{equation}
In fact, if $U_{\rcut}\in\RR^{N\times k}$ is the solution to the spectral relaxation~\eqref{prog:spectral}, then $\widehat{Z} = U_{\rcut}U_{\rcut}^{\top}$ satisfies all the conditions in~\eqref{prog:primal} except for the nonnegativity condition. 

\emph{SDP relaxation of NCut:}
The partition matrix in~\eqref{def:Xncut} shares three properties with those in~\eqref{def:Xr}:
\[
Z\succeq 0, \quad Z\geq 0, \quad \Tr(Z) = k.
\]
The only difference is the appearance of the term $D^{\frac{1}{2}}1_N$ instead of $1_N$, 
\[
Z D^{\frac{1}{2}}1_N = \sum_{a=1}^k \frac{1}{1_{\Gamma_a}^{\top}D1_{\Gamma_a}} D^{\frac{1}{2}}1_{\Gamma_a}1_{\Gamma_a}^{\top} D1_N = \sum_{a=1}^k \frac{1_{\Gamma_a}^{\top} D1_N}{1_{\Gamma_a}^{\top}D1_{\Gamma_a}}D^{\frac{1}{2}} 1_{\Gamma_a} =D^{\frac{1}{2}} 1_N.
\]
As a result, the corresponding convex relaxation of normalized cuts is 
\begin{equation}\label{prog:primal-ncut}
\min_{Z\in {\cal S}_N} \lag L_{\sym}, Z\rag, \quad \text{s.t.}  \quad Z\succeq 0, \quad Z\geq 0, \quad \Tr(Z) = k,\quad ZD^{\frac{1}{2}}1_N = D^{\frac{1}{2}}1_N.
\end{equation}
Similarly, we can also see that the main difference between~\eqref{prog:primal-ncut} and~\eqref{prog:spectral-ncut} is the nonnegativity condition. 
We summarize our approach in Algorithm~\ref{SDP}. 
\begin{algorithm}[h!]
\caption{SDP relaxation of spectral clustering:~{\bf RatioCut-SDP} and~{\bf NCut-SDP}}\label{SDP}
\begin{algorithmic}[1]
\State {\bf Input:} Given a dataset $\{x_i\}_{i=1}^N$ and the number of clusters $k$, construct the weight matrix $W$ from $\{x_i\}_{i=1}^N.$
\State Compute the unnormalized graph Laplacian $L = D- W$ or its normalized graph Laplacian $L_{\sym} = I_N - D^{-\frac{1}{2}}WD^{-\frac{1}{2}}$.
\State Solve the following semidefinite programs:

\noindent a) {\bf RatioCut-SDP:}
\[
\widehat{Z} :=\text{argmin}_{Z\in {\cal S}_N} \lag L, Z\rag \quad \text{s.t.} \quad Z\succeq 0, \quad Z\geq 0, \quad  \Tr(Z) = k, \quad Z1_N =1_N.
\]
 b) {\bf NCut-SDP:}
\begin{align*}
\widehat{Z} :=\text{argmin}_{Z\in {\cal S}_N} \lag L_{\sym}, Z\rag \quad \text{s.t.}\quad & Z\succeq 0, \quad Z\geq 0, \quad \Tr(Z) = k, \quad ZD^{\frac{1}{2}}1_N =D^{\frac{1}{2}}1_N.
\end{align*}

\State Obtain the cluster partitioning based on $\widehat{Z}$.
\end{algorithmic}
\end{algorithm}

From a numerical viewpoint Algorithm~\ref{SDP}  does not lend itself easily to an efficient implementation for large scale data clustering. The question of how to solve~\eqref{prog:primal} and~\eqref{prog:primal-ncut} in a computationally efficient manner is a topic for future research. { However, one can easily run spectral clustering, and then use Theorem~\ref{thm:main} and Theorem~\ref{thm:main2}  to check if the resulting partition is optimal.} 
In this paper our focus is on getting theoretical insights into the performance of graph cuts and spectral clustering. 

Define the ground truth partition matrix $X$ as
\begin{equation}\label{def:X}
X : = 
\begin{cases}
X_{\rcut}, & \text{for RatioCut in~\eqref{def:Xr}}, \\
X_{\ncut}, & \text{for NCut in~\eqref{def:Xncut}}. 
\end{cases}
\end{equation}
Thus, the key questions we need to address are:
\begin{quote}
{\em Under which conditions does  Algorithm~\ref{SDP}  exactly recover the underlying partition $X$ in~\eqref{def:X}? 
Are these conditions approximately optimal?}
\end{quote}
As discussed above, the main difference of RatioCut-SDP and NCut-SDP from the spectral relaxation~\eqref{prog:spectral} and~\eqref{prog:spectral-ncut} comes from the nonnegativity constraint. We would like to see how this constraint in the SDP relaxation contributes to the final performance.

In fact, this relaxation is not entirely new. Xing and Jordan~\cite{XingJ03} proposed a very similar SDP relaxation for normalized k-cut by considering the nonnegativity constraint and applied the SDP relaxation to several datasets. 
Another closely related type of convex relaxation has originally been proposed by Peng and Wei for k-means-type clustering~\cite{PengW07}. There, instead of $L$ (or $L_{\sym}$), one has a matrix containing the squared pairwise Euclidean distances between data points or a similarity matrix. In recent years, theoretical guarantees of the Peng-Wei relaxation have been derived for 
k-means~\cite{IguchiMPV17,MixonVW17,LLLSW17,TepperSC17}. Furthermore, the  Peng-Wei relaxation has  been extended to community detection problems~\cite{YanSC17,AminiL18}.  Note that the presence of the graph Laplacian instead of an Euclidean distance matrix does not only substantially (and positively) affect the clustering performance, but it also significantly changes the proof strategy (and resulting conditions) in order to establish exact clustering guarantees.

\subsection{Main theorems}

Simple perturbation theory directly applied to the graph Laplacian so far has not led to competitive performance bounds. It either requires the futile assumption of
a graph with  properly disconnected components, or the results are merely of handwaving nature. While our analysis will also invoke perturation theory at some stage,
a crucial difference is that we get competitive and rigorous quantitative performance guarantees without imposing the unrealistic assumption of a disconnected graph.

In the following theorem we give a natural condition, called~\emph{spectral proximity condition}, under which Algorithm~\ref{SDP}  yields the correct clustering of the data. Both conditions in~\eqref{mainbound} and~\eqref{mainbound2} can be interpreted as a kernel-space analog of the Euclidean-space proximity condition appearing in the theoretical analysis of k-means~\cite{KumarK10,AwasthiS12,LLLSW17}.

\begin{theorem}[{\bf Spectral proximity condition for RatioCut-SDP}]\label{thm:main}
The semidefinite relaxation~\eqref{prog:primal} gives $X_{\rcut}$ in~\eqref{def:Xr} as the unique global minimizer if the following spectral proximity condition holds
\begin{equation}
\|D_{\delta}\| < \frac{\lambda_{k+1}(L_{\iso})}{4}, \label{mainbound}
\end{equation}
where $\lambda_{k+1}(L_{\iso})$ is the $(k+1)$-th smallest eigenvalue of the graph Laplacian $L_{\iso}$ defined in~\eqref{def:Liso}. Here $\lambda_{k+1}(L_{\iso})$ satisfies 
\[
\lambda_{k+1}(L_{\iso}) = \min_{1\leq a\leq k} \lambda_2(L^{(a,a)}_{\iso})
\]
where $\lambda_2(L^{(a,a)}_{\iso})$ is the second smallest eigenvalue of graph Laplacian w.r.t. the $a$-th cluster.
\end{theorem}

As pointed out in~\cite{NgJW02}, the success of spectral clustering depends on the within-cluster connectivity (algebraic connectivity, which is captured by $\lambda_2(L_{\iso}^{(a,a)})$), as well as the ``noise" $\|D_{\delta}\|$ which measures the inter-cluster connectivity. If the latter quantity is close to 0,  spectral clustering should succeed, because the eigenspace of $L$ w.r.t.\ the smallest $k$ eigenvalues will be  close to $U_{\iso}$. 
Our condition~\eqref{mainbound} makes the intuition behind~\cite{NgJW02} precise. Note that the operator norm of $D_{\delta}$ equals 
\[
\|D_{\delta}\| = \|W_{\delta}1_{N}\|_{\infty} =\max_{1\leq a\leq k} \max_{i\in\Gamma_a} \sum_{j\notin  \Gamma_a}w_{ij},
\]
which quantifies the maximal inter-cluster degree. If this quantity is smaller than the within-cluster connectivity $\lambda_2(L_{\iso}^{(a,a)})$ (modulo a constant factor), then convex relaxation of  RatioCut is able to find the underlying partition~\emph{exactly}. 
\vskip0.25cm

For the SDP relaxation of the normalized cuts, we have the following theorem under slightly different conditions.
\begin{theorem}[{\bf Spectral proximity condition for NCut-SDP}]\label{thm:main2}
The semidefinite relaxation~\eqref{prog:primal-ncut} gives $X_{\ncut}$ in~\eqref{def:Xncut} as the unique global minimizer if the following spectral proximity condition holds
\begin{equation}
\frac{\|P_{\delta}\|_{\infty}}{1- \|P_{\delta}\|_{\infty}} < \frac{\lambda_{k+1}(L_{\rw, \iso})}{4 }, \label{mainbound2}
\end{equation}
where $\lambda_{k+1}(L_{\rw, \iso})$ is the $(k+1)$-th smallest eigenvalue of $L_{\rw,\iso}$.
Moreover, $\lambda_{k+1}(L_{\rw, \iso})$ satisfies
\[
\lambda_{k+1}(L_{\rw,\iso}) = \min_{1\leq a\leq k} \lambda_2(L_{\rw, \iso}^{(a,a)} ) = \min_{1\leq a\leq k} \lambda_2(I_{n_a} - P_{\iso}^{(a,a)})
\]
due to the block-diagonal structure of $L_{\rw,\iso}$ and $P_{\iso}$ in~\eqref{def:Piso}.
\end{theorem}
The condition~\eqref{mainbound2} has a probabilistic interpretation. Note that $P = D^{-1}W$ is a Markov transition matrix in~\eqref{def:P}, $P_{\delta}$ consists of the off-diagonal blocks of $P$ in~\eqref{def:delta}, and $\|P_{\delta}\|_{\infty}$ is the maximal probability of a random walker leaving its own cluster after one step. Thus, if the left hand side in~\eqref{mainbound2} is small, for example less than $1$, it means a random walker starting from~\emph{any} node is~\emph{more} likely to stay in its own cluster than leave it after one step, and vice versa. In other words, the left hand side of~\eqref{mainbound2} characterizes the strength of inter-cluster connectivity. On the other hand, the right hand of $~\eqref{mainbound2}$ equals $\lambda_2(I_{n_a} - P_{\iso}^{(a,a)})$ which is the eigengap\footnote{The eigengap refers to the difference between the first and the second largest eigenvalues of the Markov transition matrix.} of the Markov transition matrix for the random walk restricted to the $a$-th cluster. It is well known that a larger eigengap implies stronger connectivity of each individual cluster as well as faster mixing time~\cite{LevinP17} of the Markov chain defined on $a$-th cluster. 
The matrix $P = D^{-1}W$ plays also a central role in the diffusion map framework~\cite{Coifman06}. 
Thus, our approach paves the way to derive theoretical guarantees for clustering based on diffusion maps.

While the convex relaxation approach to k-means leads to conditions that are directly expressible as separation conditions between clusters in terms of Euclidean distances, this is not the case in Theorem~\ref{thm:main} and Theorem~\ref{thm:main2}, nor should one expect this for general clusters. After all, the whole point of resorting to spectral clustering is that one may have to cluster datasets which are not neatly separated by the Euclidean distance, see e.g.\ the example in Section~\ref{ss:circles}. 
It is gratifying to note that the bounds in~\eqref{mainbound} and~\eqref{mainbound2} are independent of the number of clusters, $k$. This should be compared to known theoretical bounds for SDP relaxation of k-means clustering which have the undesirable property that they do depend on the number of clusters.

\vskip0.25cm
Theorems~\ref{thm:main} and~\ref{thm:main2} do not only apply to spectral clustering but also to graph cuts. The attentive reader may have noticed that Theorems~\ref{thm:main} and~\ref{thm:main2} do not rely on any information of a data generative model of the underlying clusters or  on the choice of kernel function $\Phi(\cdot)$. Instead, the assumptions in both theorems are purely algebraic conditions which only depend on the spectral properties of the graph Laplacian. Thus these two results not only apply to spectral clustering but also to general graph partition problems. 
Suppose we have an undirected graph with weight matrix $W$ (not necessarily in the form of~\eqref{def:W}) and compute the corresponding graph Laplacian $L.$ We try to partition the graph into several subgraphs such that RatioCut or NCut is minimized. Then if a given partition $\{\Gamma_a\}_{a=1}^k$ (any partition $\{\Gamma_a\}_{a=1}^k$ gives rise to $W_{\iso}$ and $L_{\iso}$) satisfies~\eqref{mainbound} or~\eqref{mainbound2}, then $\{\Gamma_a\}_{a=1}^k$ is the only global minimizer of RatioCut or NCut respectively. Moreover, this partition can be found via the SDP relaxation~\eqref{prog:primal} and~\eqref{prog:primal-ncut}.

As a result, Theorem~\ref{thm:main} and~\ref{thm:main2} also yield performance bounds for successful community detection under stochastic block model with multiple communities~\cite{Abbe17,Agarwal17,AminiL18,Bandeira18} because the community detection problem is an important example 
of {the graph partition problem}. We apply Theorem~\ref{thm:main} to the stochastic block model and present the corresponding performance bound in Section~\ref{ss:sbm}. However, the bounds obtained here will not be as tight as those found in the state-of-the-art literature~\cite{EmmanuelBH16} (by a factor of constant). {The main reason is that our derivation of Theorem~\ref{thm:main} does not assume there are exactly two clusters of equal size~\cite{EmmanuelBH16}. } 

\section{Near-optimality of spectral proximity condition}\label{s:bounds}

It is natural to ask whether the semidefinite relaxation of spectral clustering can achieve better results than ordinary k-means. 
In this section we will demonstrate by means of concrete examples that our framework can indeed achieve near-optimal clustering performance. The first two examples are deterministic examples in which the data are placed on two concentric circles or two parallel lines. {Those two examples,  where the planted clusters are either highly anisotropic or not linearly separable, are often cited to demonstrate better performance of spectral clustering over that of ordinary k-means.}
However, to the best of our knowledge, rigorous theoretic performance analysis of spectral clustering on these examples is still lacking. We will apply Theorem~\ref{thm:main} to show that the SDP relaxation of spectral clustering will work with guarantees while, on the other hand, k-means fails. 

The key ingredient to invoke Theorem~\ref{thm:main} is the estimation of the second smallest eigenvalue of the graph Laplacian associated with each cluster. While we are able to show the estimation of this quantity for deterministic examples, it is more appealing to find out a framework to compute the algebraic connectivity of graph Laplacians with data generated from a probability distribution on a manifold. This is an important mathematical problem by itself and we will discuss it briefly in Section~\ref{ss:laplacian}. In Section~\ref{ss:sbm}, we apply Theorem~\ref{thm:main} to stochastic block model and compare our performance bound with the state-of-the-art results.

\subsection{Two concentric circles}\label{ss:circles}
We first present an example in which k-means clustering obviously must fail, but spectral clustering is known to
succeed empirically, cf.~Figure~\ref{fig:twocircles}. While this example is frequently used to motivate the use of spectral clustering, 
kernel k-means, or diffusion
maps over  standard k-means, so far this motivation was solely based on empirical evidence, since until now no
theoretical guarantees have been given to justify it.
We will give an explicit condition derived from Theorem~\ref{thm:main} under which  Algorithm~\ref{SDP} is
provably able to recover the underlying clusters exactly and in addition it can do so at a nearly minimal cluster separation, 
thereby putting this popular empirical example finally on firm theoretical ground.

Suppose we have two circles centered at the origin. The data are equispaced on the circles, i.e., 
\begin{equation}\label{eq:model-2c}
x_{1,i} = r_1 \begin{bmatrix}\cos(\frac{2\pi i}{n}) \\ \sin(\frac{2\pi i}{n}) \end{bmatrix}, ~1\leq i\leq n; \qquad x_{2,j} =
r_2\begin{bmatrix}\cos(\frac{2\pi j}{m}) \\ \sin(\frac{2\pi j}{m}) \end{bmatrix},~1\leq j\leq m
\end{equation}
where $m = \lfloor n\kappa\rfloor$ and $\kappa =
\frac{r_2}{r_1} > 1$. The parameters are chosen so that the distance between adjacent points in each individual
cluster is approximately $\frac{2\pi r_1}{n}$. In this example, we pick the Gaussian kernel $\Phi_{\sigma}(x,y) = e^{-\frac{\|x-y\|^2}{2\sigma^2}}$ to construct weight matrix and graph Laplacian.

\begin{figure}[h!]
\centering
\begin{minipage}{0.48\textwidth}
\subcaptionbox{Result via Matlab's built-in k-means++ and ``$*$" stands for the centroids of clusters.}{
\includegraphics[width=60mm]{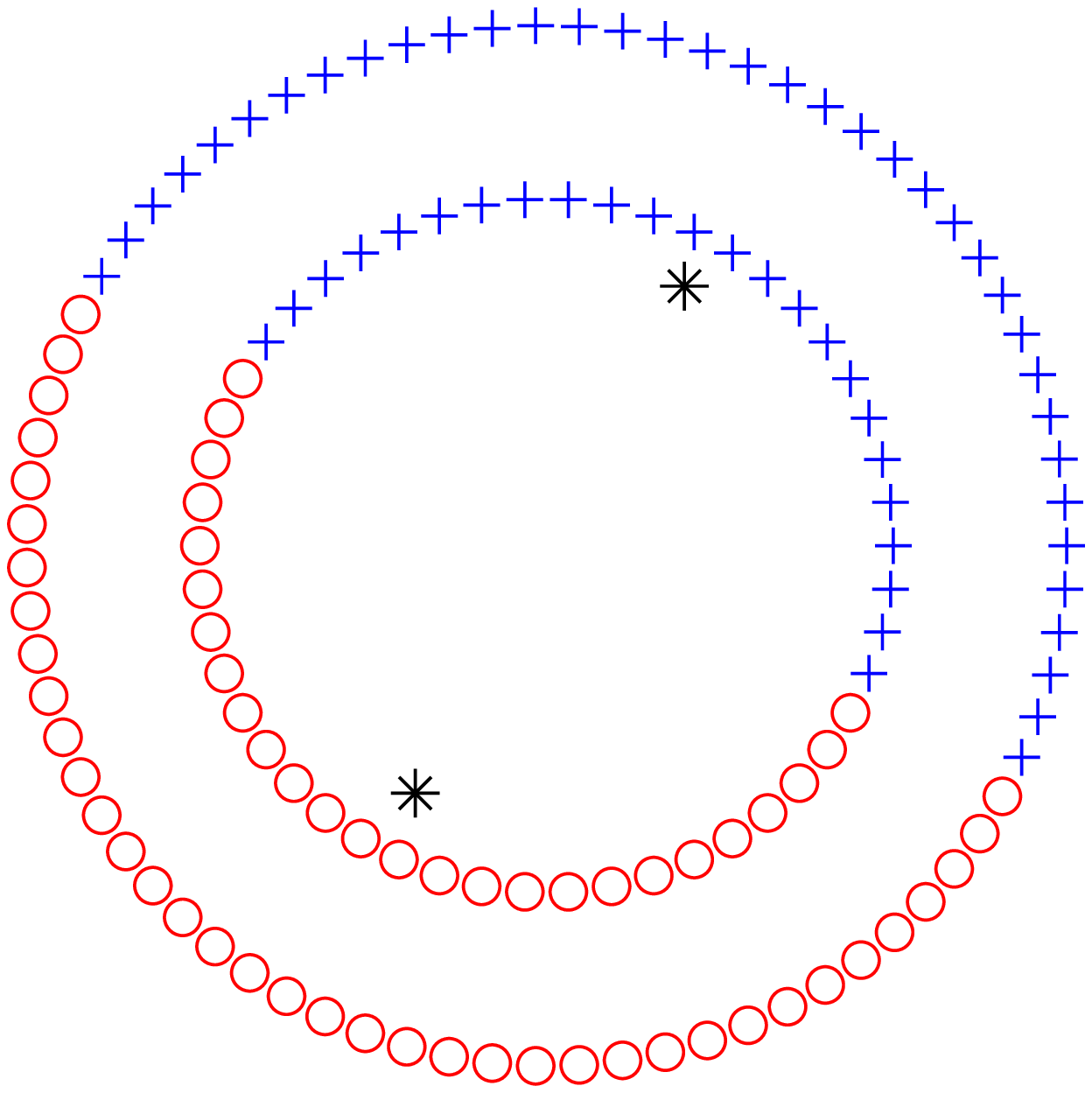}}
\end{minipage}
\hfill
\begin{minipage}{0.48\textwidth}
\subcaptionbox{Result via SDP relaxation of spectral clustering}{
\includegraphics[width=60mm]{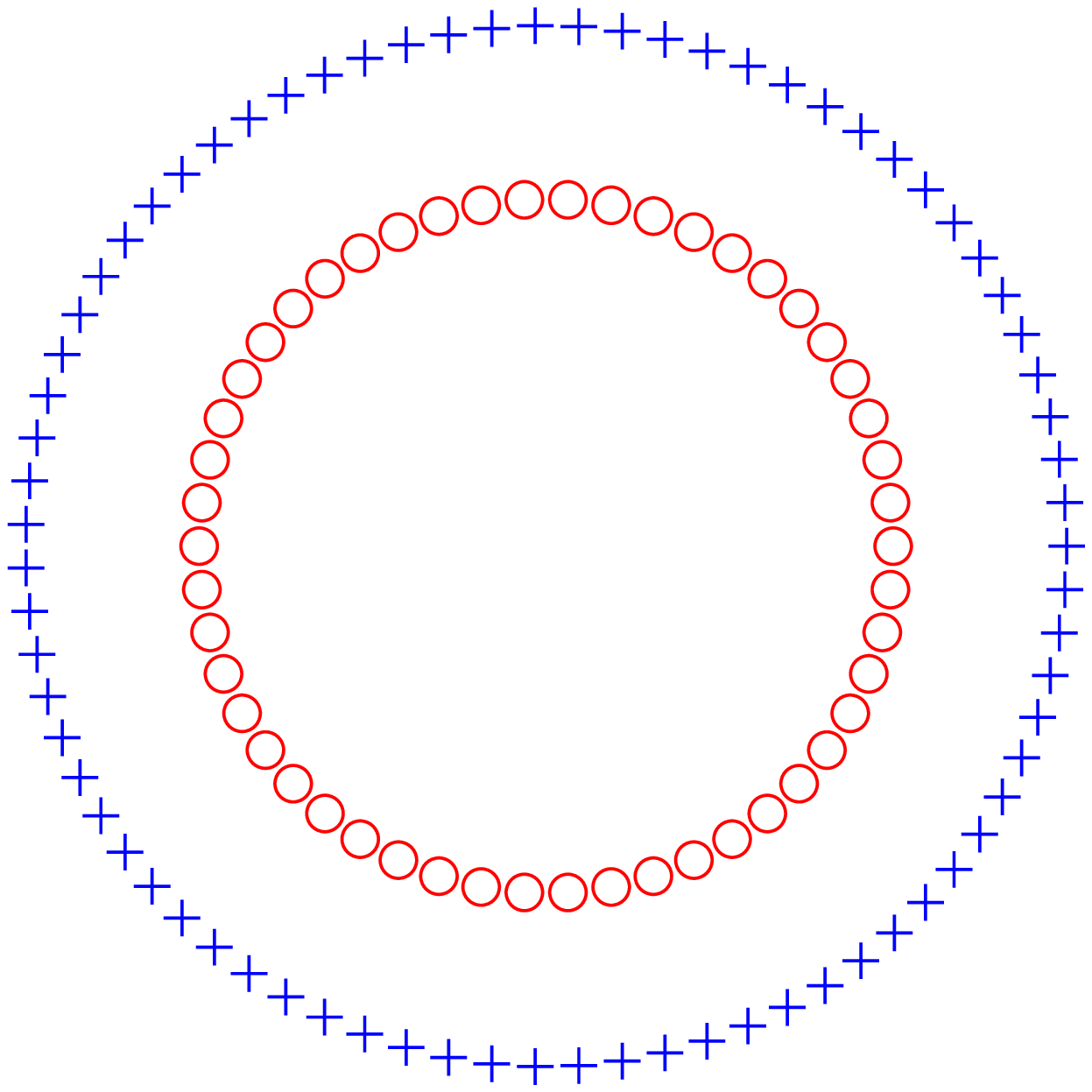}}
\end{minipage}
\caption{Two concentric circles with radii $r_1 = 1$ and $r_2 = \frac{3}{2}$, and $n = 50$ and $m = 75$. (a) If one performs Matlab's built-in
kmeans++, the underlying clusters of the data will obviously not be extracted since two clusters are not linearly separable; (b) If one performs SDP relaxation of spectral clustering with bandwidth $\sigma = \frac{4}{n}$, the two circles are recovered exactly.
}
\label{fig:twocircles}
\end{figure}

\begin{theorem}\label{thm:circles}
Let the data $\{x_i\}_{i=1}^{n+m}$  be given by $\{x_{1,i}\}_{i=1}^n \cup \{x_{2,i}\}_{i=1}^m$ as defined in~\eqref{eq:model-2c} and let the ratio of the two radii be $\kappa = \frac{r_2}{r_1}$ and 
$\Delta := \frac{r_2 - r_1}{r_1} = \kappa-1$. Let $\Phi$ be the heat kernel with $\sigma$  chosen as follows
\[
\sigma^2 = \frac{16r_1^2 \gamma}{n^2 \log(\frac{m}{2\pi})}.
\] 
Then Algorithm~\ref{SDP} recovers the underlying two clusters  $\{x_{1,i}\}_{i=1}^n$ and  $\{x_{2,i}\}_{i=1}^m$ exactly if the separation $\Delta$ satisfies
\begin{equation}\label{delta_circle}
 \Delta \geq \frac{4}{n}\sqrt{ 1+ 2\gamma\left(2 + \frac{\log (4 m)}{\log (\frac{m}{2\pi})}\right)}.
\end{equation}

\end{theorem}

To see that the separation $\Delta$ in Theorem~\ref{thm:circles} is $\widetilde{\cal O}$-optimal, assume w.l.o.g.\ $r_1=1$. In this case the minimum distance 
between points in the same circle is about $\frac{2\pi}{n}$. Therefore, we can only expect spectral clustering to recover the two circle clusters correctly if the minimum distance between points of  different circles is larger than $\frac{2\pi}{n}$, i.e. $r_2 \ge 1+\frac{2\pi}{n}$. 
Indeed, the condition in~\eqref{delta_circle} shows that a separation
$\Delta = \widetilde{\cal O}(\frac{1}{n})$ suffices for successful recovery of the two clusters.

\subsection{Two parallel lines} \label{ss:lines}

Here is another example showing the limitation of k-means, even though the two clusters are perfectly within convex boundaries.
The issue here is that the two clusters are highly anisotropic, which is a major problem for k-means\footnote{It is clear that in the given example simple rescaling of the data would make them more isotropic, but this is not the point we try to illustrate. Also, in more involved examples  consisting of anisotropic clusters of different orientation, rescaling or resorting e.g.\ to the Mahalanobis distance instead of the Euclidean distance will not really overcome the sensibility of k-means to ``geometric distortions''.}.

Suppose the data points are distributed on two lines with separation $\Delta$  as illustrated in Figure~\ref{fig:twolines}, 
\begin{equation}\label{eq:distlines}
x_{1,i} = \begin{bmatrix}
-\frac{\Delta}{2} \\
\frac{i-1}{n-1} 
\end{bmatrix}, \quad 
x_{2,i} = \begin{bmatrix}
\frac{\Delta}{2} \\
\frac{i-1}{n-1} 
\end{bmatrix}, \quad 1\leq i\leq n,
\end{equation}
where $n$ is the number of data points on each time and there are $2n$ points in total.

We claim that if $n$ is large and $\Delta < \frac{1}{2}$, the k-means optimal solution will not
return the underlying partition, as  suggested by the following calculations. For simplicity, we also assume $n$ as an even number.
If we set the cluster centers to be $c_1 = [- \frac{\Delta}{2}\,\,\, \frac{1}{2}]^\top$ and $c_2 = [ \frac{\Delta}{2}\,\,\, \frac{1}{2}]^\top$, as the geometry suggests, 
then the k-means objective function value is 
\begin{equation*}
\Psi_n(c_1, c_2)   = 2\sum_{i=1}^n \left( \frac{i-1}{n-1} - \frac{1}{2}\right)^2
 = \frac{n(2n-1)}{3(n-1)} - \frac{n}{2} 
\end{equation*}
which follows from $\sum_{i=1}^n i^2 = \frac{n(n+1)(2n+1)}{6}.$
As $n$ goes to infinity, the average objective function value over $2n$ points becomes
\[
\lim_{n\rightarrow \infty}\frac{\Psi_n(c_1,c_2)}{2n} = \frac{1}{3} - \frac{1}{4}=  \frac{1}{12}.
\]
However, if we pick the cluster centers as $c_1 = [ 0 \,\,\, \frac{n-2}{4(n-1)}]^{\top}$ and $c_2= [ 0 \,\,\, \frac{3n-2}{4(n-1)}]^{\top}$, then 
\begin{align*}
\Psi_n(c_1,c_2)& = 4\sum_{i=1}^{\frac{n}{2}} \left(\frac{\Delta^2}{4} + \left( \frac{i-1}{n-1} -
\frac{n-2}{4(n-1)}\right)^2\right)  
= \frac{n\Delta^2}{2} + \frac{n(n-2)}{6(n-1)} - \frac{n(n-2)^2}{8(n-1)^2}.
\end{align*}
The limit of average objective function value in this case is
\[
\lim_{n\rightarrow \infty}\frac{\Psi_n(c_1,c_2)}{2n} = \frac{\Delta^2}{4} + \frac{1}{48}.
\]
If $\Delta < \frac{1}{2}$, the second case gives a smaller objective function value. Thus, k-means must fail to recover the two clusters
if $\Delta < \frac{1}{2}$ and $n$ is large.

However, the SDP relaxation of spectral clustering will not have this issue as demonstrated by the following theorem:
\begin{figure}[h!]
\centering
\begin{minipage}{0.45\textwidth}
\subcaptionbox{Clustering via $k$-means. The centroids are placed at $c_1 \approx [0 \,\,\, \frac{1}{4}]^{\top}$ and $c_2\approx [0\,\,\, \frac{3}{4}]^{\top}$.}{\includegraphics[width=75mm]{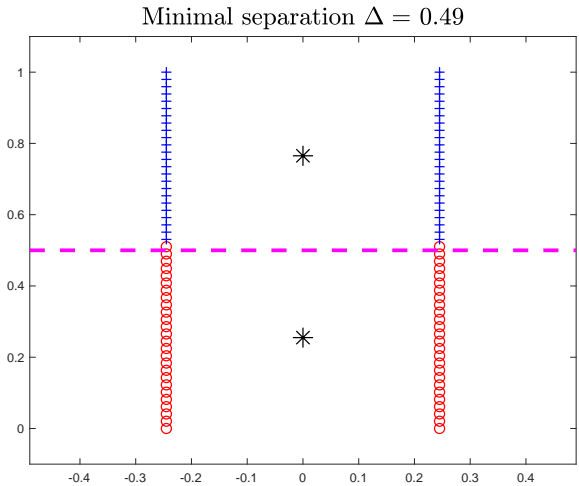}}
\end{minipage}
\hfill
\begin{minipage}{0.45\textwidth}\subcaptionbox{Clustering via Algorithm~\ref{SDP}. The centroids are placed at $c_1 = [-\frac{\Delta}{2} \,\,\, \frac{1}{2}]^{\top}$ and $c_2 =  [\frac{\Delta}{2}\,\,\, \frac{1}{2}]^{\top}$.}{
\includegraphics[width=75mm]{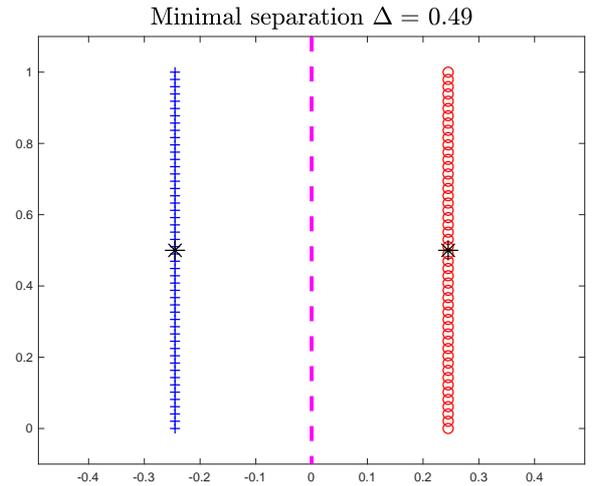}}
\end{minipage}
\caption{The lines are separated by $\Delta = 0.49$ and 50 points are equi-spaced on the unit interval. The
objective function values of the two scenarios above are approximately 8.1787 and 8.6735  for (a) and (b) respectively. In this case,
the centroids in the left plot  gives a smaller k-means objective function value. Hence, the
k-means criterion is unable to disentangle the two linearly separable manifolds if the two clusters are not well separated. }
\label{fig:twolines}
\end{figure}

\begin{theorem}\label{thm:lines}
Let the data $\{x_i\}_{i=1}^{2n}$ be given by $\{x_{1,i}\}_{i=1}^n \cup \{x_{2,i}\}_{i=1}^n$ as defined in~\eqref{eq:distlines}. 
Let $\Phi$ be the heat kernel  with bandwidth
\[
\sigma^2 = \frac{\gamma }{(n-1)^2 \log (\frac{n}{\pi})}, \qquad \gamma >0.
\]
Assume the separation $\Delta$ satisfies
\[
\Delta \geq \frac{1}{n-1} \sqrt{1 + \frac{6\gamma \log n }{ \log (\frac{n}{\pi})}}.
\]
Then Algorithm~\ref{SDP} recovers the underlying two clusters  $\{x_{1,i}\}_{i=1}^n$ and  $\{x_{2,i}\}_{i=1}^n$ exactly.
\end{theorem}

The separation distance in Theorem~\ref{thm:lines} is nearly optimal, since the distance between adjacent points within a cluster
is about $\frac{1}{n}$ and the distance between the clusters for which Algorithm~\ref{SDP}
is guaranteed to return the correct clustering is $\Delta = \widetilde{\cal O}(\frac{1}{n})$.

\subsection{Examples with random data and open problems}\label{ss:laplacian}
From a practical viewpoint it can be more appealing to consider random data instead of deterministic examples discussed above. However, in general, it is not an easy task to control the lower bound of the graph Laplacian from random data that are sampled from a probability density function supported on a manifold. Several factors will influence the spectrum of the graph Laplacian, e.g., the number of data points, the geometry of the manifold (shape, volume, connectivity, dimension, etc), the properties of probability density function, and the choice of kernel function $\Phi(\cdot)$ (w.l.o.g. we assume $\Phi$ is normalized, i.e., $\int \Phi(z)\diff z = 1.$) and its parameters, such as the bandwidth $\sigma.$ We propose the following open problem and point out one possible solution. 
\begin{open}
Suppose there are $n$ data points drawn from a probability density function $p(x)$ supported on a manifold ${\cal M}$. 
How can we estimate the second smallest eigenvalue of the graph Laplacian (either normalized or unnormalized) given the kernel function $\Phi$ and $\sigma$? 
\end{open}

In fact, numerous connections exist between graph Laplacians and Laplace-Beltrami operators on the manifold~\cite{Trillos18,TrillosS16,SingerW16,Singer06,BelkinN05,Chung97}. 
Let $L$ be the graph Laplacian constructed from $\{x_i\}_{i=1}^n$ sampled from a probability density function $p(x)$ supported on a Riemannian manifold ${\cal M}$ with/without boundary. 
Define the weighted Laplace-Beltrami operator $\Delta_{{\cal M}}$ on ${\cal M}$ as
\[
\Delta_{{\cal M}}(f) : = -\frac{1}{p}\text{div}(p^2 \nabla f)
\]
where the divergence operator $``\text{div}"$ and gradient $``\nabla"$ are defined according to the Riemannian metric, cf~\cite{DoCarmo92}. The pointwise convergence of the graph Laplacian $L$ to $\Delta_{{\cal M}}$ as well as the convergence of the normalized graph Laplacian have been discussed in several excellent works such as~\cite{BelkinN05,Singer06,Coifman06}.

From our discussion in Section~\ref{s:sdp}, one may have realized that the more relevant convergence of the graph Laplacian is spectral convergence: the convergence of the spectra of the graph Laplacian to those of its continuous limit and more importantly,~\emph{the convergence rate}. We make it more precise here: for the differential operator $\Delta_{\cal M}$, one considers the eigenvalue/eigenfunction problem with Neumann boundary condition:
\begin{equation}\label{eq:neumann}
\Delta_{{\cal M}} f=\lambda f \quad \text{in } {\cal M}, \qquad \frac{\pa f}{\pa \bn} = 0 \quad \text{on } \pa {\cal M}
\end{equation}
where $\bn$ is the normal vector and $\pa {\cal M}$ is the boundary of ${\cal M}$. In particular, this problem reduces to an eigenvalue/eigenfunction problem if the manifold has no boundary. 

We let $\lambda_2(\Delta_{\cal M})$ be the second smallest eigenvalue to~\eqref{eq:neumann}. 
It has been shown in~\cite{TrillosS16} that if ${\cal M}$ is an open, bounded, and connected domain in $\RR^m$ with $m\geq 2$  and $\sigma = \widetilde{{\cal O}}\left( \left( \frac{\log n}{n}\right)^{\frac{1}{2m}}\right)$, the rescaled second smallest eigenvalue $\frac{2}{n\sigma^2}\lambda_2(L)$  will converge to $\eps_{\Phi}\lambda_2(\Delta_{\cal M})$ almost surely when $n$ gets larger where  $\eps_{\Phi}$ represents the surface tension\footnote{Surface tension is defined as $\eps_{\Phi} = \int_{R^m}|z^{(1)}|^2\Phi(z) \diff z $ where $z^{(1)}$ is the first component of $z$.}. Similar results also hold for the normalized graph Laplacian as shown in~\cite{TrillosS16}. 
Moreover,~\cite{SingerW16} has extended the spectral convergence from graph Laplacians to connection Laplacians. 

If one knows $\lambda_2(\Delta_{{\cal M}})$ for certain simple but important cases such as a line segment or a circle equipped with uniform distribution $p(x)$, it is possible to get an estimate of $\lambda_2(L)$ via $\lambda_2(\Delta_{{\cal M}})$ and obtain the performance guarantee of spectral clustering SDP from Theorems~\ref{thm:main} and~\ref{thm:main2}. A rigorous justification of this connection relies on the spectral convergence rate of the graph Laplacian to the Laplacian eigenvalue problem with Neumann boundary condition, which is still missing, to the best of our knowledge.  Under proper conditions,~\cite{Trillos18} gives the spectral convergence rate of ${\cal O}\left(\frac{\log n}{n}\right)^{\frac{1}{2m}}$ for the graph Laplacian to converge to the Laplace-Beltrami operator on a Riemannian manifold ${\cal M}$. However, the kernel function $\Phi$ has a compact support which the heat kernel does not satisfy, and more severely,  the manifold there is assumed to have no boundary. Thus we give another open problem, the solution of which will lead to a better and more complete understanding of SDP relaxation of spectral clustering for random data.
\begin{open}
Assume $n$ data points are sampled independently from a probability density function $p(x)$ supported on ${\cal M}$ and construct a graph Laplacian $L$ with kernel function $\Phi(\cdot)$ with the size of neighborhood $\sigma$.
What is the spectral convergence rate of the graph Laplacian to the Laplacian eigenvalue problem with Neumann boundary condition in~\eqref{eq:neumann}?
\end{open}

\subsection{Stochastic block model}
\label{ss:sbm}
The stochastic block model has been studied extensively as an example of community detection problem in the recent few years ~\cite{EmmanuelBH16,Bandeira18,Abbe17,Agarwal17,AminiL18,YanSC17}. Here we treat community detection problem under the stochastic block model  as a special case of {the graph partition problem}. Let us quickly review the basics of {the stochastic block model}. Assume there are two communities and each of them has $n$ members and in total $N = 2n$ members. The adjacency matrix is a binary random matrix whose entries are given as follows, 
\begin{enumerate}
\item if member $i$ and $j$ are in the same community, $\Pr(w_{ij} = 1) = p$ and $\Pr(w_{ij} = 0) = 1-p$;
\item if member $i$ and $j$ are in different communities, $\Pr(w_{ij} = 1) = q$ and $\Pr(w_{ij} = 0) = 1-q$.
\end{enumerate}
Here $w_{ij} = w_{ji}$ and each $w_{ij}$ is independent. 
We assume $p > q$ so that the connectivity within each individual community is stronger than that between different communities.
The core question regarding the stochastic block model is to study when we are able to recover the underlying community exactly. Remarkable progress have been made by analyzing different types of convex relaxation and many performance bounds have been obtained so far. Interested readers may refer to the literature mentioned above for more details. Here we provide our performance bound in terms of $p$ and $q$ as an application of our theory to the stochastic block model. 

\begin{theorem}\label{thm:sbm}
Let $p = \frac{\alpha\log N}{N}$ and $q = \frac{\beta\log N}{N}$. The RatioCut-SDP~\eqref{prog:primal} recovers the underlying communities exactly if
\[
\alpha > 26\left(  \frac{1}{3} + \frac{\beta}{2} + \sqrt{\frac{1}{9} + \beta} \right)
\]
with high probability.
\end{theorem}
We defer the proof of this theorem to Section~\ref{ss:proof-sbm}.
Compared with the state-of-the-art results such as~\cite{EmmanuelBH16,Bandeira18} where $\sqrt{\alpha} - \sqrt{\beta} > \sqrt{2}$ is needed for exact recovery, our performance bound is slightly looser by a constant factor. 
 The near-optimal performance guarantee given by our analysis is not entirely surprising. As pointed out in~\cite{Bandeira18}, the Goemans-Williamson type of SDP relaxation\footnote{Here, the Goemans-Williamson type of SDP relaxation is given by $\max \Tr( (2W - (1_N1_N^{\top} -I_N ))Z )$, s.t. $Z\succeq 0$ and $Z_{ii} = 1$. Note this relaxation is designed specifically for the case of two clusters.} succeeds if
\[
\lambda_2(D_{\iso} - D_{\delta} - W + \frac{1}{2} 1_{N}1_{N}^{\top}) > 0.
\]
In fact, the condition above is implied by
\begin{equation}\label{eq:gw}
\min \lambda_2(L_{\iso}^{(a,a)}) > 2\|D_{\delta}\|
\end{equation}
which differs from our Theorem~\ref{thm:main} only by a factor of 2. We present the proof of this claim~\eqref{eq:gw} in Section~\ref{ss:proof-sbm}.

\section{Numerical explorations}\label{s:numerics}
In this section, we present a few examples to complement our theoretic analysis. One key ingredient in Theorem~\ref{thm:main} is the estimation of the second smallest eigenvalue of the graph Laplacian of each individual cluster. In general, it is not easy to estimate this quantity, especially for random instances. Therefore, we turn to certain numerical simulations to see if the~\emph{spectral proximity condition} holds for the data drawn from an underlying distribution supported on a manifold. We are in particular interested in the performances under different choices of minimal separation $\Delta$ and bandwidth $\sigma.$ In general, the larger $\sigma$ gets,  the stronger the within and inter-cluster connectivity are. Thus $\sigma$ cannot be arbitrarily large (just think about the extreme case $\sigma = \infty$ and the whole graph turns into a complete graph with equal edge weight); on the other hand, it is easier to pick a proper $\sigma$  if the minimal separation $\Delta$ is larger. The rule of thumb of choosing $\sigma$ is to increase the within-cluster connectivity while controlling the inter-cluster connectivity.
We will also compare those numerical results with ordinary k-means (or k-means SDP) and demonstrate the advantage of spectral clustering. 
\subsection{Two concentric circles}
In Section~\ref{s:bounds}, we discuss two deterministic examples in which k-means fails to recover the underlying partition, as well as the conditions under which~\eqref{prog:primal} and~\eqref{prog:primal-ncut} succeed. Here we run numerical examples for their corresponding random instances and see how~\eqref{mainbound} and~\eqref{mainbound2} work for RatioCut-SDP and NCut-SDP relaxation respectively.

In the first example, we assume the data are uniformly distributed on two concentric circles with radii $r_1=1$ and $r_2=1 + \Delta$. We sample $n=250$ and $m = \lfloor250(1+\Delta)\rfloor$ for these two circles respectively so that the distance between two adjacent points on each circle is approximately ${\cal O}(\frac{1}{n}).$ For each pair of $(\Delta, \sigma)$, we run 50 experiments to see how many times the condition~\eqref{mainbound} and~\eqref{mainbound2} hold respectively. More precisely, in the RatioCut-SDP, we compute the second smallest eigenvalue of Laplacian for each individual circle and $\|D_{\delta}\|$, and we treat the recovery is successful if~\eqref{mainbound} satisfies. Similar procedures are performed for NCut-SDP and count how many instances satisfy~\eqref{mainbound2}.

The size of the neighborhood $\sigma$ is chosen as $\sigma = \frac{p}{n}$ with the horizontal parameter $p$ in Figure~\ref{fig:toy1} varying from $1$ to $25$; we test different values for the minimal separation $\Delta$ between two circles. The results are illustrated in Figure~\ref{fig:toy1}: the performances of RatioCut-SDP and NCut-SDP are quite similar. 
If $\Delta\geq 0.2 $ and $5 \leq p\leq 75\Delta - 8$, then exact recovery is guaranteed with high probability. The distance between two adjacent points on one circle is approximately $\frac{2\pi}{n} \approx 0.025$ and our theorem succeeds if the minimal separation $\Delta$ is about 8 times larger than the ``average" distance between adjacent points within one cluster.

\begin{figure}[h!]
\centering
\begin{minipage}{0.48\textwidth}
\includegraphics[width=75mm]{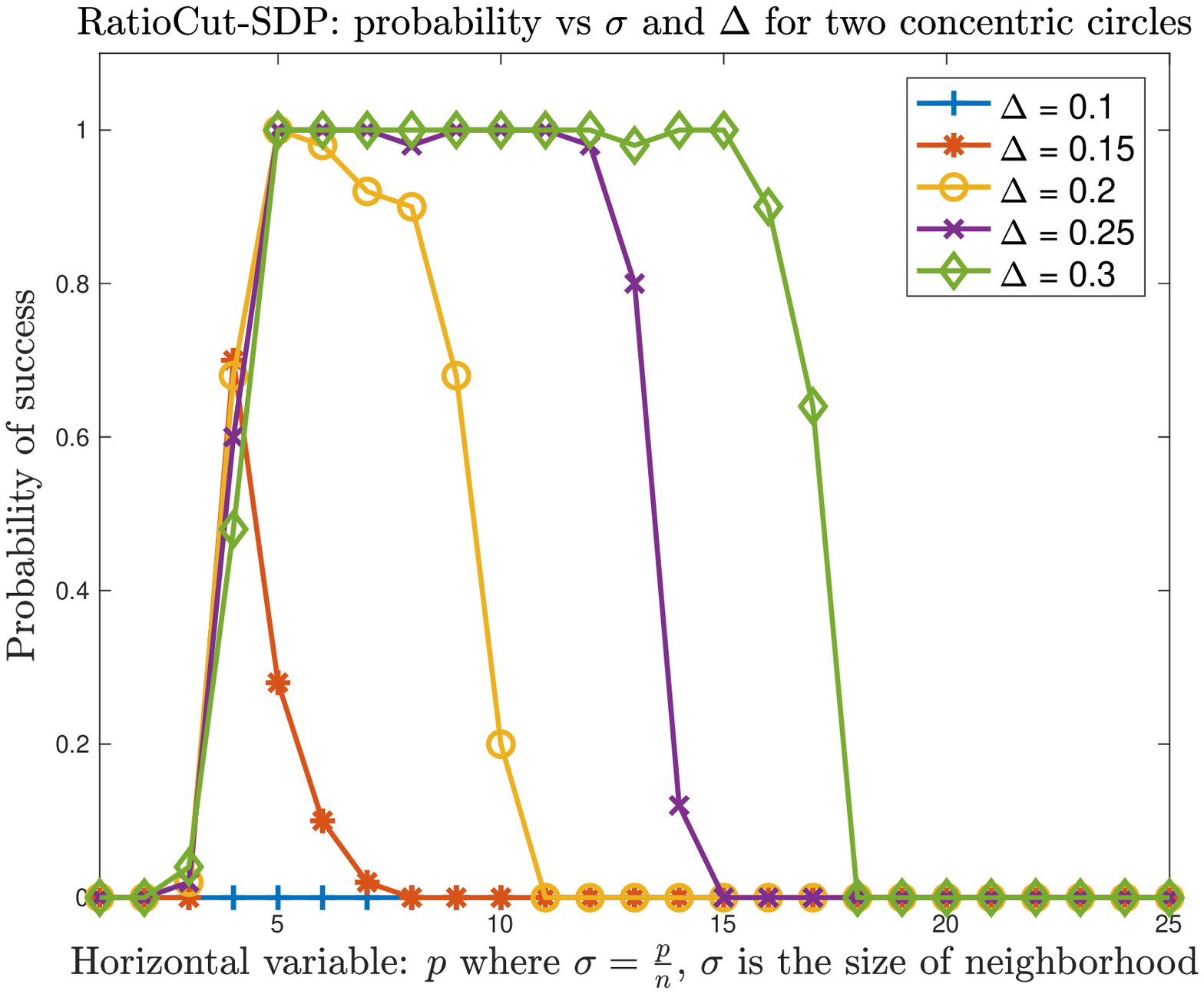}
\end{minipage}
\hfill
\begin{minipage}{0.48\textwidth}
\includegraphics[width=75mm]{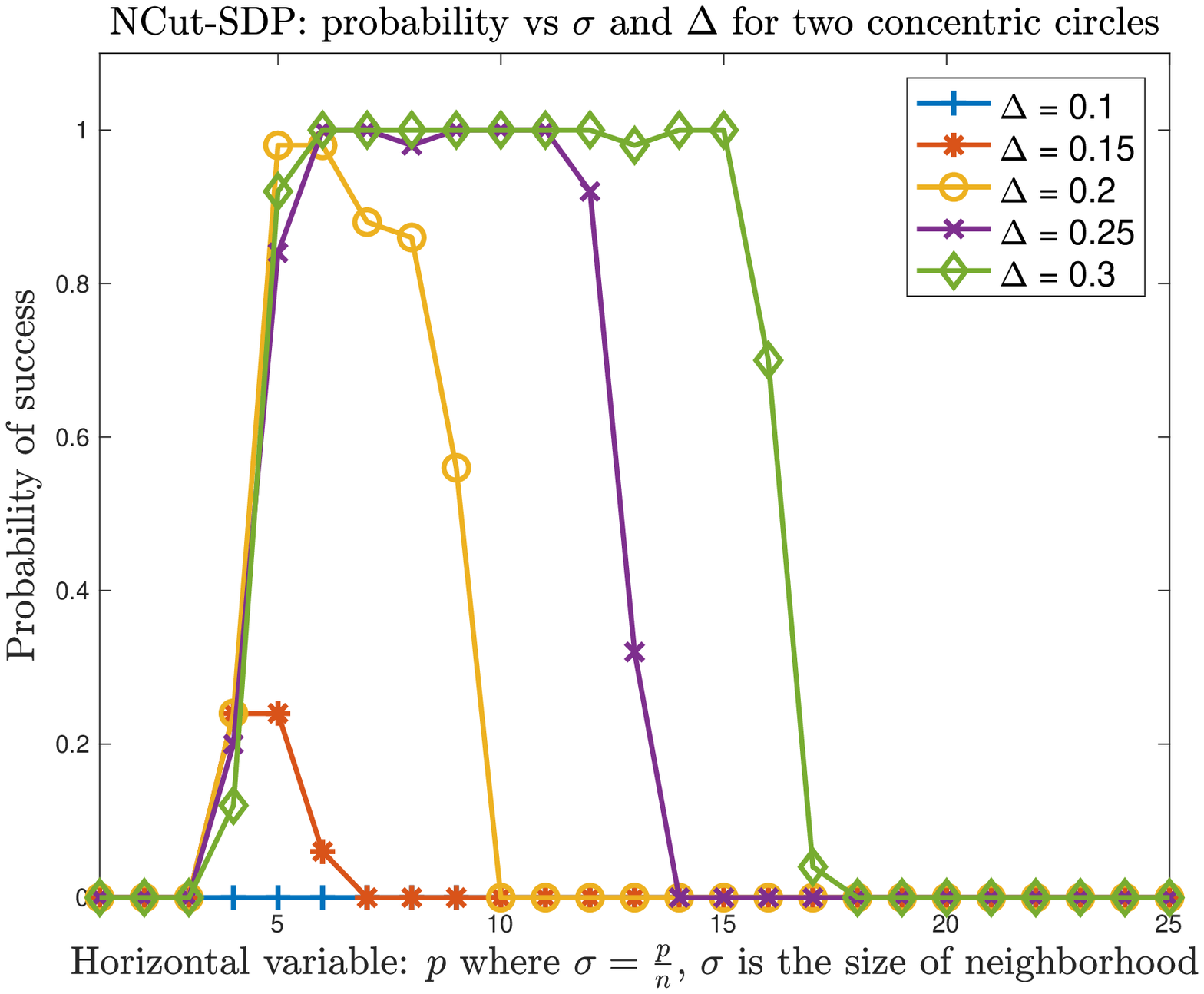}
\end{minipage}
\caption{Two concentric circles with radii $r_1=1$ and $r_2 = 1+\Delta$. The smaller circle has $n=250$ uniformly distributed points and the larger one has $m = \lfloor 250(1+\Delta)\rfloor$ where $\Delta = r_2-r_1$. Left: RatioCut-SDP; Right: NCut-SDP.}
\label{fig:toy1}
\end{figure}

\subsection{Two parallel lines}

For the two-lines case, we define  two clusters via $[-\frac{\Delta}{2} \,\,\, x_i]^{\top}$ and $[\frac{\Delta}{2} \,\,\, y_i]^{\top}, 1\leq i\leq 250$ where all $\{x_i\}_{i=1}^n$ and $\{y_i\}_{i=1}^n$ are i.i.d. uniform random variables on $[0,1]$. Thus the two clusters are exactly $\Delta$ apart. We also run 50 experiments for any  pair of $(\Delta,\sigma)$ where $\sigma=\frac{p}{2n}$ with $1\leq p\leq 20$ and $n=250$. Then we compute how many of those random instances satisfy conditions~\eqref{mainbound} and~\eqref{mainbound2}, similar to what we have done previously. Empirically, the SDP relaxation of spectral clustering achieves exact recovery with high probability if $\Delta \geq 0.05$ and $2\leq p\leq 150\Delta-4$, which outperforms the ordinary k-means by a huge margin. Recall that in k-means, when $\Delta < \frac{1}{2}$ and $n$ is large, the global minimum of k-means is unable to detect the underlying clusters correctly. Here, the SDP relaxation works provably even if $\Delta\geq 0.05$ which is very close to $\frac{\log n}{n}$ where $n = 250.$

\begin{figure}[h!]
\centering
\begin{minipage}{0.48\textwidth}
\includegraphics[width=75mm]{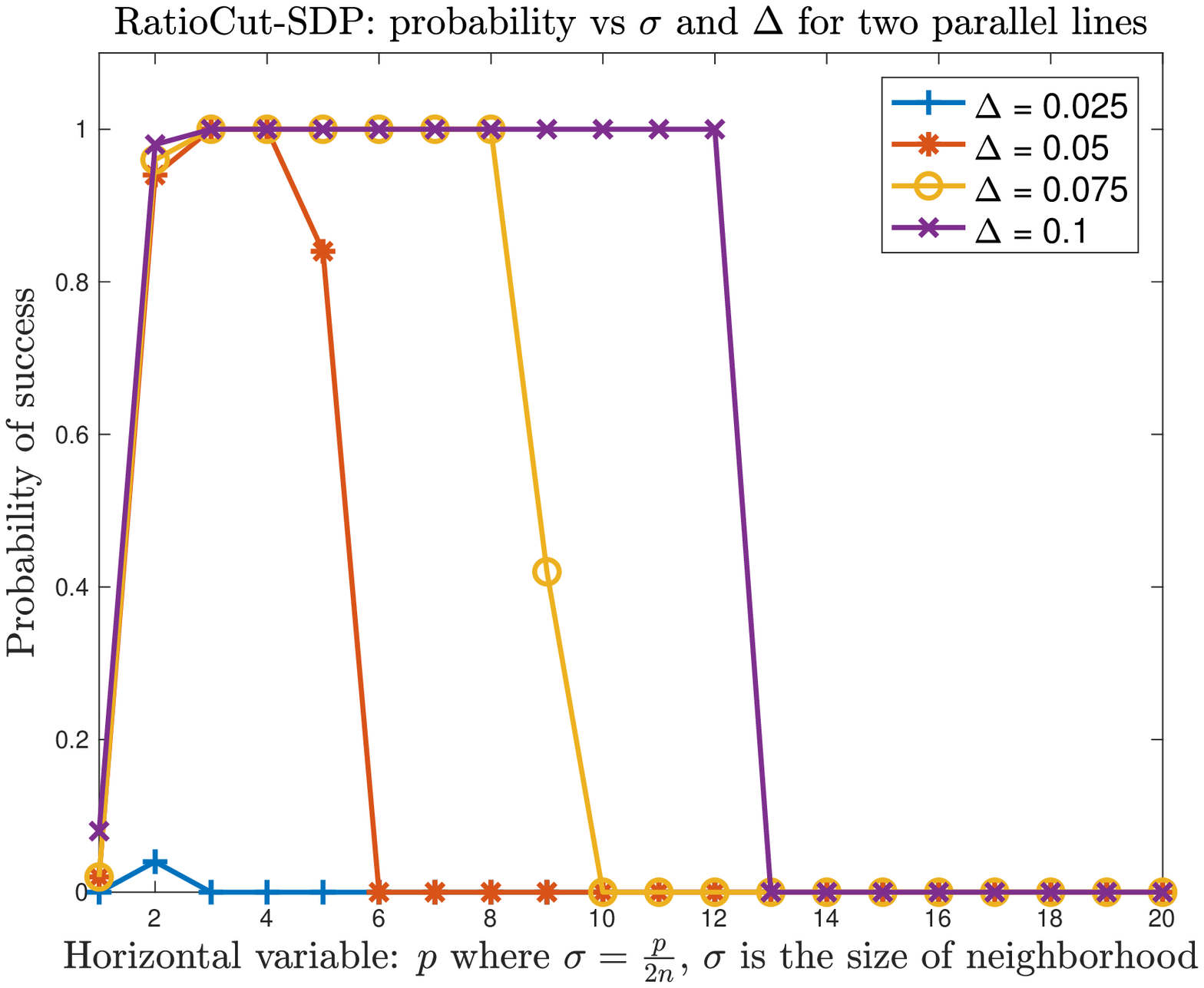}
\end{minipage}
\hfill
\begin{minipage}{0.48\textwidth}
\includegraphics[width=75mm]{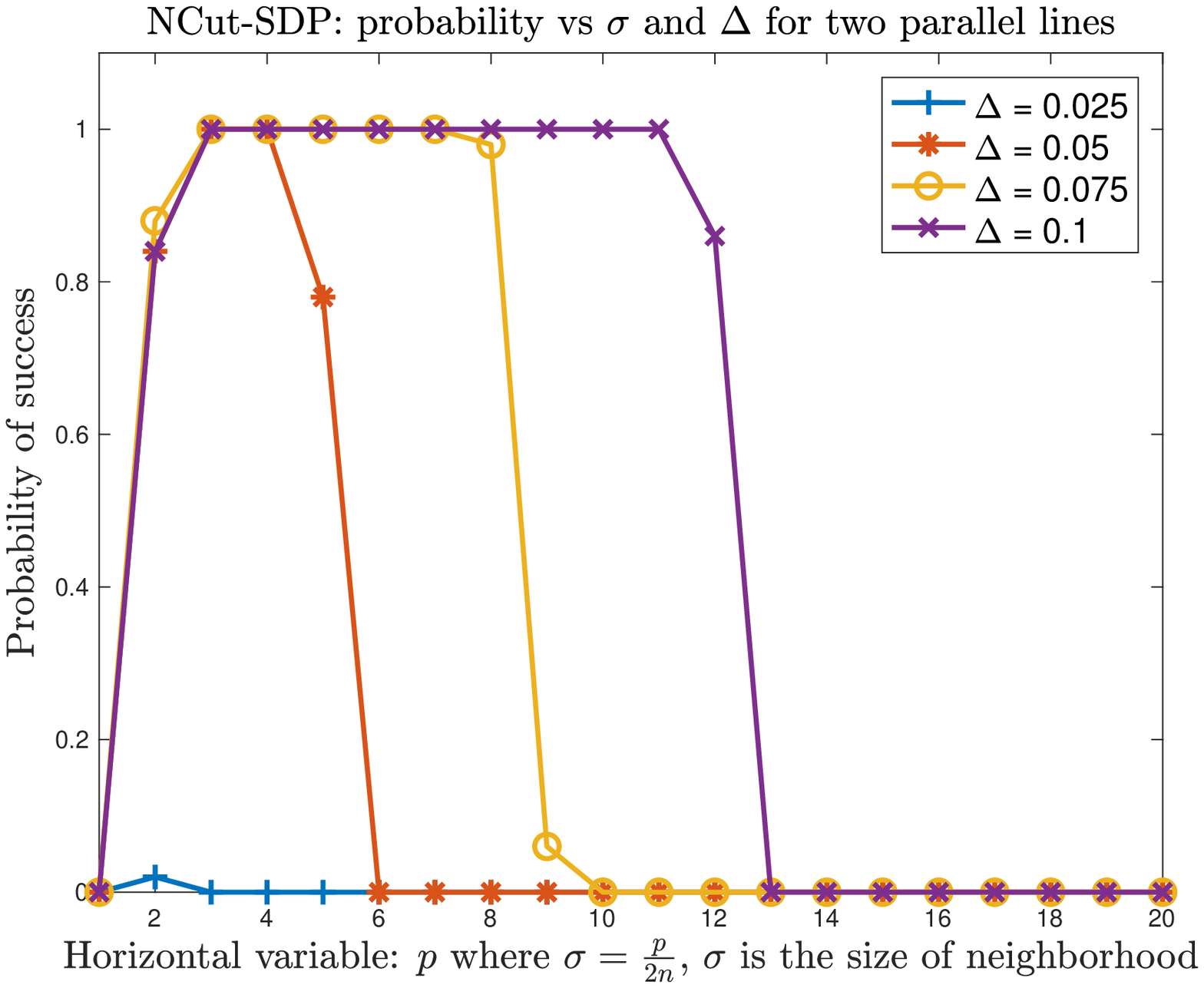}
\end{minipage}
\caption{Two parallel line segments of unit length with separation $\Delta$ and 250 points are sampled uniformly on each line. Left: RatioCut-SDP; Right: NCut-SDP.}
\label{fig:toy2}
\end{figure}

\subsection{Stochastic ball model and comparison with k-means}
Now we apply Theorems~\ref{thm:main} and~\ref{thm:main2} to the stochastic ball model with two clusters located on a 2D plane and compare the results with those via k-means SDP. 
The stochastic ball model is believed in a way optimal for k-means: the clusters are within convex boundaries and perfectly isotropic. However, we will find spectral clustering SDP performs much better. 
Consider the stochastic ball model with two clusters which satisfy
\[
x_{1,i} = \begin{bmatrix}
-\frac{\Delta}{2} - 1 \\
0
\end{bmatrix} + r_{1,i}, \qquad 
x_{2,i} = \begin{bmatrix}
\frac{\Delta}{2} + 1\\
0
\end{bmatrix} + r_{2,i}
\]
where $\{r_{1,i}\}_{i=1}^n$ and $\{r_{2,i}\}_{i=1}^n$ are i.i.d. uniform random vectors on  the 2D unit disk. From the definition, we know that the support of probability density function of each cluster is included in a unit disk centered at $[-\frac{\Delta}{2} - 1\,\,\, 0]^{\top}$ and $[\frac{\Delta}{2} +1\,\,\, 0]^{\top}$ respectively. If $\Delta > 0$, then the supports of those two distributions are separated by at least $\Delta$. 
\begin{figure}[h!]
\centering
\includegraphics[width=80mm]{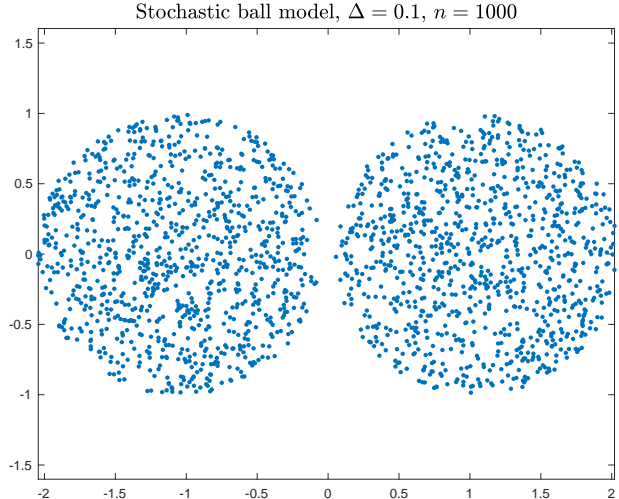}
\caption{Stochastic ball model with two clusters and the separation between the centers is $2 + \Delta$. Each cluster has 1000 points.}
\end{figure}

The results are illustrated in Figure~\ref{fig:toy3} and~\ref{fig:toy4}, and we summarize the empirical sufficient condition for exact recovery in terms of $n$, $\sigma$, and $\Delta$:
\begin{enumerate}
\item if $n = 250,$ we require $\Delta \geq 0.2$, $\sigma = \frac{p}{5\sqrt{n}}$ and $2\leq p \leq 40\Delta-4$;
\item if $n = 1000$, we require $\Delta \geq 0.1,$ $\sigma = \frac{p}{5\sqrt{n}}$, and $2\leq p\leq 60\Delta - 2$.
\end{enumerate}
In other words, when the number of points increases, the minimal separation $\Delta$ for exact recovery will also decrease because a smaller $\sigma$ can be picked to ensure strong within-cluster connectivity while the inter-cluster connectivity diminishes  simultaneously  for a fixed $\Delta$.

To compare the performance of spectral clustering SDP with that of k-means SDP, we use a necessary condition in~\cite{LLLSW17}. The k-means SDP is exactly in the form of~\eqref{prog:primal} but the graph Laplacian $L$ is replaced by the squared distance matrix of data. The necessary condition in~\cite{LLLSW17} states that the exact recovery via k-means SDP is impossible if
\[
\Delta \leq \sqrt{\frac{3}{2}} - 1 \approx 0.2247.
\]
On the other hand, we see much better performance via~\eqref{prog:primal} and~\eqref{prog:primal-ncut} from Figure~\ref{fig:toy3} and~\ref{fig:toy4} respectively. Even if the separation $\Delta$ is below $0.2$, one can still achieve exact recovery with high probability with a proper choice of $\sigma  = {\cal O}(\frac{1}{\sqrt{n}}).$ 

\begin{figure}[h!]
\centering
\begin{minipage}{0.48\textwidth}
\includegraphics[width=75mm]{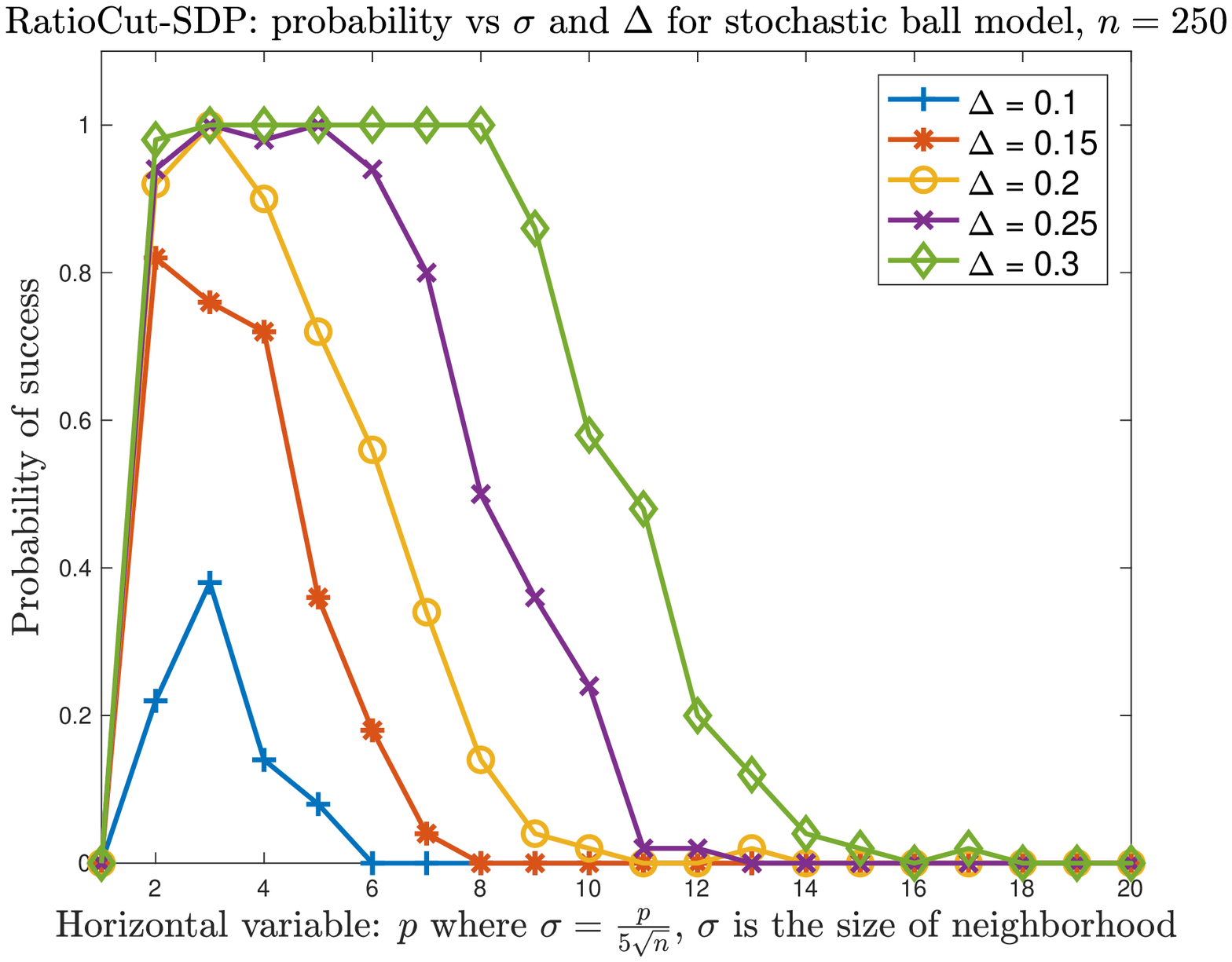}
\end{minipage}
\hfill
\begin{minipage}{0.48\textwidth}
\includegraphics[width=75mm]{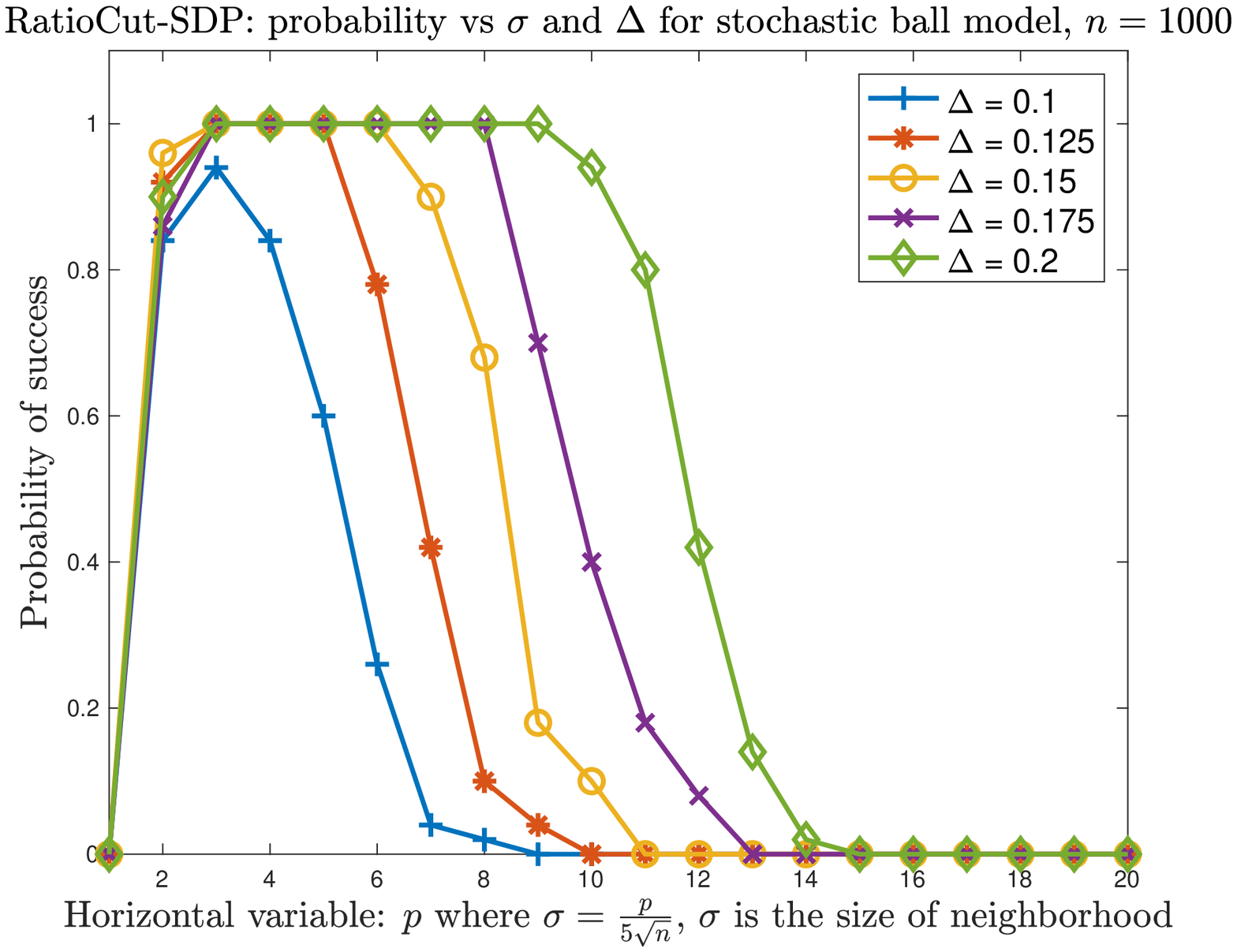}
\end{minipage}
\caption{Performance of RatioCut-SDP  for 2D stochastic ball model. Left: each ball has $n = 250$ points; Right: each ball contains $n = 1000$ points.}
\label{fig:toy3}
\end{figure}

\begin{figure}[h!]
\centering
\begin{minipage}{0.48\textwidth}
\includegraphics[width=75mm]{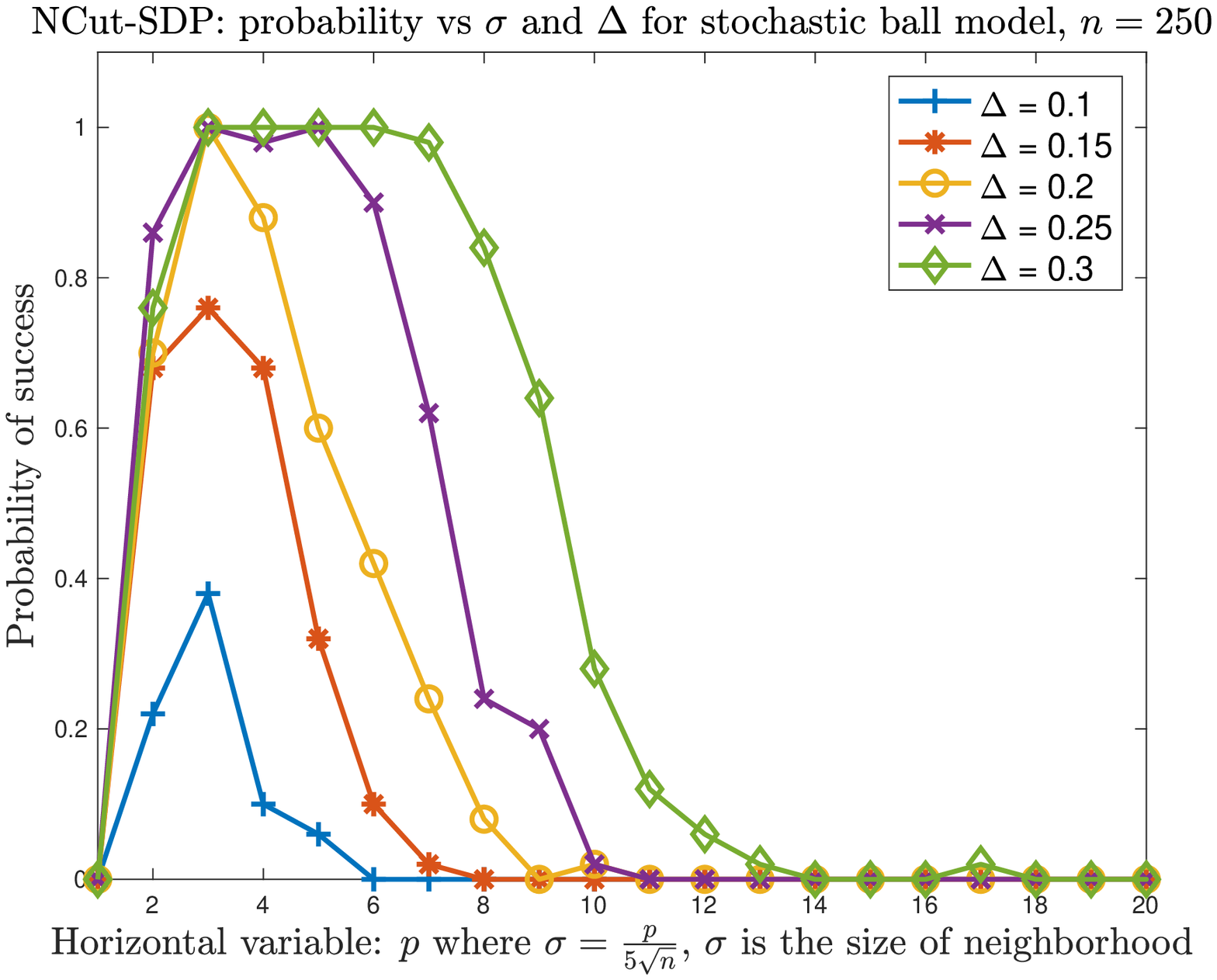}
\end{minipage}
\hfill
\begin{minipage}{0.48\textwidth}
\includegraphics[width=75mm]{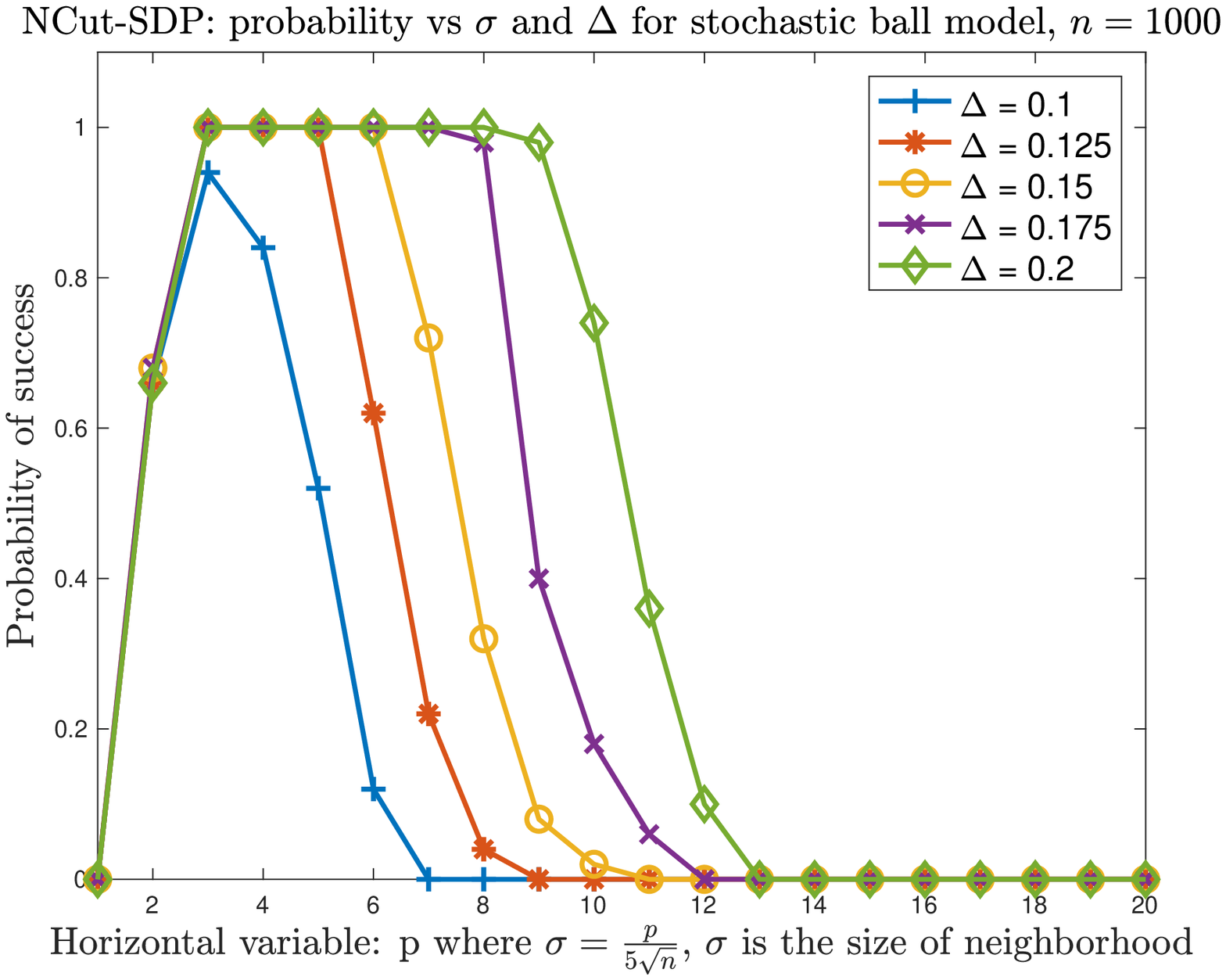}
\end{minipage}
\caption{Performance of NCut-SDP for 2D stochastic ball model. Left: each ball has $n = 250$ points; Right: each ball contains $n = 1000$ points.}
\label{fig:toy4}
\end{figure}

\section{Proofs}\label{s:proofs}
\subsection{Proofs of Theorem~\ref{thm:main} and~\ref{thm:main2}}
To certify $X$ in~\eqref{def:X} as the global minimizer of~\eqref{prog:primal} and~\eqref{prog:primal-ncut}, we resort to the powerful tool of Lagrangian duality theory~\cite{BoydV04,BenN01}. While some of the calculations follow from our previous paper~\cite{LLLSW17}, we include them to make the presentation more self-contained. The proof starts with finding a sufficient condition that guarantees $X$ to be the global minimizer and then we show the assumptions in Theorem~\ref{thm:main} and~\ref{thm:main2} satisfy the proposed sufficient conditions. 

In our proof, we will use the famous Gershgorin circle theorem repeatedly, the proof of which can be found in many sources such as~\cite[Chapter 7]{GolubV96}.
\begin{theorem}[\bf {Gershgorin circle theorem}]\label{thm:gersh}
Given a matrix $Z = (z_{ij})_{1\leq i,j\leq n}\in\RR^{n\times n}$ and all of its eigenvalues $\{\lambda_l(Z)\}_{l=1}^n$ are contained in the union of the circles centered at $\{z_{ii}\}_{i=1}^n$,
\[
\{ \lambda_l(Z)\}_{l=1}^n \subseteq \bigcup_{i=1}^n \Big\{ x\in\CC \Big| |x - z_{ii}| \leq \sum_{j\neq i} |z_{ij}| \Big\}.
\]
In particular, if $Z$ is also nonnegative, all the eigenvalues satisfy
\[
\max_{1\leq l\leq n}|\lambda_l(Z)| \leq \|Z\|_{\infty}.
\]
\end{theorem}

\subsubsection{Notation and preliminaries}
To begin with, we introduce a few notations.
Due to the similarity between the programs~\eqref{prog:primal} and~\eqref{prog:primal-ncut}, we combine them into the following general form:
\begin{equation}\label{prog:prim}
\min_{Z\in{\cal S}_N} \lag A, Z\rag, \quad \text{s.t.} \quad Z\succeq 0, \quad Z\geq 0, \quad \Tr(Z) = k, \quad Z\varphi = \varphi.
\end{equation}
Here $\varphi\in\RR^N$, $\varphi_a\in\RR^{n_a}$ (which is the $a$-th block of $\varphi$), and $A$ are defined as follows:
\begin{enumerate}
\item For RatioCut-SDP in~\eqref{prog:primal}, 
\begin{equation}\label{def:r-varphi}
\varphi := 1_N, \quad \varphi_a := 1_{n_a}, \quad A := L.
\end{equation}
\item For NCut-SDP in~\eqref{prog:primal-ncut}, 
\begin{equation}\label{def:n-varphi}
\varphi := D^{\frac{1}{2}}1_N, \quad \varphi_a := (D^{(a,a)})^{\frac{1}{2}}1_{n_a}, \quad A := L_{\sym}.
\end{equation}
\end{enumerate}

With the definition of $\varphi_a$ in~\eqref{def:r-varphi} and~\eqref{def:n-varphi}, we can put the ground truth $X$ into a unifying form
\[
X^{(a,a)}  = \frac{1}{\|\varphi_a\|^2}\varphi_a\varphi_a^{\top}, \qquad X^{(a,b)} = 0, \quad a\neq b
\]
which is a block-diagonal matrix.

Our theoretic analysis also rely on a few commonly used convex cones.
Let ${\cal K}$ and ${\cal K}^*$ be a pair of cone and dual cone\footnote{The dual cone ${\cal K}^*$ of ${\cal K}$ is defined as $\{W : \lag W, Z\rag\geq 0, \forall Z\in {\cal K}\}$; {in particular, $({\cal K}^*)^* = {\cal K}$ holds.}}:
\begin{equation}\label{def:Kcone}
{\cal K} = {\cal S}_N^+ \cap \RR^{N\times N}_{+}, \quad {\cal K}^* =   {\cal S}_N^+ + \RR^{N\times N}_+ = \{Q + B: Q\succeq 0, B\geq 0\}
\end{equation} 
where ${\cal K}$ is the intersection of two self-dual cones, i.e., the positive semi-definite cone ${\cal S}_N^+$ and the nonnegative cone $\RR^{N\times N}_{+}$. By definition, ${\cal K}$ is a pointed\footnote{The cone ${\cal K}$ is pointed if for $Z\in {\cal K}$ and $-Z\in {\cal K}$, $Z$ must be $0$, see Chapter 2 in~\cite{BenN01}.} and closed convex cone with a nonempty interior. 
 For the last two constraints in~\eqref{prog:prim}, we define a useful linear operator $\A$ which maps ${\cal S}_N$ to $\RR^{N+1}$ as follows
\begin{equation}\label{def:A}
\A: {\cal S}_N \rightarrow \RR^{N+1}: \quad \A(Z) = \begin{bmatrix}
\lag I_N, Z\rag \\
\frac{1}{2}(Z + Z^{\top})\varphi
\end{bmatrix}.
\end{equation}
Obviously, {it holds that} $\Tr(Z) = \lag I_N, Z\rag$ and $Z\varphi = \frac{1}{2}(Z + Z^{\top})\varphi$, and thus the last two constraints in~\eqref{prog:prim} can be written as
\[
\A(X) = \begin{bmatrix}k \\ \varphi \end{bmatrix} = : b.
\]
Its dual operator $\A^*$ under the canonical inner product over $\RR^{N\times N}$ is given by
\[
\A^*(\lambda) := \frac{1}{2}(\alpha \varphi^{\top} + \varphi\alpha^{\top}) + zI_N
\]
where $\lambda : = \begin{bmatrix} z \\ \alpha\end{bmatrix}\in\RR^{N+1}$ is the dual variable with $z\in\RR$ and $\alpha\in\RR^N$. Therefore, an equivalent form of~\eqref{prog:prim} is 
\[
\min_{Z\in {\cal K}} \lag A, Z\rag, \quad \text{s.t.} \quad \A(Z) = b.
\]
The Lagrangian function can be expressed in the form of
\[
\mathcal{L}(Z, \lambda) := \lag A, Z\rag + \lag \lambda, \A(Z) - b\rag =  \lag \A^*(\lambda) + A, Z\rag - \lag \lambda, b\rag.
\]
By taking the infimum over ${\cal K} : = \{Z: Z\succeq 0, Z\geq 0\}$, we have
\[
\inf_{Z\in{\cal K}}{\cal L}(Z, \lambda) = -\lag \lambda, b\rag
\]
if $\A^*(\lambda) + A \in{\cal K}^*$ and then obtain the dual program of~\eqref{prog:prim}:
\[
\max - \lag\lambda, b\rag, \quad \text{s.t.} \quad \A^*(\lambda) + A\in {\cal K}^*.
\]
Here $\A^*(\lambda) + A\in {\cal K}^*$ means that $\A^*(\lambda) + A$ can be written as the sum of a positive semidefinite matrix and a nonnegative matrix, i.e.,
\begin{equation}\label{def:QB}
\A^*(\lambda) + A = Q + B,
\end{equation}
where $Q\succeq 0$ and $B\geq 0$. 

Finally, we define two linear subspaces on ${\cal S}_N$ which will be useful later:
\begin{equation}\label{def:T}
\begin{array}{rl}%
T & : =  \{ XZ + ZX - XZX: Z\in {\cal S}_N\}, \\
\TB  & := \{ (I_N - X)Z(I_N - X): Z\in {\cal S}_N\}.
\end{array}
\end{equation}
We then denote $Z_{T}$ and $Z_{\TB}$ as the orthogonal projection of $Z$ onto $T$ and $\TB$ respectively. More specifically, the corresponding $(a,b)$-block of $Z_T$ and $Z_{\TB}$ can be written into
\begin{align*}
Z^{(a,b)}_{\TB} & : = \left(I_{n_a} - \frac{\varphi_a\varphi_a^{\top}}{\|\varphi_a\|^2}\right)Z^{(a,b)}\left(I_{n_b} - \frac{\varphi_b\varphi_b^{\top}}{\|\varphi_b\|^2}\right), \\
Z^{(a,b)}_{\TB} & : =  \frac{\varphi_a\varphi_a^{\top}}{\|\varphi_a\|^2}Z^{(a,b)}+Z^{(a,b)}\frac{\varphi_b\varphi_b^{\top}}{\|\varphi_b\|^2} -  \frac{\varphi_a^{\top}Z^{(a,b)}\varphi_b}{\|\varphi_a\|^2\|\varphi_b\|^2} \varphi_a\varphi_b^{\top}.
\end{align*}

\subsubsection{Optimality condition and dual certificate}
From the theory of convex optimization~\cite{BenN01}, we know that $X$ is a global minimizer (not necessarily unique) of~\eqref{prog:prim} if complementary slackness holds
\[
\lag \A^*(\lambda) + A, X\rag = \lag Q+ B, X\rag= 0
\]
where $\lambda$ is the dual variable. From complementary slackness, we are able to find a few useful relations regarding $B,Q$, and $X.$ 

From $Q\succeq 0$, $B\geq 0$, and $X\in{\cal K}$, we have 
\[
\lag Q,X\rag = \lag B, X\rag = 0
\]
because both $\lag Q,X\rag$ and $\lag B, X\rag$ are nonnegative, and their sum equals 0.
Moreover, it holds that
\begin{equation}\label{cond:cs}
B^{(a,a)} = 0, \quad QX = 0
\end{equation}
where $B^{(a,a)} = 0$ follows from $X^{(a,a)} = \frac{1}{\|\varphi_a\|^2}\varphi_a\varphi_a^{\top} > 0$ and
\[
0 = \lag B, X\rag = \sum_{a=1}^k \frac{1}{\|\varphi_a\|^2}\lag B^{(a,a)}, \varphi_a \varphi_a^{\top}\rag.
\]
On the other hand, $ QX=0$ follows from $\lag Q,X\rag=0$, $Q\succeq 0$, and $X\succeq 0.$ By definition of $\TB$ in~\eqref{def:T}, we can see that  $QX = 0$ implies $Q\in \TB$.

With the discussion above, we are ready to present a sufficient condition to certify $X$ as the unique global minimizer of~\eqref{prog:prim}.
\begin{proposition}[{\bf Sufficient condition}]\label{prop:suff}
Suppose $X$ is a feasible solution of~\eqref{prog:prim} and there exists $Q$ such that $Q\succeq 0$ and $X$ satisfies $QX = 0$ where $Q$ is defined in~\eqref{def:QB}. Then $X$ is the unique minimizer of~\eqref{prog:prim} if $B^{(a,b)} > 0$ for all $a\neq b$ and $B^{(a,a)} = 0.$
\end{proposition}
Note that the choices of $Q$ and $B$ are not arbitrary; the sum of $Q$ and $B$ must satisfy $Q + B= \A^*(\lambda) +A$ for certain $\lambda.$ 
\begin{proof}[{\bf Proof of Proposition~\ref{prop:suff}}]
Let $\widetilde{X}$ be another feasible solution, i.e., $\widetilde{X}\in{\cal K}$ and $\A(\widetilde{X}) = b$ but $\widetilde{X}\neq X$. The goal is to show that $\lag A, \widetilde{X}\rag > \lag A, X\rag$, i.e., the objective function value evaluated at $\widetilde{X}$ is strictly larger than that evaluated at $X$.

Assume $\A^*(\lambda) + A= Q + B$ for certain $\lambda$. Since $\widetilde{X}$ and $Q$ are positive semidefinite, we have $\lag \widetilde{X}, Q\rag\geq 0$. Hence, it holds that
\begin{align*}
\lag Q, \widetilde{X}\rag & = \lag Q, \widetilde{X} - X\rag  = \lag \A^*(\lambda) + A - B, \widetilde{X} - X\rag \\ 
& = \lag A - B, \widetilde{X} - X\rag \geq 0
\end{align*}
which follows from $Q\succeq 0$, $QX = 0$, and both $\widetilde{X}$ and $X$ satisfy the linear constraints. Therefore, combined with $\lag B, X\rag=0$, we get
\[
\lag A, \widetilde{X} \rag - \lag A, X\rag \geq \lag B, \widetilde{X} - X\rag = \lag B, \widetilde{X}\rag \geq 0.
\]
Hence, it suffices to show that $\lag B, \widetilde{X}\rag > 0$ under $B^{(a,b)} > 0$. We achieve this by proving $\lag B, \widetilde{X}\rag = 0$ if and only if $\widetilde{X} = X$.

Suppose $\lag B, \widetilde{X}\rag = 0$ and $B^{(a,b)} > 0$, then $\widetilde{X}^{(a,b)} = 0$ follows from 
\[
\sum_{a=1}^k\sum_{b=1}^k \lag B^{(a,b)}, \widetilde{X}^{(a,b)}\rag =0 \Longleftrightarrow \lag B^{(a,b)}, \widetilde{X}^{(a,b)}\rag = 0.
\]

In other words, the support of $\widetilde{X}$ is contained in that of $X$, i.e., $\widetilde{X}$ is also a block-diagonal matrix. 
Therefore, combined with $\widetilde{X}\varphi = \varphi$, we have $\widetilde{X}^{(a,a)}\varphi_a = \varphi_a$ which means $\widetilde{X}$ has 1 as an eigenvalue with multiplicity $k$. On the other hand, $\Tr(\widetilde{X}) = k$ along with $\widetilde{X}\succeq 0$ implies that each block of $\widetilde{X}$ is rank-1 and must satisfy $\widetilde{X} = X.$
\end{proof}

According to Proposition~\ref{prop:suff}, it suffices to construct $B$ and $Q$ such that $QX = 0$, $Q\succeq 0$, $B^{(a,a)} = 0$, and $B^{(a,b)} > 0$ for all $a\neq b.$ Note that
\begin{equation}\label{def:Q}
Q = \frac{1}{2}(\alpha \varphi^{\top} + \varphi\alpha^{\top}) + zI_N + A -B
\end{equation}
which contains  three unknowns $\alpha$, $z$, and $B.$ In fact, we are able to determine $\alpha$ in terms of $z$ and hence by the following lemma, we express $Q$ explicitly in terms of $z$ and $B$ which can be found in~\eqref{def:Qab} and~\eqref{def:Qaa}. This is made possible by $QX = 0$ which is equivalent to 
\begin{equation}\label{eq:Qab}
Q^{(a,b)}\varphi_b = 0, \qquad \forall a,b.
\end{equation}

\begin{lemma}
Given $QX = 0$ and $Q\succeq 0$, the $a$-th block of $\alpha\in\RR^N$ is determined by
\begin{equation}\label{def:alphaa}
\alpha_a = -\frac{2}{\|\varphi_a\|^2}A^{(a,a)}\varphi_{a} - \frac{1}{\|\varphi_a\|^2}\left(z -  \frac{\varphi_a^{\top}A^{(a,a)}\varphi_{a}}{\|\varphi_a\|^2}\right)\varphi_a.
\end{equation}
\end{lemma}

\begin{proof}
By definition of $Q^{(a,a)}$ in~\eqref{def:Q} and $B^{(a,a)} = 0$, we have
\[
Q^{(a,a)}\varphi_a = \left(\frac{1}{2}(\alpha_a \varphi_a^{\top} + \varphi_a \alpha_a^{\top}) + zI_{n_a}  + A^{(a,a)}\right)\varphi_a = 0.
\]
Solving for $\alpha_a$ gives
\begin{equation}\label{eq:alphaa}
\alpha_a = -\frac{2}{\|\varphi_a\|^2}A^{(a,a)}\varphi_{a} - \frac{1}{\|\varphi_a\|^2}(\alpha_a^{\top}\varphi_a + 2z )\varphi_{a}.
\end{equation}
Multiplying both sides with $\varphi^{\top}_{a}$ from the left gives an expression of $\alpha_a^{\top}\varphi_a$, i.e.,
\[
\alpha_a^{\top}\varphi_{a} = -\frac{2}{\|\varphi_a\|^2}\varphi_a^{\top}A^{(a,a)}\varphi_{a} - (\alpha_a^{\top}\varphi_{a} + 2z ) \Longrightarrow \alpha_a^{\top}\varphi_a  = -\frac{1}{\|\varphi_a\|^2}\varphi_a^{\top}A^{(a,a)}\varphi_{a} - z 
\]
Plugging $\alpha_a^{\top}\varphi_a$ back to~\eqref{eq:alphaa} finishes the proof.
\end{proof}

Once we have $\alpha_a$ in~\eqref{def:alphaa} and substitute it into~\eqref{def:Q}, we have an explicit expression for $Q^{(a,b)}$:
\begin{align}
\begin{split}
Q^{(a,b)} &  = -\left(\frac{A^{(a,a)}\varphi_a\varphi_b^{\top}}{\|\varphi_a\|^2} + \frac{\varphi_a\varphi_b^{\top} A^{(b,b)}}{\|\varphi_b\|^2}\right)  
+ \frac{1}{2}\left( \frac{\varphi_a^{\top} A^{(a,a)}\varphi_a } {\|\varphi_a\|^4} + \frac{ \varphi_b^{\top} A^{(b,b)}\varphi_b} {\|\varphi_b\|^4} \right) \varphi_a\varphi_b^{\top}  \\
& \qquad  - \frac{z}{2}\left( \frac{1}{\|\varphi_a\|^2} + \frac{1}{\|\varphi_b\|^2}\right)\varphi_a\varphi_b^{\top} + (A^{(a,b)} - B^{(a,b)}). \label{def:Qab}
\end{split}
\end{align}
For diagonal blocks $Q^{(a,a)}$,
\begin{align}
Q^{(a,a)} & = \frac{1}{2}(\alpha_a \varphi_a^{\top} + \varphi_a\alpha_a^{\top}) + zI_{n_a} + A^{(a,a)} \nonumber \\
& = -\frac{A^{(a,a)}\varphi_a\varphi_a^{\top} + \varphi_a\varphi_a^{\top} A^{(a,a)}}{\|\varphi_a\|^2} + z\left(I_{n_a} -  \frac{\varphi_a\varphi_a^{\top}}{\|\varphi_a\|^2}\right) + \frac{\varphi_a^{\top} A^{(a,a)}\varphi_a} {\|\varphi_a\|^4}\varphi_a\varphi_a^{\top} + A^{(a,a)} \nonumber \\
& =  \left(I_{n_a} - \frac{ \varphi_a\varphi_a^{\top}}{\|\varphi_a\|^2}\right) (A^{(a,a)} + zI_{n_a})\left(I_{n_a} - \frac{\varphi_a\varphi_a^{\top}}{\|\varphi_a\|}\right). \label{def:Qaa}
\end{align}

Now we summarize the discussion and clarify our goal: we want to pick up $z$ and $B$ such that
\begin{enumerate}
\item $Q^{(a,b)}\varphi_b = 0$ which is equivalent to $QX = 0$. Moreover, with direct computation, one can show that $B^{(a,b)}$ must satisfy
\begin{equation}\label{eq:Bphi}
B^{(a,b)}\varphi_b = \|\varphi_b\|^2 u_{a,b}\in\RR^{n_a}
\end{equation}
where $u_{a,b}$ only depends on $z$ and satisfies
\begin{align}
\begin{split}
u_{a,b} & := \frac{A^{(a,b)}\varphi_{b}}{\|\varphi_b\|^2} - \frac{A^{(a,a)}\varphi_{a}}{\|\varphi_a\|^2} + \frac{1}{2}\left( \frac{\varphi_a^{\top} A^{(a,a)}\varphi_a } {\|\varphi_a\|^4} - \frac{\varphi_b^{\top} A^{(b,b)}\varphi_b} {\|\varphi_b\|^4} \right)\varphi_a \\
& \qquad  - \frac{z}{2}\left( \frac{1}{\|\varphi_a\|^2} + \frac{1}{\|\varphi_b\|^2}\right)\varphi_{a}.  \label{def:u}
\end{split}
\end{align}

\item $Q$ is positive semidefinite, i.e., $Q \succeq 0.$ In fact, we have $Q\in\TB$ since $QX = 0.$
\item $B^{(a,b)} > 0$, $B^{(a,a)} = 0$, and $B$ is symmetric.
\end{enumerate}
Finally, we arrive at the following requirements for the constructions of dual certificate:
\begin{equation}\label{cond:Q}
B^{(a,b)}\varphi_b = \|\varphi_b\|^2 u_{a,b}, \quad  Q_{\TB}\succeq 0, \quad B^{(a,b)} > 0, \quad B^{(a,a)} = 0
\end{equation}
for all $a\neq b$ where $u_{a,b}$ depends on $z$ in~\eqref{def:u}.

Here we choose $B$ as 
\begin{equation}\label{eq:conB}
B^{(a,b)} := u_{a,b}\varphi_b^{\top} + \varphi_a u_{b,a}^{\top} - \frac{\varphi_a^{\top}u_{a,b}}{\|\varphi_a\|^2}\varphi_a\varphi_b^{\top}.
\end{equation} 

Note that $B$ whose $(a,b)$-block satisfies~\eqref{eq:conB} belongs to $T$ in~\eqref{def:T} since $B_{\TB} = 0$. This fact significantly simplifies our argument later. Moreover,
Lemma~\ref{lem:Bsym} implies the $B$ satisfies~\eqref{eq:Bphi} and $B$ is also symmetric. 

Now it is easy to see that there is~\emph{only one} variable $z$ to be determined since $B$ is a function of $z$, as implied by~\eqref{eq:conB} and~\eqref{def:u}.  Next, we will prove that a choice of $z$ exists such that $B \geq 0$ and $Q_{\TB}\succeq 0$ hold simultaneously under the assumptions in Theorem~\ref{thm:main} and~\ref{thm:main2}. Combining all those results together finishes our proof. 

\begin{lemma}\label{lem:Bsym}
The matrix $B$ is symmetric  and  satisfies~\eqref{eq:Bphi} under the construction of $B$ in~\eqref{eq:conB}.
\end{lemma}

\begin{proof}[{\bf Proof of Lemma~\ref{lem:Bsym}}]
By the construction of $B^{(a,b)}$ in~\eqref{eq:conB}, we have
\[
B^{(a,b)}\varphi_b = \|\varphi_b\|^2 u_{a,b} + \|\varphi_b\|^2 \left(  \frac{u_{b,a}^{\top}\varphi_b}{\|\varphi_b\|^2}  - \frac{\varphi_a^{\top}u_{a,b}}{\|\varphi_a\|^2} \right) \varphi_a.
\]
Then $B^{(a,b)}$ satisfies~\eqref{eq:Bphi} if the second term above vanishes and thus it suffices to show
\begin{equation}\label{eq:unorm}
\frac{u_{b,a}^{\top}\varphi_b}{\|\varphi_b\|^2} = \frac{u_{a,b}^{\top}\varphi_a}{\|\varphi_a\|^2}, \quad \forall a\neq b.
\end{equation}
With $u_{a,b}$ in~\eqref{def:u}, direct computations gives
\begin{align*}
\frac{u_{a,b}^{\top}\varphi_a}{\|\varphi_a\|^2} & = \frac{\varphi_a^{\top}A^{(a,b)}\varphi_b}{\|\varphi_a\|^2\|\varphi_b\|^2} - \frac{1}{2}\left( \frac{\varphi_a^{\top} A^{(a,a)}\varphi_a } {\|\varphi_a\|^4} + \frac{\varphi_b^{\top} A^{(b,b)}\varphi_b} {\|\varphi_b\|^4} \right) - \frac{z}{2}\left( \frac{1}{\|\varphi_a\|^2} + \frac{1}{\|\varphi_b\|^2}\right).
\end{align*}
Note that $\frac{u_{b,a}^{\top}\varphi_b}{\|\varphi_b\|^2}$ and $\frac{u_{a,b}^{\top}\varphi_a}{\|\varphi_a\|^2}$ have the same terms except their first terms $\varphi_a^{\top}A^{(a,b)}\varphi_b$ and $\varphi_b^{\top}A^{(b,a)}\varphi_a$.
By the definition of $\varphi_a$ and $A^{(a,b)}$ in~\eqref{def:r-varphi} and~\eqref{def:n-varphi}, we have
\[
\varphi_a^{\top}A^{(a,b)}\varphi_b = 
\begin{cases}
1_{n_a}^{\top}L^{(a,b)}1_{n_b}, & \text{ if }A = L, \\
1_{n_a}^{\top}(D^{(a,a)})^{\frac{1}{2}} L_{\sym}^{(a,b)}  (D^{(b,b)})^{\frac{1}{2}}1_{n_b} = 1_{n_a}^{\top}L^{(a,b)}1_{n_b}, &  \text{ if } A = L_{\sym},
\end{cases}
\]
where $L^{(a,b)}_{\sym} = (D^{(a,a)})^{-\frac{1}{2}}L^{(a,b)}(D^{(b,b)})^{-\frac{1}{2}}.$
Since $L^{(a,b)} = -W^{(a,b)}$ and $W^{(a,b)} = (W^{(b,a)})^{\top}$, it holds that
\[
\varphi_a^{\top}A^{(a,b)}\varphi_b = \varphi_b^{\top}A^{(b,a)}\varphi_a, \quad a\neq b.
\]
Then we have~\eqref{eq:unorm} and $B$ satisfies~\eqref{eq:Bphi}. 
The proof of the symmetry of $B$ is straightforward by using~\eqref{eq:unorm}.

\end{proof}

By substituting $u_{a,b}$ into~\eqref{eq:conB}, we have the expression for $B^{(a,b)}$ with $a\neq b$,
\begin{align}
\begin{split}\label{eq:conB-ex}
B^{(a,b)} & = \frac{\varphi_{a}\varphi_{a}^{\top}A^{(a,b)}}{\|\varphi_a\|^2}  +  \frac{A^{(a,b)}\varphi_{b}\varphi_b^{\top}}{\|\varphi_b\|^2} - \frac{A^{(a,a)}\varphi_{a}\varphi_b^{\top}}{\|\varphi_a\|^2} - \frac{\varphi_a \varphi_b^{\top}A^{(b,b)}}{\|\varphi_b\|^2}  - \frac{z}{2}\left( \frac{1}{\|\varphi_a\|^2} + \frac{1}{\|\varphi_b\|^2}\right)\varphi_{a}\varphi_b^{\top} \\
& \qquad -\left( \frac{\varphi_a^{\top}A^{(a,b)}\varphi_b}{\|\varphi_a\|^2 \|\varphi_b\|^2} - \frac{1}{2}\left( \frac{\varphi_a^{\top} A^{(a,a)}\varphi_a } {\|\varphi_a\|^4} + \frac{\varphi_b^{\top} A^{(b,b)}\varphi_b} {\|\varphi_b\|^4} \right) \right)\varphi_a\varphi_b^{\top}
\end{split}
\end{align}
by using $A^{(a,b)} = (A^{(b,a)})^{\top}.$

\subsubsection{Proof of Theorem~\ref{thm:main}}

\begin{proposition}[{\bf Proof of Theorem~\ref{thm:main}}]\label{prop:z}
Assume we have
\[
\|D_{\delta}\| < \frac{\min_{1\leq a\leq k} \lambda_2(L^{(a,a)}_{\iso})}{4} = \frac{\lambda_{k+1}(L_{\iso})}{4}
\]
and $B$ is chosen in~\eqref{eq:conB}. There exists a choice of $z$ between $-\lambda_{k+1}(L_{\iso})$ and $-4\|D_{\delta}\|$ such that for all $a\neq b$,
\begin{equation}
B^{(a,b)} > 0, \quad Q_{\TB}\succeq 0,
\end{equation}
hold simultaneously which certifies $X_{\rcut}$ as the unique global minimizer to~\eqref{prog:primal}.
\end{proposition}

\begin{proof}
In this RatioCut case, we set
\[
\varphi_a = 1_{n_a}, \quad A = L, 
\]
according to~\eqref{def:r-varphi}. In this case, $B^{(a,b)}$ is in the form of
\begin{align}
\begin{split}\label{eq:conB-ex1}
B^{(a,b)} & = \frac{1_{n_a}1_{n_a}^{\top}L^{(a,b)}}{n_a} +  \frac{L^{(a,b)}1_{n_b}1_{n_b}^{\top}}{n_b}   - \frac{L^{(a,a)}1_{n_a}1_{n_b}^{\top}}{n_a} - \frac{1_{n_a} 1_{n_b}^{\top}L^{(b,b)}}{n_b} - \frac{z}{2}\left( \frac{1}{n_a} + \frac{1}{n_b}\right)1_{n_a}1_{n_b}^{\top} \\
& \qquad -\left( \frac{1_{n_a}^{\top}L^{(a,b)}1_{n_b}}{n_an_b} - \frac{1}{2}\left( \frac{1_{n_a}^{\top} L^{(a,a)}1_{n_a} } {n_a^2} + \frac{1_{n_b}^{\top} L^{(b,b)}1_{n_b}} {n_b^2} \right) \right)1_{n_a}1_{n_b}^{\top}.
\end{split}
\end{align}
The proof proceeds in two steps: we first give a bound for $z$ such that $B^{(a,b)}  > 0$ hold and  $Q_{\TB}\succeq 0$ respectively; then we show such a parameter $z$ exists under the assumption of Theorem~\ref{thm:main}.

\paragraph{\bf Step 1: A sufficient condition for $B^{(a,b)} > 0$.}

We claim that if $z < -4\|D_{\delta}\|$, then $B^{(a,b)} > 0$.
By the definition of $L^{(a,b)}$, we have $L^{(a,b)} = -W^{(a,b)}$ and 
\begin{equation}\label{eq:Laa1}
L^{(a,a)}1_{n_a} = -\sum_{l\neq a} L^{(a,l)}1_{n_l} = \sum_{l\neq a}W^{(a,l)}1_{n_l}
\end{equation}
which follows from $L1_N = 0.$
Due to the nonnegativity of $W^{(a,b)}$, it holds that
\begin{equation}\label{eq:Wab1}
0\leq L^{(a,a)}1_{n_a} = \sum_{l\neq a}W^{(a,l)}1_{n_l} \leq \|D_{\delta}\| 1_{n_a}.
\end{equation}
Therefore, we get lower bounds for both $L^{(a,b)}1_{n_b}$ and $-L^{(a,a)}1_{n_a} $ as follows:
\begin{equation}\label{eq:B1-b}
L^{(a,b)}1_{n_b}=  -W^{(a,b)}1_{n_b}  \geq -\|D_{\delta}\| 1_{n_a}, \quad -L^{(a,a)}1_{n_a} = -\sum_{l\neq a}W^{(a,l)}1_{n_l} \geq -\|D_{\delta}\|1_{n_a}.
\end{equation}
Naturally, we also have
\begin{align}\label{eq:B1-c}
1_{n_a}^{\top}L^{(a,a)}1_{n_a} \geq 0, \qquad 1_{n_a}^{\top}L^{(a,b)}1_{n_b} = -1_{n_a}^{\top}W^{(a,b)}1_{n_b} \leq 0,
\end{align}
and thus the last term in~\eqref{eq:conB-ex1} is non-positive. Now applying~\eqref{eq:B1-b} and~\eqref{eq:B1-c} to~\eqref{eq:conB-ex1} results in
\[
B^{(a,b)} \geq -2\|D_{\delta}\|\left( \frac{1}{n_a} + \frac{1}{n_b}\right) 1_{n_a}1_{n_b}^{\top} - \frac{z}{2} \left( \frac{1}{n_a} + \frac{1}{n_b}\right) 1_{n_a}1_{n_b}^{\top} > 0
\]
if $z < -4\|D_{\delta}\|.$

\paragraph{\bf Step 2: A sufficient condition for $Q_{\TB}\succeq 0$.}

The equations $B^{(a,b)}1_{n_b} = n_b u_{a,b}$ guarantee $Q\in\TB$. Now we will show  $Q_{\TB}\succeq 0$ if
$z \geq- \min_{1\leq a\leq k}\lambda_2(L^{(a,a)}_{\iso}) $.
First we project $Q$ to $\TB$ and then each projected $Q^{(a,b)}$ obeys
\begin{align*}
Q^{(a,b)}_{T^{\bot}} & = \left(I_{n_a} - \frac{J_{n_a\times n_a}}{n_a}\right) (L^{(a,b)} - B^{(a,b)})\left(I_{n_b} - \frac{J_{n_b\times n_b}}{n_b}\right), \quad a\neq b, \\
Q_{\TB}^{(a,a)} & = \left(I_{n_a} - \frac{J_{n_a\times n_a}}{n_a}\right) (L^{(a,a)} + zI_{n_a})\left(I_{n_a} - \frac{J_{n_a\times n_a}}{n_a}\right), \quad a = b,
\end{align*}
which follow from the expression of $Q^{(a,b)}$ in~\eqref{def:Qab} and~\eqref{def:Qaa}.
Note that $B$ in~\eqref{eq:conB} is inside the subspace $T$ and hence $B_{\TB} = 0.$ As a result, we have
\[
Q_{\TB} = ( L + zI_N - B )_{\TB} = (L + zI_N)_{\TB}.
\]

Let $v\in\RR^N$ be a unit vector in $\TB$ and that means $Xv = 0\in\RR^k$, i.e., $1_{n_a}^{\top}v_a = 0$ for $1\leq a\leq k$ where $v_a$ is the $a$-th block of $v$. We aim to prove $v^{\top}Qv\geq 0$ for all such $v\in\RR^N$ and $\|v\| = 1$. There hold
\begin{align}\label{eq:Qlow1}
\begin{split}
v^{\top}Qv & = v^{\top}Q_{\TB}v = v^{\top}(L + zI_N )v = v^{\top}(L_{\iso} + L_{\delta})v  + z \\
& \geq \sum_{1\leq a\leq k } v_a^{\top}L^{(a,a)}_{\iso} v_a  + z  \geq \min_{1\leq a\leq k} \lambda_2(L_{\iso}^{(a,a)}) + z
\end{split}
\end{align}
where $L = L_{\iso} + L_{\delta}$ follows from~\eqref{def:delta} and $L_{\delta} = D_{\delta} - W_{\delta}\succeq 0$ since $L_{\delta}$ is a graph Laplacian. The last inequality in~\eqref{eq:Qlow1} is guaranteed by the variational characterization of the second smallest eigenvalue of symmetric matrices.
Therefore, $Q_{\TB}\succeq 0$ if
\[
z \geq -\min_{1\leq a\leq k} \lambda_2(L_{\iso}^{(a,a)}).
\]
Combining $z< -4\|D_{\delta}\|$ with $z\geq  -\min_{1\leq a\leq k} \lambda_2(L_{\iso}^{(a,a)})$, such a parameter $z$ exists for~\eqref{cond:Q} if 
\[
\|D_{\delta}\| < \frac{\min_{1\leq a\leq k} \lambda_2(L^{(a,a)}_{\iso})}{4}.
\]
\end{proof}

\subsubsection{Proof of Theorem~\ref{thm:main2}}
\begin{proposition}[{\bf Proof of Theorem~\ref{thm:main2}}]\label{prop:z2}
Assume we have
\[
\frac{\|P_{\delta}\|_{\infty}}{1 - \|P_{\delta}\|_{\infty}} < \frac{\min_{1\leq a\leq k} \lambda_2((D^{(a,a)}_{\iso})^{-1}L^{(a,a)}_{\iso})}{4 }
\]
and $B$ is chosen in~\eqref{eq:conB}.
Then there exists a choice of $z$  such that for all $a\neq b$,
\begin{equation}
B^{(a,b)} > 0, \quad Q_{\TB}\succeq 0,
\end{equation}
hold simultaneously which certifies $X_{\ncut}$ as the unique global minimizer to~\eqref{prog:primal-ncut}.
\end{proposition}

\begin{proof}
In the case of normalized cuts, we pick $\varphi_a$ and $A$ according to~\eqref{def:n-varphi}, 
\[
\varphi_a = (D^{(a,a)})^{\frac{1}{2}}1_{n_a}, \quad A = L_{\sym}
\]
where $\|\varphi_a\|^2 = 1_{n_a}^{\top}D^{(a,a)}1_{n_a} = \Vol(\Gamma_a).$ 


\paragraph{\bf Step 1: Proof of $B^{(a,b)} > 0$.}

 Note that the diagonal entries of degree matrix $D$ are strictly positive, and we only need to prove $(D^{(a,a)})^{-\frac{1}{2}}B^{(a,b)}(D^{(b,b)})^{-\frac{1}{2}} > 0$ which is true if $z< -4\|P_{\delta}\|_{\infty}.$ By
using $A^{(a,b)} = L^{(a,b)}_{\sym} = (D^{(a,a)})^{-\frac{1}{2}}L^{(a,b)}(D^{(b,b)})^{-\frac{1}{2}}$, $\varphi_a = (D^{(a,a)})^{\frac{1}{2}}1_{n_a}$, and~\eqref{eq:conB-ex}, we have
\begin{align}
\begin{split}\label{eq:conB-ex2}
(D^{(a,a)})^{-\frac{1}{2}}B^{(a,b)}(D^{(b,b)})^{-\frac{1}{2}} 
& = \frac{1_{n_a}1_{n_a}^{\top}L^{(a,b)} (D^{(b,b)})^{-1} }{\|\varphi_a\|^2} +  \frac{(D^{(a,a)})^{-1} L^{(a,b)}1_{n_b}1_{n_b}^{\top}}{\|\varphi_b\|^2} \\
& \qquad  - \frac{(D^{(a,a)})^{-1} L^{(a,a)}1_{n_a}1_{n_b}^{\top}}{\|\varphi_a\|^2} - \frac{1_{n_a} 1_{n_b}^{\top}L^{(b,b)}(D^{(b,b)})^{-1}  }{\|\varphi_b\|^2}   \\
&  \qquad - \frac{z}{2}\left( \frac{1}{\|\varphi_a\|^2} + \frac{1}{\|\varphi_b\|^2}\right)1_{n_a}1_{n_b}^{\top} \\
& \qquad -\left( \frac{1_{n_a}^{\top}L^{(a,b)}1_{n_b}}{\|\varphi_a\|^2 \|\varphi_b\|^2} - \frac{1}{2}\left( \frac{1_{n_a}^{\top} L^{(a,a)}1_{n_a} } {\|\varphi_a\|^4} + \frac{1_{n_b}^{\top} L^{(b,b)}1_{n_b}} {\|\varphi_b\|^4} \right) \right)1_{n_a}1_{n_b}^{\top}.
\end{split}
\end{align}
The last term in~\eqref{eq:conB-ex2} cannot be positive because of~\eqref{eq:B1-b} and~\eqref{eq:B1-c}.
Note that
\begin{align}
(D^{(a,a)})^{-1}L^{(a,b)} 1_{n_b} & = - (D^{(a,a)})^{-1}W^{(a,b)}1_{n_b} = -P^{(a,b)}1_{n_b} \geq -\|P_{\delta}\|_{\infty} 1_{n_a}, \label{eq:Lphi1} \\
(D^{(a,a)})^{-1} L^{(a,a)} 1_{n_a} & = (I_{n_a} - P^{(a,a)}) 1_{n_a} = \sum_{l\neq a}P^{(a,l)}1_{n_l} \leq \|P_{\delta}\|_{\infty}1_{n_a}, \label{eq:Lphi2}
\end{align}
where $L^{(a,b)} = -W^{(a,b)}$, $L^{(a,a)} = D^{(a,a)} - W^{(a,a)}$, and $P^{(a,b)} = (D^{(a,a)})^{-1}W^{(a,b)}.$

Plugging all these expressions~\eqref{eq:Lphi1} and~\eqref{eq:Lphi2} into~\eqref{eq:conB-ex2} results in
\begin{align*}
(D^{(a,a)})^{-\frac{1}{2}}B^{(a,b)}(D^{(b,b)})^{-\frac{1}{2}} & \geq -\left(2\|P_{\delta}\|_{\infty} + \frac{z}{2}\right)\left( \frac{1}{\|\varphi_a\|^2} + \frac{1}{\|\varphi_b\|^2} \right) 1_{n_a}1_{n_b}^{\top} > 0
\end{align*}
under $z < -4\|P_{\delta}\|_{\infty}$. Therefore, $B^{(a,b)} > 0$ holds for all $a\neq b$ if $z < -4\|P_{\delta}\|_{\infty}.$

\paragraph{\bf Step 2: Proof of $Q_{\TB} \succeq  0$.}

The equations $B^{(a,b)}\varphi_b = \|\varphi_b\|^2 u_{a,b} > 0$ guarantees $Q\in\TB$ and thus it suffices to show $Q_{\TB}\succeq 0.$ We will show the claim is true if
\[
z\geq -(1-\|P_{\delta}\|_{\infty})\min_{1\leq a\leq k}\lambda_2(L^{(a,a)}_{\rw,\iso}) .
\]
First we project $Q$ to $\TB$ and thus
\begin{align*}
Q^{(a,b)}_{T^{\bot}} & = \left(I_{n_a} - \frac{\varphi_a\varphi_a^{\top}}{\|\varphi_a\|^2}\right) (L_{\sym}^{(a,b)} - B^{(a,b)})\left(I_{n_b} - \frac{\varphi_b\varphi_b^{\top}}{\|\varphi_b\|^2}\right), \quad a\neq b, \\
Q_{\TB}^{(a,a)} & = \left(I_{n_a} - \frac{\varphi_a\varphi_a^{\top}}{\|\varphi_a\|^2}\right) (L_{\sym}^{(a,a)} + zI_{n_a})\left(I_{n_a} - \frac{\varphi_a\varphi_a^{\top}}{\|\varphi_a\|^2}\right), \quad a = b,
\end{align*}
where $Q^{(a,b)}$ can be found in~\eqref{def:Qab} and~\eqref{def:Qaa}. 

Let $v\in\RR^N$ be a unit vector in the range of $\TB$ and that means its $a$-th block $v_a$ of $v$ satisfies
\[
v_a^{\top}\varphi_a = v_a^{\top}(D^{(a,a)})^{\frac{1}{2}}1_{n_a} = 0, \quad \forall 1\leq a\leq k.
\]
Note that $v^{\top}B v =0$ which follows from the construction of $B$ and $B\in T.$ Therefore, $v^{\top}Qv$ has a lower bound as
\begin{align}\label{eq:vQv2}
\begin{split}
v^{\top}Qv & = v^{\top} (L_{\sym} + zI_N - B)v = v^{\top} L_{\sym}v  + z \\
& \geq (1- \|P_{\delta}\|_{\infty}) \cdot \min_{1\leq a\leq k}\lambda_2\left(  L_{\rw, \iso}^{(a,a)}\right)  + z
\end{split}
\end{align}
where $L_{\rw, \iso}^{(a,a)} = (D_{\iso}^{(a,a)})^{-1}L_{\iso}^{(a,a)}$ is in~\eqref{def:Piso}.
In the inequality~\eqref{eq:vQv2} above, we use the claim that 
\begin{equation}\label{eq:claim}
v^{\top}L_{\sym}v \geq  (1- \|P_{\delta}\|_{\infty}) \cdot \min_{1\leq a\leq k}\lambda_2\left( L_{\rw, \iso}^{(a,a)}\right)
\end{equation}
for any $v$ in the range of $\TB$ and $\|v\| = 1.$ Now we are going to prove this claim.

\paragraph{\bf Proof of the claim~\eqref{eq:claim}:}
First of all, $v^{\top}L_{\sym}v$ has its lower bound as
\begin{align*}
v^{\top}L_{\sym}v & = v^{\top} D^{-\frac{1}{2}}L D^{-\frac{1}{2}}v = v^{\top} D^{-\frac{1}{2}}(L_{\iso} + L_{\delta}) D^{-\frac{1}{2}}v \\
&  \geq v^{\top} D^{-\frac{1}{2}}L_{\iso}  D^{-\frac{1}{2}}v \geq \sum_{a=1}^k v_a^{\top}(D^{(a,a)})^{-\frac{1}{2}} L_{\iso}^{(a,a)}(D^{(a,a)})^{-\frac{1}{2}} v_a\\ 
&  \geq\sum_{a=1}^k \lambda_2( (D^{(a,a)})^{-\frac{1}{2}}L_{\iso}^{(a,a)}(D^{(a,a)})^{-\frac{1}{2}} )  \|v_a\|^2\\
&  \geq \min_{1\leq a\leq k} \lambda_2( (D^{(a,a)})^{-\frac{1}{2}}L_{\iso}^{(a,a)} (D^{(a,a)})^{-\frac{1}{2}} ).
\end{align*}
which follows from $L = L_{\iso} + L_{\delta}$, $L_{\delta}\succeq 0$, and $v_a\perp (D^{(a,a)})^{\frac{1}{2}} 1_{n_a}.$ In particular, the second last inequality is ensured by the variational characterization of the second smallest eigenvalue of symmetric matrices.
By using the fact that $SS^{\top}$ and $S^{\top}S$ always have the same eigenvalues for any square matrix $S$, then we have
\[
\lambda_2( (D^{(a,a)})^{-\frac{1}{2}}L_{\iso}^{(a,a)}(D^{(a,a)})^{-\frac{1}{2}} ) = \lambda_2( (L_{\iso}^{(a,a)})^{\frac{1}{2}}(D^{(a,a)})^{-1}(L_{\iso}^{(a,a)})^{\frac{1}{2}} ) 
\]
where $S $ is set as $(D^{(a,a)})^{-\frac{1}{2}}(L_{\iso}^{(a,a)})^{\frac{1}{2}}. $
Moreover, it holds that
\begin{align*}
(L_{\iso}^{(a,a)})^{\frac{1}{2}}(D^{(a,a)})^{-1}(L_{\iso}^{(a,a)})^{\frac{1}{2}} 
& = (L_{\iso}^{(a,a)})^{\frac{1}{2}}(D_{\iso}^{(a,a)})^{-\frac{1}{2}}D_{\iso}^{(a,a)} (D^{(a,a)})^{-1} (D_{\iso}^{(a,a)})^{-\frac{1}{2}} (L_{\iso}^{(a,a)})^{\frac{1}{2}}  \\
& \succeq (1 - \|P_{\delta}\|_{\infty}) \cdot (L_{\iso}^{(a,a)})^{\frac{1}{2}}  (D_{\iso}^{(a,a)})^{-1} (L_{\iso}^{(a,a)})^{\frac{1}{2}}
\end{align*}
where the second inequality follows from $D = D_{\iso} + D_{\delta}$, $\|D^{-1}D_{\delta}\| \leq \|P_{\delta}\|_{\infty}$, and
\[
(D^{(a,a)})^{-1} D_{\iso}^{(a,a)} = I_{n_a} - (D^{(a,a)})^{-1} D_{\delta}^{(a,a)} \succeq  (1 - \|P_{\delta}\|_{\infty})I_{n_a}.
\]
Note that both $(L_{\iso}^{(a,a)})^{\frac{1}{2}}(D^{(a,a)})^{-1}(L_{\iso}^{(a,a)})^{\frac{1}{2}} $ and $(L_{\iso}^{(a,a)})^{\frac{1}{2}}  (D_{\iso}^{(a,a)})^{-1} (L_{\iso}^{(a,a)})^{\frac{1}{2}}$ have $1_{n_a}$ in the null space. Thus their corresponding  second smallest eigenvalues satisfy
\begin{align*}
\lambda_2((L_{\iso}^{(a,a)})^{\frac{1}{2}}(D^{(a,a)})^{-1}(L_{\iso}^{(a,a)})^{\frac{1}{2}} ) & \geq (1-\|P_{\delta}\|_{\infty}) \cdot \lambda_2 ((L_{\iso}^{(a,a)})^{\frac{1}{2}}  (D_{\iso}^{(a,a)})^{-1} (L_{\iso}^{(a,a)})^{\frac{1}{2}}) \\
& = (1-\|P_{\delta}\|_{\infty}) \cdot \lambda_2 ((D_{\iso}^{(a,a)})^{-\frac{1}{2}}  L_{\iso}^{(a,a)} (D_{\iso}^{(a,a)})^{-\frac{1}{2}}) \\
& = (1-\|P_{\delta}\|_{\infty}) \cdot \lambda_2 ( (D_{\iso}^{(a,a)})^{-1}L_{\iso}^{(a,a)}) \\
& = (1-\|P_{\delta}\|_{\infty}) \cdot \lambda_2 ( L_{\rw, \iso}^{(a,a)}), 
\end{align*}
where $L_{\rw, \iso}$ is defined in~\eqref{def:Piso}.

\vskip0.5cm
Hence, from~\eqref{eq:vQv2},  the lower bound of $v^{\top}Qv$ satisfies
\begin{align*}
v^{\top}Qv & \geq (1- \|P_{\delta}\|_{\infty}) \cdot \min_{1\leq a\leq k}\lambda_2\left(  L_{\rw, \iso}^{(a,a)}\right)  + z \geq 0
\end{align*}
if $z \geq - (1- \|P_{\delta}\|_{\infty}) \cdot \min_{1\leq a\leq k}\lambda_2\left(  L_{\rw, \iso}^{(a,a)}\right).$

Note that we also require $z< -4\|P_{\delta}\|_{\infty}$ to ensure $B^{(a,b)} > 0$ and thus 
\[
-4\|P_{\delta}\|_{\infty} > z\geq -(1- \|P_{\delta}\|_{\infty}) \cdot \min_{1\leq a\leq k}\lambda_2\left(  L_{\rw, \iso}^{(a,a)}\right)
\]
is needed to ensure the existence of $z$.
This is implied by
\[
\frac{\|P_{\delta}\|_{\infty}}{1 - \|P_{\delta}\|_{\infty}} < \frac{\min_{1\leq a\leq k} \lambda_2( L_{\rw,\iso}^{(a,a)}  ) }{4}
\]
which is exactly the assumption in Theorem~\ref{thm:main2}.
\end{proof}

\subsection{Proof of Theorem~\ref{thm:circles} and~\ref{thm:lines}}
We begin with presenting two useful supporting results.
The first one is the famous Gr\"onwall's inequality which was proposed by Gr\"onwall in~\cite{Gronwall19} and can be found in~\cite{Walter98} as well.
\begin{theorem}[{\bf Gr\"onwall's inequality}]\label{thm:gron}
If $g(t)$ is nonnegative and $f(t)$ satisfies the integral inequality
\[
f(t) \leq f(t_0) + \int_{t_0}^t g(s)f(s)\diff s, \quad\forall t \geq t_0,
\]
then
\[
f(t) \leq f(t_0)\exp\left( \int_{t_0}^t g(s)\diff s\right), \quad \forall t\geq t_0.
\]
\end{theorem}

\begin{lemma}\label{lem:gaussian}
For a standard Gaussian random variable $g\sim \mathcal{N}(0,1)$ and $u> 0$,
\[
\Pr(g\geq u) = \frac{1}{\sqrt{2\pi}}\int_u^{\infty}e^{-\frac{t^2}{2}}\diff t \leq \frac{1}{2}e^{-\frac{u^2}{2}}.
\]
\begin{proof}
The proof can be found in~\cite{LLLSW17} but we provide it here for completeness. 
\begin{align*}
\Pr(g \geq u)  & = \frac{1}{\sqrt{2\pi}}\int_u^{\infty} e^{-\frac{t^2}{2}}\diff t = \frac{1}{\sqrt{2\pi}}e^{-\frac{u^2}{2}}\int_u^{\infty} e^{-\frac{t^2 - u^2}{2}}\diff t \\
& \leq \frac{1}{\sqrt{2\pi}}e^{-\frac{u^2}{2}}\int_u^{\infty} e^{-\frac{(t-u)^2}{2}}\diff t  = \frac{1}{2} e^{-\frac{u^2}{2}}
\end{align*}
where $t^2 - u^2 \geq (t-u)^2$ for $t\geq u > 0.$
\end{proof}
\end{lemma}

\subsubsection{Proof of Theorem~\ref{thm:circles}: Two concentric circles}
To prove Theorem~\ref{thm:circles} via Theorem~\ref{thm:main}, we need two quantities: a lower bound for the second smallest eigenvalue of  the Laplacian generated from $\{x_{1,i}\}_{i=1}^n$ and $\{x_{2,j}\}_{j=1}^m$ respectively which characterizes the within-cluster connectivity; and an upper bound of $\|D_{\delta}\|$ which quantifies the inter-cluster connectivity.
We first give a lower bound for the graph Laplacian generated from data on a single circle with Gaussian kernel.

\begin{lemma}\label{lem:lambda-2c}
Suppose $n$ data points  (with $n \geq 7$) are equi-spaced on a circle with radius $r$. Let 
\[
\sigma^2 = \frac{16r^2\gamma}{n^2 \log(\frac{n}{2\pi})}
\]
and the second smallest eigenvalue of the associated graph Laplacian $L = D-W$ satisfies
\[
\lambda_2(L) \gtrsim   \left( \frac{2\pi}{n}\right)^{\frac{1}{2\gamma} + 2} = \left( \frac{2\pi}{n}\right)^{\frac{8r^2}{\sigma^2 n^2 \log(\frac{n}{2\pi})} + 2}, \quad \forall \gamma > 0.
\]
\end{lemma}

\begin{proof}
Let $x_i = r\begin{bmatrix} \cos(\frac{2\pi i}{n}) \\ \sin(\frac{2\pi i}{n})\end{bmatrix}$ with $1\leq i\leq n$. The weight $w_{ij}$ obeys
\[
w_{ij} = e^{-\frac{\|x_i - x_j\|^2}{2\sigma^2}} = e^{-\frac{r^2}{\sigma^2}\left(1 - \cos\left(\frac{2(i-j)\pi}{n}\right)\right)} = e^{-\frac{2r^2}{\sigma^2}\sin^2\left(\frac{(i-j)\pi}{n}\right)}
\]
where $\| x_i  -x_j\|^2 = 2r^2(1 - \cos(\frac{2(i-j)\pi}{n})) = 4r^2 \sin^2 (\frac{(i-j)\pi}{n})$. 

The key to this estimation is the fact that when $\sigma^2$ is small, $ \widetilde{L} : =  e^{ \frac{2r^2}{\sigma^2}\sin^2(\frac{\pi}{n}) } L$ is very close to $L_0$ where
\[
L_0 = 
\begin{bmatrix}
2 & -1 & 0 & \cdots & 0 & -1 \\
-1 & 2 & -1 & \cdots & 0 & 0 \\
0 & -1 & 2 & \cdots & 0 & 0 \\
\vdots & \vdots & \vdots & \ddots & \vdots & \vdots \\
0 & 0 & 0 & \cdots & 2 & -1 \\
-1 & 0 & 0 & \cdots & -1 & 2
\end{bmatrix}
\]
which is the graph Laplacian of an $n$-cycle.
We write down the explicit formula for each entry of $\widetilde{L}$ as
\[
\widetilde{L}_{ij} = 
\begin{dcases}
-e^{ \frac{2r^2}{\sigma^2}\left( \sin^2(\frac{\pi}{n}) - \sin^2(\frac{(i-j)\pi}{n})\right) } , & \text{ if } i\neq j, \\
2 + e^{ \frac{2r^2}{\sigma^2}\sin^2(\frac{\pi}{n}) }  \sum_{l=2}^{n-2} e^{ - \frac{2r^2}{\sigma^2}\sin^2(\frac{\pi l}{n})) },  & \text{ if } i = j.
\end{dcases}
\]
From Weyl's Inequality in~\cite{Stewart90}, we have
\[
\lambda_2(\widetilde{L}) \geq \lambda_2(L_0) - \|\widetilde{L} - L_0\|.
\]

In fact, $\widetilde{L} - L_0$ is also a graph Laplacian generated from the weight matrix with $(i,j)$-entry $e^{ \frac{2r^2}{\sigma^2} \sin^2 (\frac{\pi}{n}) }w_{ij}\cdot 1_{\{2\leq |i-j|\leq n-2\}}$.
Thus, the operator norm of $\widetilde{L} - L_0$ is bounded by twice the maximal degree, i.e., 
\[
\|\widetilde{L} - L_0 \| \leq 2 e^{\frac{2r^2}{\sigma^2}\sin^2(\frac{\pi}{n})} \sum_{l=2}^{n-2}e^{-\frac{2r^2}{\sigma^2} \sin^2(\frac{\pi l}{n})}  \leq 4 e^{\frac{2r^2}{\sigma^2} \sin^2 (\frac{\pi}{n})}  \sum_{l=2}^{\lfloor \frac{n}{2}\rfloor} e^{- \frac{2r^2}{\sigma^2} \sin^2 (\frac{\pi l}{n}) } 
\]
because of the Gershgorin circle theorem and the symmetry of $\widetilde{L} - L_0$. 

Note that $\frac{2x}{\pi} \leq \sin(x) \leq x$ for $0\leq x\leq \frac{\pi}{2}$, and then
\begin{align*}
\|\widetilde{L} - L_0\| 
&   \leq 4 e^{ \frac{2r^2\pi^2}{n^2\sigma^2}}  \sum_{l=2}^{\lfloor \frac{n}{2}\rfloor} e^{-  \frac{8r^2l^2}{n^2\sigma^2}  }  \leq 4 e^{\frac{2r ^2\pi^2}{n^2\sigma^2} }  \left( e^{-\frac{32r^2 }{n^2\sigma^2}} +\int_2^{\infty} e^{- \frac{8r^2t^2}{n^2\sigma^2} }\diff t \right)\\
& =  4e^{\frac{2r^2(\pi^2-16)}{n^2\sigma^2}} + \frac{n\sigma}{r} e^{\frac{2r^2\pi^2}{n^2\sigma^2} }  \int_{\frac{8r}{n\sigma}}^{\infty} e^{- \frac{s^2}{2} }\diff s \\
& \leq \left(4 + \frac{n\sigma}{r} \sqrt{\frac{\pi}{2}}\right)  e^{\frac{2r^2(\pi^2 - 16)}{n^2\sigma^2} },
\end{align*}
where $s = \frac{4r t}{n\sigma}$; the second inequality is due to the monotonicity of  the Gaussian kernel, and the last inequality follows from Lemma~\ref{lem:gaussian}.

Note that $\lambda_2(L_0) = 2 - 2\cos\left(\frac{2\pi}{n}\right)$, see~\cite[Chapter 1]{BrouwerH11}, and hence
\begin{align*}
\lambda_2(\widetilde{L})  & \geq 2 - 2\cos\left( \frac{2\pi}{n}\right) -\left(4 + \frac{n\sigma}{r} \sqrt{\frac{\pi}{2}} \right) e^{\frac{2r^2(\pi^2 - 16)}{n^2\sigma^2} }  \\
& \gtrsim \frac{4\pi^2}{n^2}  - \left(4 + \frac{n\sigma}{r} \sqrt{\frac{\pi}{2}}  \right)e^{\frac{2r^2(\pi^2 - 16)}{n^2\sigma^2} }.
\end{align*}
Now we substitute $\sigma^2 = \frac{16r^2\gamma}{n^2\log \left( \frac{n}{2\pi}\right)}$ back into $\lambda_2(\widetilde{L})$, and by definition of $\widetilde{L}$, the second smallest eigenvalue of $L$ satisfies
\begin{align*}
\lambda_2(L) & = e^{-\frac{2r^2}{\sigma^2} \sin^2(\frac{\pi}{n})}\lambda_2(\widetilde{L}) 
\geq e^{-\frac{8r^2}{n^2\sigma^2}}\lambda_2(\widetilde{L}) \\
& \gtrsim \left(\frac{2\pi}{n}\right)^{\frac{1}{2\gamma }} \left( \frac{4\pi^2}{n^2}  -\left(4+ \frac{n\sigma}{r} \sqrt{\frac{\pi}{2}} \right) e^{\frac{2r^2(\pi^2 - 16)}{n^2\sigma^2} } \right) \\
& \geq \left(\frac{2\pi}{n}\right)^{\frac{1}{2\gamma}} \left( \frac{4\pi^2}{n^2}  - \left(4 + 2\sqrt{\frac{2\pi \gamma}{ \log(\frac{n}{2\pi})}}\right) \left(\frac{2\pi}{n}\right)^{\frac{3}{4\gamma}} \right).
\end{align*}
By letting $0<\gamma \leq \frac{1}{4}$, we have
\begin{equation}\label{eq:L2-small}
\lambda_2(L)  \gtrsim \left( \frac{2\pi}{n}\right)^{\frac{1}{2\gamma}+2}.
\end{equation}
So far, we have established a lower bound of $\lambda_2(L)$ for small $\gamma \leq \frac{1}{4}$ (or small $\sigma^2$ equivalently). Now we extend this bound for any $\gamma > 0.$
Let $L(t)$ be the graph Laplacian w.r.t. the weight matrix $W(t)$ whose $(i,j)$-entry is $w_{ij}(t) = e^{-\frac{2r^2}{t}\sin^2\left(\frac{(i-j)\pi}{n}\right)}$ and the derivative of each $w_{ij}(t)$ obeys
\begin{align}
\begin{split}\label{eq:wlow1}
\frac{\diff w_{ij}(t)}{\diff t} & = \frac{2r^2}{t^2} \sin^2\left(\frac{(i-j)\pi}{n}\right) w_{ij}(t)  \\
& \geq \frac{2r^2}{t^2} \cdot \frac{4}{\pi^2} \cdot \frac{\pi^2}{n^2}w_{ij}(t) = \frac{8r^2}{n^2 t^2}w_{ij}(t) > 0, \qquad \text{if } i\neq j.
\end{split}
\end{align}
Note that $L(t) = \diag(W(t)1_{N}) - W(t)$ and by fundamental theorem of calculus, we have
\[
L(t) = \int_{t_0}^t \frac{\diff L(s)}{\diff s}\diff s +L(t_0)
\]
where $\frac{\diff L(t)}{\diff t}$ is also a graph Laplacian w.r.t. the weight matrix $\frac{\diff w_{ij}(t)}{\diff t}$. 
For given $t$, let $v$  be the normalized eigenvector w.r.t. the second smallest eigenvalue of $L(t)$, then
\begin{align*}
\lambda_2(L(t)) & = v^{\top}L(t)v = \int_{t_0}^t v^{\top}\left(\frac{\diff L(s)}{\diff s}\right)v \diff s + v^{\top} L(t_0) v \\
& \geq \frac{8r^2}{n^2}\int_{t_0}^t \frac{v^{\top}\left(L(s)\right)v}{s^2} \diff s + \lambda_2(L(t_0)) \\
& \geq \frac{8r^2}{n^2 } \int_{t_0}^t \frac{\lambda_2(L(s))}{s^2}\diff s + \lambda_2(L(t_0))
\end{align*}
where the first inequality follows from the quadratic form of graph Laplacian~\eqref{eq:qual} and~\eqref{eq:wlow1}, 
\[
v^{\top}\left(\frac{\diff L(s)}{\diff s}\right)v = \sum_{i<j} \frac{\diff w_{ij}(s)}{\diff s}(v_i - v_j)^2 \geq \frac{8r^2}{n^2 s^2} \sum_{i<j} w_{ij}(s) (v_i - v_j)^2 = \frac{8r^2}{n^2 s^2} v^{\top}L(s)v.
\]
By Gr\"onwall's inequality, i.e., Theorem~\ref{thm:gron}, with $f(t) = -\lambda_2(L(t))$ and $g(t) = \frac{8r^2}{n^2 t^2},$
\[
\lambda_2(L(t)) 
\geq \lambda_2(L(t_0)) e^{\frac{8r^2}{n^2}\int_{t_0}^{t} \frac{1}{s^2}\diff s } = \lambda_2(L(t_0)) e^{-\frac{8r^2}{n^2}\left( \frac{1}{t} - \frac{1}{t_0}\right) }.
\]

Finally, we let $t = \sigma^2 = \frac{16 r^2\gamma}{n^2\log(\frac{n}{2\pi})} $ with $\gamma \geq  \frac{1}{4}$ and $t = \sigma_0^2$ with $\gamma_0 = \frac{1}{4}$ $(\sigma \geq \sigma_0)$. Then by substituting these parameters into the estimation above and applying~\eqref{eq:L2-small}, the Fiedler eigenvalue of $L(\sigma^2)$ satisfies
\[
\lambda_2(L(\sigma^2)) \geq \lambda_2(L(\sigma_0^2)) \left(\frac{2\pi}{n}\right)^{\frac{1}{2\gamma} - \frac{1}{2\gamma_0}} =\left( \frac{2\pi}{n} \right)^{\frac{1}{2\gamma_0} + 2} \left(\frac{2\pi}{n}\right)^{\frac{1}{2\gamma} - \frac{1}{2\gamma_0}}  = \left( \frac{2\pi}{n}\right)^{\frac{1}{2\gamma}+2}
\]
for any $\gamma > 0.$
\end{proof}

Note that the lower bound in Lemma~\ref{lem:lambda-2c} is not tight at all. Fortunately, it will not affect our performance bound too much. 
With Lemma~\ref{lem:lambda-2c} at our disposal, we are ready to prove Theorem~\ref{thm:circles}.

\begin{proof}[{\bf Proof of Theorem~\ref{thm:circles}}]
Suppose the data satisfy~\eqref{eq:model-2c} and consider the associated weight matrix
\[
W = 
\begin{bmatrix}
W^{(1,1)} & W^{(1,2)} \\
W^{(2,1)} & W^{(2,2)}
\end{bmatrix}\in\RR^{(n +m)\times (n + m)}
\]
where $\kappa = \frac{r_2}{r_1} > 1$ and $m = \lfloor n\kappa\rfloor > n$.
To apply Theorem~\ref{thm:main}, we should estimate $\lambda_2(L_{\iso}^{(a,a)})$ where
\[
\quad L_{\iso}^{(a,a)} = \diag(W^{(a,a)}1_{n_a}) - W^{(a,a)}, \quad a=1,2,
\]
with $n_1 = n$ and $n_2 = m$, and the inter-cluster connectivity $\|W^{(1,2)}1_m\|_{\infty}$ and $\|W^{(2,1)}1_n\|_{\infty}$. Note that the distance between two points in different clusters is always greater than $r_2 - r_1.$ 
As a result, every entry in $W^{(1,2)}$ is bounded by $e^{-\frac{(r_1 - r_2)^2}{2\sigma^2}}$ 
where the bandwidth $\sigma$ is chosen as
\begin{equation}\label{eq:sigma_cir}
\sigma^2 = \frac{16r_1^2 \gamma}{n^2 \log(\frac{m}{2\pi})}. 
\end{equation}
The inter-cluster connectivity is bounded by
\begin{align*}
\|D_{\delta}\| & = \max\{\| W^{(1,2)}1_m \|_{\infty}, \|W^{(2,1)}1_n\|_{\infty}\} \leq m e^{-\frac{(r_1 - r_2)^2}{2\sigma^2}}  \\
& = me^{-\frac{n^2(\kappa-1)^2 \log(\frac{m}{2\pi})}{32\gamma }} = m\left( \frac{2\pi}{m}\right)^{\frac{n^2\Delta^2 }{32\gamma }}
\end{align*}
where $\Delta = \kappa - 1 = \frac{r_2-  r_1}{r_1}$.

By Lemma~\ref{lem:lambda-2c} and the $\sigma^2$ in~\eqref{eq:sigma_cir}, the second smallest eigenvalue of $L_{\iso}^{(a,a)}$ satisfies
\begin{align*}
\lambda_2(L_{\iso}^{(1,1)}) & \gtrsim \left( \frac{2\pi}{n}\right)^{\frac{\log(\frac{m}{2\pi})}{2\gamma\log(\frac{n}{2\pi})} + 2} \geq \left( \frac{2\pi}{m}\right)^{\frac{1}{2\gamma}}  \frac{4\pi^2}{n^2}, \qquad \lambda_2(L_{\iso}^{(2,2)}) \gtrsim \left( \frac{2\pi}{m}\right)^{\frac{1}{2\gamma} + 2}
\end{align*}
where $\frac{n}{r_1}\approx \frac{m}{r_2}.$
Since $m > n$, the lower bound for $\min_{a=1,2}\lambda_2(L_{\iso}^{(a,a)})$ satisfies
\[
\min_{a=1,2} \lambda_2(L_{\iso}^{(a,a)}) \gtrsim \left( \frac{2\pi}{m} \right)^{\frac{1}{2\gamma } + 2}.
\]
Under the separation condition of Theorem~\ref{thm:circles}, 
\[
 \Delta^2 \geq \frac{16}{n^2 } + \frac{32\gamma}{n^2}\left(2 + \frac{\log (4 m)}{\log (\frac{m}{2\pi})}\right),
\]
we have
\begin{align*}
\|D_{\delta}\| & \leq m \left(\frac{2\pi}{m}\right)^{\frac{n^2\Delta^2}{32\gamma}} \leq m\left(\frac{2\pi}{m}\right)^{ \frac{1}{2\gamma} + 2 + \frac{\log (4 m)}{\log (\frac{m}{2\pi})} }\\
& \leq \frac{1}{4}  \left(\frac{2\pi}{m}\right)^{\frac{1}{2\gamma} + 2} \lesssim\frac{\min_{a=1,2} \lambda_2(L_{\iso}^{(a,a)})}{4}
\end{align*} 
where $\frac{2\pi}{m} < 1.$
As a consequence of Theorem~\ref{thm:main}, the exact recovery via~\eqref{prog:primal} is guaranteed under the conditions stated in Theorem~\ref{thm:circles}.
\end{proof}

\subsubsection{Proof of Theorem~\ref{thm:lines}: Two parallel lines}

We first estimate the second smallest eigenvalue of the graph Laplacian of one single cluster and then apply Theorem~\ref{thm:main}.

\begin{lemma}\label{lem:line}
Suppose there are $n$ equispaced points on the unit interval and the weight matrix $W$ is constructed via the Gaussian kernel with $\sigma^2 = \frac{\gamma }{(n-1)^2 \log (\frac{n}{\pi})}$. Then, the second smallest eigenvalue of graph Laplacian $L = D - W$ satisfies
\begin{equation}\label{eq:line-fl}
\lambda_2(L) \gtrsim  \left( \frac{\pi}{n}\right)^{\frac{1}{2\gamma } + 2}, \quad \gamma > 0.
\end{equation}
\end{lemma}

\begin{proof}
 For each cluster in~\eqref{eq:distlines}, its weight matrix is a Toeplitz matrix and satisfies
\begin{equation}\label{def:Wtwolines}
w_{ij} = e^{-\frac{|i-j|^2}{2\sigma^2(n-1)^2 }}, \quad 1\leq i, j\leq n,
\end{equation}
with Gaussian kernel $\Phi_{\sigma}(x,y) = e^{-\frac{\|x-y\|^2}{2\sigma^2}}$. 

The proof strategy is similar to that of Lemma~\ref{lem:lambda-2c}.
We first show the lower bound of the Fiedler eigenvalue~\eqref{eq:line-fl} holds for $\gamma \leq \frac{1}{2}$ and later extend this to any $\gamma > 0.$
We claim that $\widetilde{L}:=e^{\frac{1}{2\sigma^2(n-1)^2}} L$ is very close to $L_0$ if  $\sigma^2$ is small,  where
\[
L_0 = 
\begin{bmatrix}
1 & -1 & 0 & \cdots & 0 & 0 \\
-1 & 2 & -1 & \cdots & 0 & 0 \\
0 & -1 & 2 & \cdots & 0 & 0 \\
\vdots & \vdots & \vdots & \ddots & \vdots & \vdots \\
0 & 0 & 0 & \cdots & 2 & -1 \\
0 & 0 & 0 & \cdots & -1 & 1
\end{bmatrix}
\]
which is the graph Laplacian of a path of $n$ nodes.
Note that  $\lambda_2(L_0)$ is given explicitly by $\lambda_2(L_0)= 2 - 2\cos\left( \frac{\pi}{n}\right)$, also see~\cite[Chapter 1]{BrouwerH11}.
All the entries in $\widetilde{L}$ are of the following form:
\[
\widetilde{L}_{ij} = 
\begin{dcases}
-e^{\frac{1 - |i-j|^2}{2\sigma^2(n-1)^2}}, & \text{ if }i\neq j, \\
2 + e^{\frac{1}{2\sigma^2(n-1)^2}}  \sum_{l: |l-i| \geq 2} e^{\frac{-|i-l|^2}{2\sigma^2(n-1)^2}},  & \text{ if } i = j, \text{ and }i\neq 1 \text{ or }n,\\
1 + e^{\frac{1}{2\sigma^2(n-1)^2}} \sum_{l:|l-i|\geq 2}e^{\frac{-|i-l|^2}{2\sigma^2(n-1)^2}},  & \text{ if } i = j, \text{ and } i= 1 \text{ or }n .
\end{dcases}
\]
Still, by Weyl's inequality in~\cite{Stewart90}, the second smallest eigenvalue of $\widetilde{L}$ satisfies
\[
\lambda_2(\widetilde{L}) \geq \lambda_2(L_0) - \|\widetilde{L} - L_0\|.
\]
To have a lower bound for $\lambda_2(\widetilde{L})$, it suffices to get an upper bound for $\|\widetilde{L} - L_0\|.$ Note $\widetilde{L} - L_0$ is also a graph Laplacian generated from the weight matrix whose $(i,j)$-entry is $e^{\frac{1}{2\sigma^2(n-1)^2}}w_{ij}\cdot 1_{\{|i-j|\geq 2\}}$. Thus, the operator norm of $\widetilde{L} - L_0$ is bounded by twice the maximal degree of the weight matrix $\{e^{\frac{1}{2\sigma^2(n-1)^2}}w_{ij}\cdot 1_{\{|i-j|\geq 2\}}\}$. Therefore, the operator norm of  $\widetilde{L} - L_0$ satisfies
\[
\|\widetilde{L} - L_0\| \leq 4 e^{\frac{1}{2\sigma^2(n-1)^2}}  \sum_{l=2}^{\lfloor \frac{n+1}{2}\rfloor} e^{-\frac{l^2}{2\sigma^2(n-1)^2}}    \leq 4 e^{\frac{1}{2\sigma^2(n-1)^2}}  \sum_{l=2}^{\infty} e^{-\frac{l^2}{2\sigma^2(n-1)^2}}.
\]
This is also due to Gershgorin's circle theorem, see Theorem~\ref{thm:gersh}, as well as the symmetry of $\widetilde{L} - L_0$. 
By using Lemma~\ref{lem:gaussian}, we immediately have an upper bound of $\|\widetilde{L} - L_0\|$ as follows
\begin{align*}
\|\widetilde{L} - L_0\| & \leq 4 e^{\frac{1}{2\sigma^2(n-1)^2}}  \sum_{l=2}^{\infty} e^{-\frac{l^2}{2\sigma^2(n-1)^2}} \\
&   \leq 4 e^{\frac{1}{2\sigma^2(n-1)^2}} \left(e^{-\frac{2}{\sigma^2(n-1)^2}}+ \int_2^{\infty} e^{-\frac{t^2}{2\sigma^2(n-1)^2}}\diff t\right) \\
& =  4e^{-\frac{3}{2\sigma^2(n-1)^2}} + 4\sigma(n-1) e^{\frac{1}{2\sigma^2(n-1)^2}} \int_{\frac{2}{\sigma(n-1)}}^{\infty} e^{-\frac{s^2}{2}} \diff s \\
& \leq \left( 4 + 2\sqrt{2\pi} \sigma (n-1)\right)e^{-\frac{3}{2\sigma^2(n-1)^2}}
\end{align*}
where $s = \frac{t}{\sigma(n-1)}$. Hence,
\begin{align*}
\lambda_2(\widetilde{L})  & \geq 2 - 2\cos\left( \frac{\pi}{n}\right) - \left( 4 +2\sqrt{2\pi}\sigma (n-1)\right) e^{-\frac{3}{2\sigma^2(n-1)^2}} \\
& \gtrsim \frac{\pi^2}{n^2}  -  \left( 4 +2\sqrt{2\pi}\sigma (n-1)\right)e^{-\frac{3}{2\sigma^2(n-1)^2}}.
\end{align*}
Note $\sigma^2 = \frac{\gamma}{(n-1)^2\log(\frac{n}{\pi})}$ and we have
\begin{align*}
\lambda_2(L) & = e^{-\frac{1}{2\sigma^2(n-1)^2}}\lambda_2(\widetilde{L}) \gtrsim e^{-\frac{\log(\frac{n}{\pi})}{2\gamma}} \left( \frac{\pi^2}{n^2} - \left( 4 + 2\sqrt{\frac{2\pi\gamma}{\log (\frac{n}{\pi})}} \right)\left( \frac{\pi}{n}\right)^{\frac{3}{2\gamma}}\right) \\
& = \left( \frac{\pi}{n}\right)^{\frac{1}{2\gamma}} \left( \frac{\pi^2}{n^2} - \left( 4 + 2\sqrt{\frac{2\pi\gamma}{\log (\frac{n}{\pi})}} \right) \left( \frac{\pi}{n}\right)^{\frac{3}{2\gamma}}\right).
\end{align*}
To ensure that $\lambda_2(L)$  has a non-trivial lower bound, we pick $\gamma \leq \frac{1}{2}$ and then 
\[
\lambda_2(L) \gtrsim  \left( \frac{\pi}{n}\right)^{\frac{1}{2\gamma}+2}.
\]

Now we extend $\gamma$ to $\RR_+.$ It is not hard to see that $\lambda_2(L)$ (if you treat the weight as a function of bandwidth $\sigma$) is an increasing function of $\sigma^2$ because the larger $\sigma^2$ is, the more connected the graph becomes. 
Define $L(t)$ to be the Laplacian matrix associated with $W(t) = (w_{ij}(t))_{ij}$ where $w_{ij} = e^{-\frac{|i - j|^2}{2(n-1)^2t}}.$ There holds
\begin{equation}\label{eq:dwt2}
\frac{\diff w_{ij}(t)}{\diff t} = \frac{|i  - j|^2}{2(n-1)^2t^2}e^{-\frac{|i - j|^2}{2(n-1)^2t}} \geq \frac{1}{2(n-1)^2 t^2}e^{-\frac{|i-j|^2}{2(n-1)^2 t}}, \quad i\neq j.
\end{equation}
Then
\[
L(t)  = \int_{t_0}^{t} \frac{\diff L(s)}{\diff s} \diff s  + L(t_0) 
\]
where $ \frac{\diff L(s)}{\diff s}$ denotes the Laplacian matrix generated by the weight $\left(\frac{\diff w_{ij}(s)}{\diff s} \right)_{1\leq i,j\leq n}$. Therefore, the three matrices $L(t)$, $L(t_0)$, and $\frac{\diff L(s)}{\diff s}$ are all graph Laplacians and have the constant function in their nullspace. 

We apply $\lambda_2(\cdot)$ to  both sides of the equation above, and then the following relation holds 
\begin{equation}\label{eq:lam2Lt}
\lambda_2(L(t)) \geq \int_{t_0}^{t}\lambda_2\left( \frac{\diff L(s)}{\diff s}\right)\diff s + \lambda_2(L(t_0))
\end{equation}
which follows from the variational form of the second smallest eigenvalue.
For $ \frac{\diff L(s)}{\diff s}$ and $v\perp 1_n$, we have
\begin{align*}
v^{\top}\left(\frac{\diff L(s)}{\diff s}\right) v & = \sum_{i<j} \frac{ \diff w_{ij}(s)}{\diff s}(v_i - v_j)^2 \\
& \geq \frac{1}{2(n-1)^2 s^2} \sum_{i<j}w_{ij}(s)(v_i - v_j)^2  = \frac{1}{2(n-1)^2s^2}v^{\top}L(s)v
\end{align*}
which follows from~\eqref{eq:qual} and~\eqref{eq:dwt2}. Hence
\[
\lambda_2\left(\frac{\diff L(s)}{\diff s}\right) \geq \frac{1}{2(n-1)^2 s^2}\lambda_2(L(s)).
\]
By substituting this expression into~\eqref{eq:lam2Lt}, we have
\[
\lambda_2(L(t)) \geq \frac{1}{2(n-1)^2}\int_{t_0}^{t}\frac{\lambda_2(L(s))}{s^2}\diff s + \lambda_2(L(t_0)).
\]
By applying Gr\"onwall's inequality (Theorem~\ref{thm:gron}) with $f(t) = -\lambda_2(L(t))$ and $g(t) = \frac{1}{2(n-1)^2 t^2}$,  we obtain
\[
\lambda_2(L(t)) \geq \lambda_2(L(t_0)) e^{\frac{1}{2(n-1)^2}\int_{t_0}^t \frac{1}{s^2}\diff s}  = \lambda_2(L(t_0)) e^{\frac{1}{2(n-1)^2} \left(\frac{1}{t_0} - \frac{1}{t} \right) }.
\]

So we get a lower bound of $\lambda_2(L(t))$ for all $t > t_0$.
Setting $t=\sigma^2$ and $t_0 = \sigma_0^2$ with $t> t_0$, we get
\[
\lambda_2(L(\sigma^2)) \geq \lambda_2(L(\sigma_0^2)) \exp\left(\frac{1}{2(n-1)^2} \left(\frac{1}{\sigma^2_0} - \frac{1}{\sigma^2} \right) \right).
\]

By letting $\sigma^2_0 = \frac{\gamma_0}{(n-1)^2 \log (\frac{n}{\pi})}$ ($\gamma_0 = \frac{1}{2}$) and using $\lambda_2(L(\sigma_0^2)) \gtrsim  \frac{\pi^3}{n^3}$, we see that
 the second smallest eigenvalue of $L(\sigma^2)$ is bounded by
\[
\lambda_2(L(\sigma^2)) \gtrsim \frac{\pi^3}{n^3} \cdot \frac{n}{\pi} \cdot e^{-\frac{1}{2\sigma^2(n-1)^2}} = \frac{\pi^2}{n^2} \cdot e^{-\frac{1}{2\sigma^2(n-1)^2}}.
\]
Now set $\sigma^2 = \frac{\gamma}{(n-1)^2 \log (\frac{n}{\pi})}$, which yields
\[
\lambda_2(L(\sigma^2)) \gtrsim \left( \frac{\pi}{n}\right)^{\frac{1}{2\gamma}+2}
\]
for any $\gamma>0.$

\end{proof}

We are ready to proceed to the proof of Theorem~\ref{thm:lines}.
\begin{proof}[{\bf Proof of Theorem~\ref{thm:lines}}]
Note that the any two points on different lines are separated by at least $\Delta$. Under 
\[
\Delta^2 \geq 6\sigma^2 \log n + \frac{1}{(n-1)^2}, \quad \sigma^2 = \frac{\gamma}{(n-1)^2\log(\frac{n}{\pi})}
\] 
it holds that
\[
\| D_{\delta}\| \leq n e^{-\frac{\Delta^2}{2\sigma^2}} \leq ne^{-\left( 3\log n + \frac{1}{2\sigma^2(n-1)^2}\right)} = e^{-\left( 2\log n + \frac{1}{2\sigma^2 (n-1)^2} \right)}.
\]
Then we apply Lemma~\ref{lem:line} and  get 
\[
\min_{a=1,2}\lambda_2(L_{\iso}^{(a,a)}) 
\gtrsim \left( \frac{\pi}{n}\right)^{\frac{1}{2\gamma} + 2}
= e^{-\left(\frac{1}{2\gamma } + 2\right)\log \left(\frac{n}{\pi}\right)} 
= e^{-\left(2 \log(\frac{n}{\pi})+ \frac{1}{2\sigma^2(n-1)^2}\right)}.
\]
Thus, exact recovery is guaranteed since the assumptions of Theorem~\ref{thm:main} are fulfilled,
\[
\|D_{\delta}\|  \leq e^{-\left( 2\log n + \frac{1}{2\sigma^2 (n-1)^2} \right)} \leq  \frac{1}{4}\cdot e^{-\left(2\log(\frac{n}{\pi}) + \frac{1}{2\sigma^2(n-1)^2}\right)}  \lesssim \frac{1}{4} \min_{a=1,2} \lambda_2(L_{\iso}^{(a,a)}).
\]

\end{proof}

\subsection{Proof of Theorem~\ref{thm:sbm}: Stochastic block model}
\label{ss:proof-sbm}

The proof relies on two ingredients: a lower bound of the second smallest eigenvalue of the random graph Laplacian; and an upper bound of $\|D_{\delta}\|$. Both quantities can be easily obtained via tools from random matrix theory~\cite{Tropp12} and the Bernstein inequality for scalar random variables~\cite{Ver12}.
\begin{theorem}[{\bf Bernstein inequality}]
For a finite sequence of centered independent random variables $\{z_k\}$ with $|z_k|\leq R$, 
\[
\Pr\left( \sum_{k=1}^n z_k \geq t \right) \leq \exp\left( -\frac{t^2/2}{\sum_{k=1}^n \E z_k^2 + \frac{1}{3}Rt}\right).
\]
\end{theorem}

\begin{theorem}[{\bf Matrix Chernoff inequality}]\label{thm:chernoff}
Consider a finite sequence $\{Z_k\}$ of independent, random, self-adjoint matrices with dimension $n$. Assume that each random matrix satisfies 
\[
Z_k\succeq 0, \quad \|Z_k\| \leq R.
\]
Let $Z = \sum_{k=1}^n Z_k$ and define $\mu_{\min} = \lambda_{\min}(\E(Z))$. Then
\begin{equation}\label{eq:chern}
\Pr(\lambda_{\min}\left( Z\right) \leq (1-\eta)\mu_{\min}) \leq n \left[ \frac{e^{-\eta}}{(1-\eta)^{1-\eta}}\right]^{\mu_{\min}/R}, \quad 0\leq \eta\leq 1.
\end{equation}
\end{theorem}
\begin{remark}
Instead of using the right hand side of~\eqref{eq:chern} directly, one can use the following simpler form,
\[
\frac{e^{-\eta}}{(1-\eta)^{1-\eta}} \leq e^{-\frac{\eta^2}{2}}, \quad 0\leq \eta\leq 1.
\]
\end{remark}

Utilizing the matrix Chernoff inequality, we present the following lemma for a lower bound of the eigengap.
\begin{lemma}\label{lem:l2}
Let $W$ be an $n\times n$ symmetric random matrix whose $(i,j)$ entry is binary and takes  value 1 with probability $p$. Then the second smallest eigenvalue of its corresponding graph Laplacian satisfies
\[
\Pr\left(\lambda_{2}(L) \geq (1-\eta) np\right) \geq 1 - n\exp\left(-\frac{np\eta^2}{4}\right)
\]
where $0<\eta<1.$
\end{lemma}

\begin{proof}
Suppose $W$ is an $n\times n$ self-adjoint matrix, and each entry $w_{ij}$ is binary and takes value $1$ with probability $p$. Let $\{e_i\}_{i=1}^N$ be the canonical basis in $\RR^N.$ Then the graph Laplacian of $W$ is the sum of weighted rank-1 Laplacian matrices,
\[
L = \sum_{i< j} w_{ij} L_{ij}, \qquad L_{ij} := (e_i - e_j)(e_i-e_j)^{\top}, 
\] 
which follow from~\eqref{eq:qual} directly. By the definition of $W$, the expectation of $L$ satisfies
\[
\E (L) = p \sum_{i<j}L_{ij} = p(nI_n - J_{n\times n}).
\]
Hence, we have $\lambda_2(\E (L)) = np$, $L_{ij}\succeq 0$, and $\|L_{ij}\| \leq 2.$ Before applying Theorem~\ref{thm:chernoff}, we need to clarify one thing: the matrix Chernoff inequality estimates the smallest eigenvalue while one cares about the second smallest eigenvalue of $L$. This discrepancy can be easily resolved since all $\{L_{ij}\}_{i<j}$, $L$, and $\E(L)$ have 0 as the smallest eigenvalue and $1_n$ as the corresponding eigenvector. Therefore, when restricted on the complement of $1_n$, the matrix Chernoff inequality immediately applies to the second smallest eigenvalue. 
Thus Theorem~\ref{thm:chernoff} implies that
\begin{align*}
& \Pr(\lambda_{2}(L) \leq (1-\eta) np) \leq n\exp\left(-\frac{np\eta^2}{4}\right) \\
& \qquad\qquad\qquad \Longleftrightarrow \Pr(\lambda_{2}(L) \geq (1-\eta) np) \geq 1 - n\exp\left(-\frac{np\eta^2}{4}\right).
\end{align*}
\end{proof}

Now we are ready to present the proof of Theorem~\ref{thm:sbm}.
\begin{proof}[{\bf Proof of Theorem~\ref{thm:sbm}}]
For the stochastic block model with two clusters, the corresponding weight matrix and its expectation are
\[
W = \begin{bmatrix}
W^{(1,1)} & W^{(1,2)} \\
W^{(2,1)} & W^{(2,2)}
\end{bmatrix}, \quad 
\E(W) = \begin{bmatrix}
pJ_{n\times n} & qJ_{n\times n} \\
qJ_{n\times n} & pJ_{n\times n}
\end{bmatrix}\in\RR^{N\times N}
\]
where $N = 2n.$
By Lemma~\ref{lem:l2}, the graph Laplacian $L_{\iso}^{(a,a)}$ of $W^{(a,a)}$ has its second smallest eigenvalue bounded by
\[
\lambda_2(L^{(a,a)}_{\iso}) \geq (1-\eta)np, \quad a=1,2,
\]
with probability at least $1 - 2n\exp\left(-\frac{np\eta^2}{4}\right)$. 
In particular, if we pick $p = \frac{\alpha \log N}{N}$, then $\min \lambda_2(L_{\iso}^{(a,a)}) \geq \frac{(1-\eta)\alpha\log N}{2}$ holds with probability at least $1-{\cal O}(N^{-\eps})$ for $\eps>0$ if $\eta^2\alpha > 8$.

Now we take a look at $\| D_{\delta}\|$ which equals 
\[
\|D_{\delta}\| = \max \left\{ \|W^{(1,2)}1_n\|_{\infty}, \|W^{(2,1)}1_n\|_{\infty} \right\}.
\]
Each diagonal entry of $D_{\delta}$ is a sum of $n$ i.i.d. Bernoulli random variables and each of these random variables takes 1 with probability $q$.
By applying Bernstein's inequality and then taking the union bound over all entries in $D_{\delta}$, we have
\[
\Pr\left( \|D_{\delta}\| \leq nq + t \right) \geq 1 - N\exp\left( -\frac{t^2/2}{nq(1-q) + t/3}\right).
\]
To fulfill the assumptions in Theorem~\ref{thm:main}, we need $nq + t \leq \frac{(1-\eta)np}{4} \leq \frac{\min \lambda_2(L^{(a,a)}_{\iso})}{4}$ and it suffices to have $t \leq \frac{(1-\eta)np}{4} - nq = \frac{1}{2}\left( \frac{(1-\eta)\alpha}{4} - \beta\right)\log N $ where $q = \frac{\beta\log N}{N}$. Substituting these estimations into the formula above results in
\[
\Pr\left( \|D_{\delta}\| \leq \frac{\min\lambda_2(L_{\iso}^{(a,a)})}{4} \right) \geq 1 -N\exp\left( -\frac{\left( \frac{(1-\eta)\alpha}{4} - \beta\right)^2\frac{\log N}{8}}{\frac{\beta}{2} + \frac{1}{6}\left( \frac{(1-\eta)\alpha}{4} - \beta\right)}\right) = 1- {\cal O}(N^{-\eps})
\]
if we require
\[
\left( \frac{(1-\eta)\alpha}{4} - \beta\right)^2 > 4\beta + \frac{4}{3}\left( \frac{(1-\eta)\alpha}{4} - \beta\right), \qquad \eta^2 \alpha > 8.
\]
Note that the first inequality satisfies
\[
\frac{(1-\eta)\alpha}{4} - \beta > \frac{2}{3} + 2\sqrt{\frac{1}{9} + \beta} \Longleftrightarrow \alpha > \frac{8}{1-\eta}\left(  \frac{1}{3} + \frac{\beta}{2} + \sqrt{\frac{1}{9} + \beta} \right).
\]
Let $\eta = 0.6861$ and we arrive at the desired bound,
\[
\alpha > 26\left(  \frac{1}{3} + \frac{\beta}{2} + \sqrt{\frac{1}{9} + \beta} \right)
\]
which guarantees the exact recovery of the planted communities.
\end{proof}

Finally, here is the promised  proof of claim~\eqref{eq:gw}.
\begin{proof}[{\bf Proof of claim~\eqref{eq:gw}}]
We will prove that
$\lambda_2(D_{\iso} - D_{\delta} - W +\frac{1}{2}1_N1_N^{\top}) > 0$
is ensured by~\eqref{eq:gw}, i.e., $\min \lambda_2(L^{(a,a)}_{\iso}) > 2\|D_{\delta}\|.$
Let $g := \begin{bmatrix}1_n\\ -1_n \end{bmatrix}$ which is perpendicular to $1_N.$ Hence 
\begin{align*}
& \lambda_2(D_{\iso} - D_{\delta} - W + \frac{1}{2} 1_{N}1_{N}^{\top}) \\
& \qquad \qquad = \lambda_2\left(\diag( g)\left(L_{\iso} - D_{\delta} - W_{\delta} + \frac{1}{2} 1_{N}1_{N}^{\top}\right)\diag(g)\right) \\
& \qquad \qquad = \lambda_2( L_{\iso} - D_{\delta} + W_{\delta} + \frac{1}{2}gg^{\top} ) 
\end{align*}
which follows from the diagonal-block structure of $L_{\iso}$ and $D_{\delta}$, and the diagonal blocks of $W_{\delta}$ are zero.
Note that $L_{\iso} + \frac{1}{2}gg^{\top}$ and $D_{\delta} - W_{\delta}$ share the same null space spanned by $1_N$. Moreover, we have 
\[
\lambda_2(L_{\iso} + \frac{1}{2}gg^{\top}) = \lambda_3(L_{\iso})
\]
because $g$ is in the null space of $L_{\iso}$ and $\|\frac{1}{2}gg^{\top}\| = n$ cannot be equal to $\lambda_2(L_{\iso} + \frac{1}{2}gg^{\top})$ since $\|L_{\iso}\| \leq n$ holds.

Hence, by Weyl's inequality and $\|D_{\delta} - W_{\delta}\| \leq 2\|D_{\delta}\|$, it holds that
\begin{align*}
\lambda_2(D_{\iso} - D_{\delta} - W + \frac{1}{2} 1_{N}1_{N}^{\top}) & \geq \lambda_2(L_{\iso} + \frac{1}{2}gg^{\top}) - \|D_{\delta} - W_{\delta}\| \\
 & = \lambda_3(L_{\iso}) - 2\|D_{\delta}\|  = \min \lambda_2(L_{\iso}^{(a,a)}) - 2\|D_{\delta}\| > 0.
\end{align*}
\end{proof}

\section*{Acknowledgement}
S.L. thanks Afonso S. Bandeira for fruitful discussions about stochastic block models. The authors are also grateful to the anonymous referees for their careful reading of this paper and suggestions.




\end{document}